%% file: iclr2027_conference.tex
\documentclass{article} 
\usepackage{iclr2027_conference,times}

\usepackage{etoolbox}

\newcommand{\papermode}{preprint}

\ifdefstring{\papermode}{preprint}{
  \iclrfinalcopy
}{}

\ifdefstring{\papermode}{camera}{
  \iclrfinalcopy
}{}

\input{preamble}
\Crefname{section}{Section}{Sections}
\Crefname{table}{Table}{Tables}
\crefname{section}{Sec.}{Secs.}
\crefname{table}{Tab.}{Tabs.}
\crefname{figure}{Fig.}{Figs.}
\crefname{appendix}{Sec.}{Secs.}
\crefname{equation}{Eq.}{Eqs.}
\crefname{algorithm}{Algo.}{Algos.}

\newcommand{\papertitle}{Simplex Relaxation for Discrete Diffusion}
\title{\papertitle}

\author{
  \begin{tabular}{c}
    \hspace{-0.14in}Jinya Sakurai\textsuperscript{1,4}\thanks{This work was partly done while the author was a visiting student at A*STAR.}
    \quad
    Patrick Pynadath\textsuperscript{2}
    \quad
    Satoshi Hayakawa\textsuperscript{3}
    \\
    Jaehong Yoon\textsuperscript{1}
    \quad
    Xulei Yang\textsuperscript{4}
    \quad
    Nancy F. Chen\textsuperscript{4,5}
    \quad
    Xun Xu\textsuperscript{4,5}
  \end{tabular}
  \\
  \textsuperscript{1}NTU Singapore
  \quad
  \textsuperscript{2}Purdue University
  \quad
  \textsuperscript{3}The University of Tokyo
  \\
  \textsuperscript{4}Institute for Advanced Intelligence and Computing (IAIC), A*STAR
  \\
  \textsuperscript{5}Centre for Frontier AI Research (CFAR), A*STAR
}
\begin{document}

\maketitle

\ifdefstring{\papermode}{preprint}{
  \fancyhead{}
  \fancyhead[L]{Preprint}
}{}

\input{sec/0_abstruct}

\ifdefstring{\papermode}{preprint}{
\begin{center}
\small
\href{https://jinya1013.github.io/simplax}
  {\textcolor{black}{\faGlobe\ Project Page}}
\qquad
\href{https://github.com/jinya1013/simplax}
  {\textcolor{black}{\faGithub\ Code}}
\qquad
\href{https://huggingface.co/...}
  {\textcolor{black}{\faCube\ Models}}
\end{center}
}{}

\input{sec/1_introduction}
\input{sec/2_preliminaries}
\input{sec/3_method}
\input{sec/4_experiments}
\input{sec/5_related_work}
\input{sec/6_conclusion}
\clearpage
\ifdefstring{\papermode}{preprint}{
  \input{sec/999_acknowledgments}
}{
  \input{sec/999_iclr_statements}

  \ifdefstring{\papermode}{camera}{
    \input{sec/999_acknowledgments}
  }{}
}
\clearpage
\bibliography{iclr2027_conference}
\bibliographystyle{iclr2027_conference}
\clearpage
\appendix
\begin{center}
  {\Large\bfseries\papertitle\par}
  \vspace{0.5em}
  {\Large Supplementary Material\par}
\end{center}
\input{sec_app/a_owt_experiment}
\input{sec_app/b_sudoku_experiment}
\input{sec_app/c_contribution_analysis}
\input{sec_app/d_full_hierarchy}
\input{sec_app/e_closed_form}
\input{sec_app/f_continuous_time}
\input{sec_app/g_alternative_objectives}

\input{sec_app/h_sampling}

\end{document}

%% file: preamble.tex
\usepackage{graphicx}
\usepackage[table]{xcolor}
\usepackage{booktabs}
\usepackage{multirow}
\usepackage{array}
\usepackage{tabularx}
\usepackage{placeins}
\usepackage{wrapfig}

\usepackage{amsmath}
\usepackage{amssymb}
\usepackage{amsfonts}
\usepackage{amsthm}
\usepackage{bm}
\usepackage{bbm}
\usepackage{mathtools}

\usepackage{algorithm}
\usepackage{algorithmic}

\usepackage{pifont}
\usepackage[most]{tcolorbox}
\usepackage{caption}
\usepackage{tikz}
\usepackage{enumitem}
\usepackage{hyperref}
\usepackage{cleveref}
\usepackage{fontawesome5}
\usepackage{simpleicons}

\setlist[itemize]{leftmargin=2em}
\setlist[enumerate]{leftmargin=2em}

\hypersetup{
  colorlinks=true,
  linkcolor=red,
  citecolor=blue,
  urlcolor=magenta
}

\definecolor{citecolor}{RGB}{100,100,200}
\definecolor{lrssyellow}{RGB}{255,245,204}
\definecolor{lightgraytext}{gray}{0.55}

\definecolor{owtpanel}{HTML}{F3F6F8}
\definecolor{owtgray}{HTML}{68777F}
\newcolumntype{Y}{>{\raggedright\arraybackslash}X}

\newtheorem{proposition}{Proposition}

\newtheorem{lemma}{Lemma}
\newtheorem{corollary}{Corollary}

\theoremstyle{definition}

\theoremstyle{remark}

\DeclareMathOperator{\Cat}{Cat}
\DeclareMathOperator{\Dir}{Dir}

\newcommand{\KL}[2]{D_{\mathrm{KL}}\!\left[#1 \,\middle\|\, #2\right]}

\newcommand{\EE}[2]{\mathbb{E}_{#1}\!\left[#2\right]}

\newcommand{\Loss}[1]{\mathcal{L}_{#1}}
\newcommand{\bLoss}[1]{\bar{\mathcal{L}}_{#1}}

\newcommand{\inner}[2]{\left\langle #1,\, #2 \right\rangle}

\newcommand{\Vocab}{\mathcal{V}}
\newcommand{\simplex}{\Delta}
\newcommand{\one}{\mathbf{1}}
\newcommand{\ek}{\mathbf{e}_k}

\newcommand{\ba}{\mathbf{a}}
\newcommand{\bb}{\mathbf{b}}
\newcommand{\be}{\mathbf{e}}

\newcommand{\bp}{\mathbf{p}}
\newcommand{\bw}{\mathbf{w}}
\newcommand{\bx}{\mathbf{x}}
\newcommand{\by}{\mathbf{y}}
\newcommand{\bz}{\mathbf{z}}

\newcommand{\balpha}{\boldsymbol{\alpha}}

\newcommand{\bpi}{\boldsymbol{\pi}}

\newcommand{\pt}[1]{\bp_t(#1)}
\newcommand{\ps}[1]{\bp_s(#1)}
\newcommand{\rst}[2]{\mathbf{r}_{s\mid t}(#1,#2)}
\newcommand{\rhost}[2]{\boldsymbol{\rho}_{s\mid t}(#1,#2)}

\newcommand{\nspace}{\mkern-2mu}

\def\eqref#1{equation~\ref{#1}}

\def\1{\bm{1}}

\DeclareMathAlphabet{\mathsfit}{\encodingdefault}{\sfdefault}{m}{sl}
\SetMathAlphabet{\mathsfit}{bold}{\encodingdefault}{\sfdefault}{bx}{n}

\makeatletter
\renewcommand{\maketag@@@}[1]{\hbox{\m@th\normalsize\normalfont#1}}
\makeatother

%% file: sec/0_abstruct.tex
\begin{abstract}
Discrete diffusion models for categorical generation are defined by a corruption kernel, which determines the intermediate state space and the associated reverse prediction problem. Within this family, uniform diffusion has been extensively developed, with recent work connecting its categorical corruption process to continuous representations and dynamics. Motivated by this view, we ask whether uniform discrete diffusion can be augmented with an explicit continuous state while leaving its categorical corruption process unchanged. We introduce Simplax, an exact Dirichlet--categorical augmentation that couples each corrupted categorical state with an auxiliary simplex-valued variable while preserving the uniform diffusion process as its categorical marginal. This augmentation yields a tractable Rao--Blackwellized reverse-bridge objective and a stochastic reverse sampler, while retaining the corrupted categorical state as the denoiser input. Empirically, Simplax improves the generative perplexity--entropy tradeoff on unconditional OpenWebText generation. On Sudoku, a model trained exclusively on $30$-clue puzzles achieves the highest accuracy among the compared methods across all evaluated clue densities, including the minimum uniquely solvable $17$-clue regime, and also achieves the highest validity in unconditional generation.

\end{abstract}

%% file: sec/1_introduction.tex
\section{Introduction}
\label{sec:introduction}

Discrete diffusion models have emerged as a promising framework for generative
modeling over categorical data, including text, biological sequences, and
other symbolic domains
\citep{hoogeboom2021argmax,austin2021structured,campbell2022a,
lou2024discrete,zhang2025target}.
Compared with autoregressive generation, they offer a conceptually different
route to parallel prediction by learning to invert a noising process on full
categorical states. A central design choice in this framework is the corruption
kernel. This choice determines the intermediate state space and the semantics
of reverse updates, and shapes the form of the training objective used to
approximate the reverse process
\citep{campbell2022a,lou2024discrete,zhao2025unified}.

Among these constructions, uniform diffusion has recently received increasing
attention. Unlike masked diffusion, which corrupts tokens toward a distinguished
mask state, uniform diffusion corrupts tokens within the original categorical
alphabet without introducing an absorbing symbol
\citep{schiff2025simple}.
Recent work has expanded uniform diffusion along several directions, including
generalized corruption processes, repeated token revision, self-correction,
few-step generation, improved parameterizations and sampling, and scaling to
increasingly large language models
\citep{rutte2025generalized,gourevitch2026uniformdiffusionmodelsrevisited,
rutte2026scaling,sahoo2026scaling,diffusiongemma2026}.
At the same time, recent work has connected uniform diffusion to continuous
representations and dynamics
\citep{sahoo2025duality,deschenaux2026duality2}.
More broadly, a growing line of categorical and language modeling work has
explored continuous, hybrid, or flow-based formulations over one-hot, simplex,
embedding, and latent representations
\citep{jo2025continuous,lee2026flm,pynadath2025candi,
roos2026categoricalflowmaps,langflow,sflm}.

At the same time, discrete diffusion retains a practical advantage over
continuous-state formulations: categorical inputs can be embedded by direct
token lookup, without the vocabulary-wide dense embedding products required
by simplex-valued inputs.
This raises a natural question:
\begin{center}
\textit{Can uniform discrete diffusion incorporate an auxiliary continuous state \\
while preserving both its categorical process and efficient token-based denoising?}
\end{center}
We seek an exact augmentation in which the continuous state participates in
training and reverse transitions, while marginalizing it recovers the original
uniform diffusion.

In this work, we answer this question with Simplax, an exact
Dirichlet--categorical augmentation of uniform discrete diffusion. Each
corrupted categorical state $\bz_t$ is coupled to an auxiliary simplex-valued
variable $\bw_t$ through a shifted Dirichlet conditional. The hierarchy
preserves the uniform diffusion process exactly as its categorical marginal
and admits the decoder $q(\bz_t \mid \bw_t)=\Cat(\bz_t;\bw_t)$.

\begin{figure}[t]
\centering
\includegraphics[width=\textwidth]{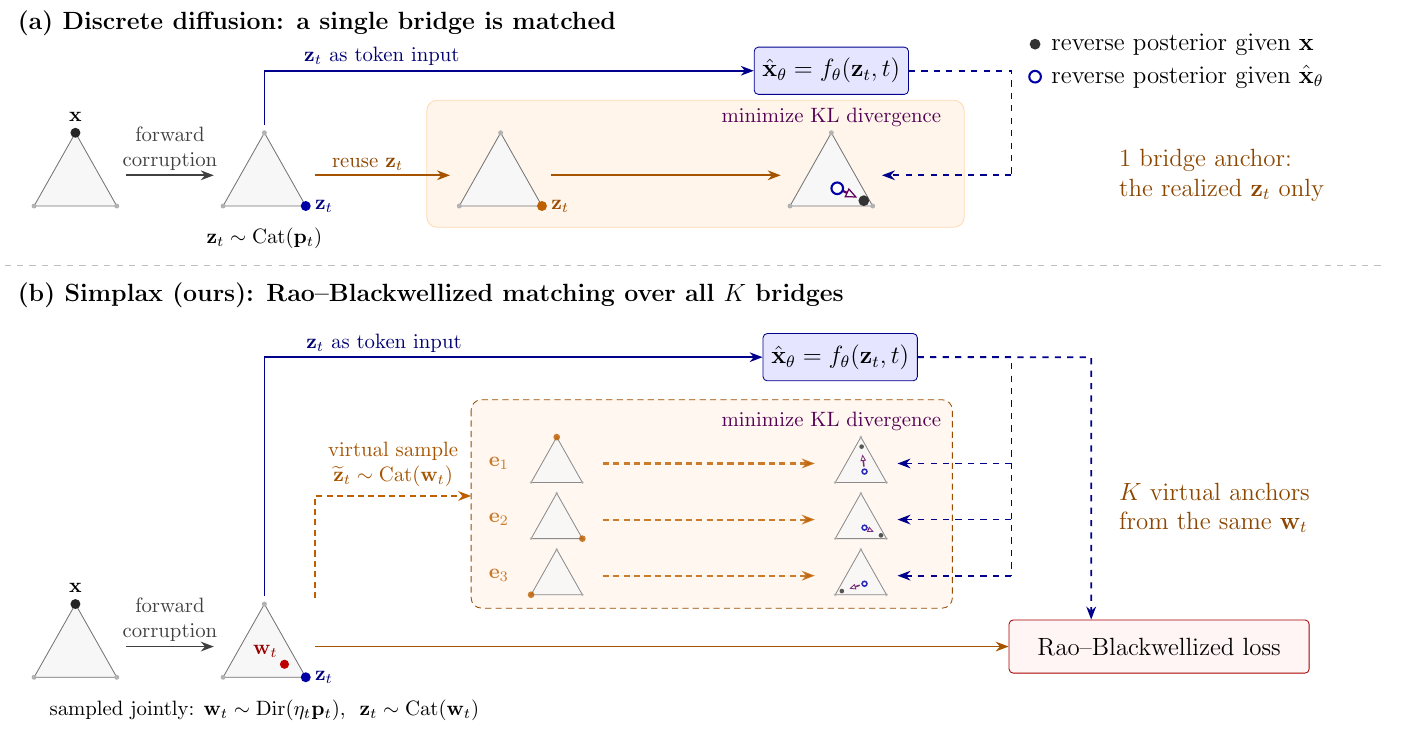}
\caption{%
Reverse-bridge matching, illustrated for $K=3$: one-hot states are vertices of
$\simplex^{K-1}$ and $\bw_t$ is an interior point.
(a) The corrupted state $\bz_t$ is both the denoiser input and the sole bridge anchor,
so a single bridge is matched per sample.
(b) Simplax samples $(\bw_t,\bz_t)$ jointly; $\bz_t$ remains the denoiser input, while
extra draws $\widetilde{\bz}_t\sim\Cat(\bw_t)$ anchor all $K$ bridges, which are matched
at once and marginalized exactly to give \cref{eq:rdb_closed_form}.%
}
\label{fig:bridge_comparison}
\end{figure}

The augmented hierarchy can be used at both training and sampling time.
For training, $\bw_t$ defines a distribution over auxiliary categorical
decodes $\widetilde{\bz}_t$, each inducing a discrete reverse bridge to be
matched (\cref{fig:bridge_comparison}). Averaging over these auxiliary bridge
anchors and marginalizing them analytically yields a tractable
Rao--Blackwellized objective. For generation, the same simplex state defines
a distribution over categorical alternatives that jointly contribute to the
reverse update, rather than conditioning the transition only on the currently
sampled token.

We evaluate Simplax on unconditional text generation with OpenWebText and
constrained categorical generation with Sudoku. On OpenWebText, Simplax
improves the generative perplexity--entropy tradeoff across a broad range of
inference budgets. On Sudoku, models are trained only on $30$-clue puzzles,
and Simplax achieves strong in-distribution performance while generalizing
to both easier and substantially harder clue densities, including $17$-clue
puzzles.

Our main contributions are as follows. First, we introduce an exact
Dirichlet--categorical augmentation of uniform discrete diffusion that
preserves the categorical process as a marginal while introducing an
auxiliary simplex-valued state. Second, we derive exact reverse-bridge
identities for the augmented process and use them to obtain a tractable
Rao--Blackwellized categorical reverse-bridge objective together with an
augmented reverse sampler. Third, we empirically demonstrate that the resulting
framework improves generation across both open-ended text modeling and
constrained categorical generation, including broad inference-budget regimes
and distribution shifts in Sudoku.

%% file: sec/2_preliminaries.tex
\section{Preliminaries}
\label{sec:preliminaries}

\noindent\textbf{Notation.}
Let
$
\Vocab
=
\{
\bx\in\{0,1\}^K
\;:\;
\sum_{i=1}^K x_i = 1
\}
$
denote the set of one-hot vectors over $K$ categories, and let
$\simplex^{K-1}$ denote the probability simplex over $K$ categories.
We represent scalar discrete random variables taking $K$ values as one-hot
column vectors in $\Vocab$.
We write
$\Cat(\cdot;\bpi)$
for the categorical distribution with class probabilities
$\bpi\in\simplex^{K-1}$.
Results involving Dirichlet densities assume
$\pi_k\nspace>\nspace0$ for every category $k$; the uniform base distribution
used in our experiments satisfies this condition.
We write
$\Dir(\cdot;\balpha)$
for the Dirichlet distribution with concentration vector
$\balpha\in\mathbb{R}_{>0}^K$.
We use
$\one\in\mathbb{R}^K$
for the all-ones vector,
$\inner{\ba}{\bb}$
for the inner product,
$\ba\mkern1mu\odot\mkern1mu\bb$
for the Hadamard product, and
$\ba\oslash\bb$
for elementwise division.
For sequences of length $L$, we write
$\bx^{1:L}\in\Vocab^L$.

\subsection{Dirichlet distribution}
The Dirichlet distribution is a distribution over the probability simplex
\citep{kotz2000continuous}.
Its mean is given by the normalized concentration vector, while the sum of the
concentration parameters controls how concentrated the distribution is around
its mean.
In particular, for
$\bp\in\simplex^{K-1}$
and
$\eta>0$,
$\Dir(\cdot;\eta\bp)$
denotes the Dirichlet distribution centered at $\bp$, with $\eta$ controlling
its concentration.

\subsection{Discrete diffusion}
We consider a discrete diffusion process on $t\in[0,1]$ with prior
$\bpi\in\simplex^{K-1}$ and noise schedule
$\alpha_t\in[0,1]$, following the categorical diffusion framework of
\citet{hoogeboom2021argmax,austin2021structured} and its formulation with a
general categorical stationary distribution \citep{zhao2025unified}.
Throughout, $0\leq s<t\leq1$, and simplex variables $\bw_u$ are defined for
$u>0$.

The noisy categorical state at time $t$ is distributed as
\begin{equation}
q(\bz_t \mid \bx)
=
\Cat\nspace\left(
\bz_t;
\pt{\bx}
\right),
\quad
\pt{\bx}
\coloneq
\alpha_t \bx + (1-\alpha_t)\bpi.
\label{eq:qzt_given_x}
\end{equation}

We use a schedule with $\alpha_0\nspace=\nspace1$,
$\alpha_1\nspace=\nspace0$, and
$0\nspace<\nspace\alpha_t\nspace<\nspace1$ for $t\in(0,1)$,
which ensures that $\bp_t(\bx)$ is strictly positive for all $t>0$.

The forward transition can be written as
\begin{equation}
q(\bz_t \mid \bz_s)
=
\Cat\nspace\left(
\bz_t;
\alpha_{t\mid s}\bz_s + (1-\alpha_{t\mid s})\bpi
\right),
\quad
\alpha_{t\mid s}
\coloneq
\frac{\alpha_t}{\alpha_s}.
\label{eq:forward_transition}
\end{equation}
The corresponding reverse posterior has the usual closed form
\begin{equation}
q(\bz_s \mid \bz_t,\bx)
=
\Cat\nspace\left(
\bz_s;
\rst{\bx}{\bz_t}
\right),
\label{eq:reverse_posterior}
\end{equation}
where
\begin{equation}
\rst{\bx}{\bz_t}
\coloneq
\frac{
\left[
\alpha_{t\mid s}\bz_t
+
(1-\alpha_{t\mid s})
\inner{\bz_t}{\bpi}
\one
\right]
\odot
\ps{\bx}
}{
\inner{\bz_t}{\pt{\bx}}
}.
\label{eq:reverse_posterior_vector}
\end{equation}

Standard discrete diffusion training minimizes the categorical reverse KL
divergence
\begin{equation}
\Loss{z_s\mid z_t}(\bz_t,\hat{\bx}_\theta,\bx)
=
\KL{q(\bz_s\mid \bz_t,\bx)}
{q(\bz_s\mid \bz_t,\hat{\bx}_\theta)}.
\label{eq:disc_kl}
\end{equation}
where
$\hat{\bx}_\theta\nspace=\nspace f_\theta(\bz_t,t)\in\simplex^{K-1}$
is the model prediction of the clean-token distribution.
We use the shorthand
$\bp_t\nspace\coloneq\nspace\pt{\bx}$ and
$\hat{\bp}_t\nspace\coloneq\nspace\pt{\hat{\bx}_\theta}$.

%% file: sec/3_method.tex
\section{Method}
\label{sec:method}

We construct Simplax by augmenting uniform discrete diffusion with an auxiliary
simplex-valued variable while leaving categorical corruption unchanged.
This defines an exact Dirichlet--categorical hierarchy over each corrupted
state, whose reverse bridge identities yield a tractable Rao--Blackwellized
training objective and a sampler from the same bridge structure.

\subsection{Simplex relaxation}
\label{sec:simplex_relaxation}

Building upon \cref{eq:forward_transition}, we consider the following joint
factorization:
\begin{equation}
q(\bx,\bz_s,\bz_t,\bw_s,\bw_t)
=
q(\bx)\,
q(\bz_s\mid \bx)\,
q(\bw_s\mid \bz_s,\bx)\,
q(\bz_t\mid \bz_s)\,
q(\bw_t\mid \bz_t,\bx).
\label{eq:joint_graphical_model}
\end{equation}

\cref{fig:graphical_model} illustrates the graphical model implied by this
factorization.

\begin{wrapfigure}{r}{0.36\textwidth}
\centering
\includegraphics[width=\linewidth]{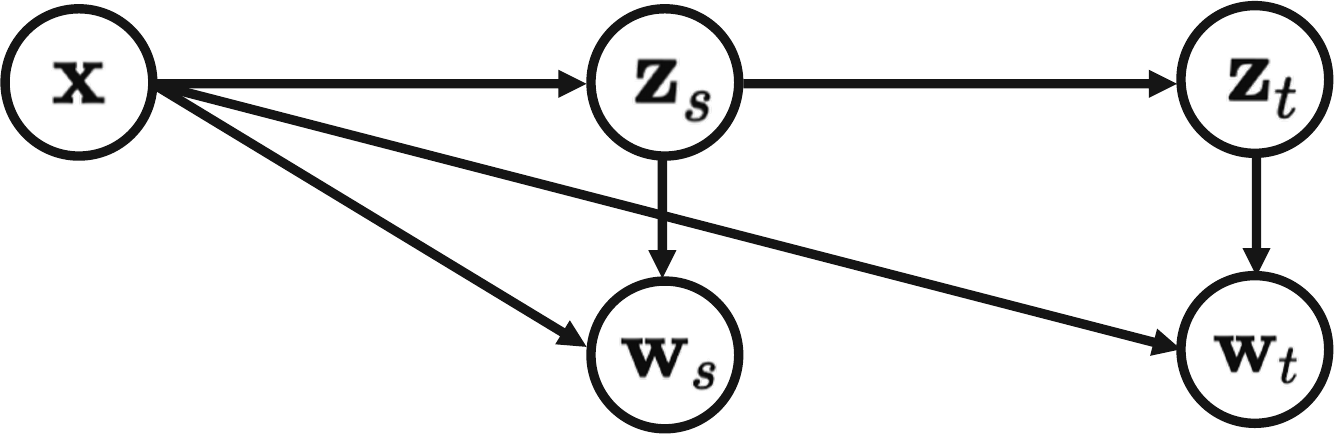}
\caption{Graphical model implied by the factorization in
\cref{eq:joint_graphical_model}.}
\label{fig:graphical_model}
\vspace{-0.2in}
\end{wrapfigure}

We introduce a simplex-valued variable
$\bw_t \in \simplex^{K-1}$ through
\begin{equation}
q(\bw_t \mid \bz_t,\bx)
=
\Dir\nspace\left(
\bw_t;
\eta_t \bp_t + \bz_t
\right),
\label{eq:dir_bridge}
\end{equation}
where $\eta_t\nspace>\nspace0$ is a concentration parameter.
Its conditional mean interpolates between the diffusion-time categorical
marginal $\bp_t$ and the sampled discrete token $\bz_t$, with the additive
one-hot count anchoring the relaxed state to $\bz_t$.

This construction yields a Dirichlet--categorical hierarchy.
The relaxed and discrete states satisfy
\begin{equation}
q(\bw_t \mid \bx)
=
\Dir\nspace\left(
\bw_t;
\eta_t \bp_t
\right),
\quad
q(\bz_t \mid \bw_t,\bx)
=
q(\bz_t \mid \bw_t)
=
\Cat\nspace\left(
\bz_t;
\bw_t
\right).
\label{eq:w_marginal_z_given_w}
\end{equation}

Moreover, the reverse conditional posterior of $\bz_s$ given $\bw_t$ and $\bx$ is
categorical:
\begin{equation}
q(\bz_s \mid \bw_t,\bx)
=
\Cat\nspace\left(
\bz_s;
\rhost{\bx}{\bw_t}
\right),
\label{eq:zs_given_wtx}
\end{equation}
where
\begin{equation}
\rhost{\bx}{\bw_t}
\coloneq
\bp_s
\odot
\left[
\alpha_{t\mid s}
\bigl(\bw_t \oslash \bp_t\bigr)
+
(1-\alpha_{t\mid s})
\,
\inner{\bw_t}{\bpi \oslash \bp_t}\,
\one
\right].
\label{eq:rho_def}
\end{equation}

In turn, the reverse conditional posterior of $\bw_s$ given $\bw_t$ and $\bx$ is a
Dirichlet mixture:
\begin{equation}
q(\bw_s \mid \bw_t,\bx)
=
\sum_{k=1}^K
\rho_{s\mid t,k}(\bx,\bw_t)\,
\Dir\nspace\left(
\bw_s;
\eta_s \bp_s + \ek
\right).
\label{eq:ws_given_wtx}
\end{equation}

See \cref{app:full_hierarchy} for the proof.
These identities show that $\bw_t$ is not an ad hoc surrogate.
It is an exact auxiliary variable whose marginal remains native to the simplex
and whose decoder back to $\bz_t$ is simply categorical sampling from $\bw_t$.

\subsection{Training objective}
\label{sec:training_objective}

A natural starting point is to match the simplex bridge directly:
\begin{equation}
\Loss{w_s \mid w_t}(\bw_t,\hat{\bx}_\theta,\bx; s,t)
=
\KL{q(\bw_s\mid \bw_t,\bx)}
{q(\bw_s\mid \bw_t,\hat{\bx}_\theta)}.
\label{eq:direct_w_kl}
\end{equation}
However, this objective is generally intractable because
$q(\bw_s\!\mid\!\bw_t,\bx)$ is a Dirichlet mixture, as shown in
\cref{eq:ws_given_wtx}. We therefore move back to the categorical space and
match discrete reverse bridges induced by categorical decodes from $\bw_t$.


To do so while retaining a categorical denoiser input, we first sample the
augmented noisy state as
\begin{align}
\bw_t
\sim
q(\bw_t\mid\bx),
\qquad
\bz_t
\sim
q(\bz_t\mid\bw_t)
=
\Cat(\bz_t;\bw_t),
\end{align}
and predict the clean-token distribution
$\hat{\bx}_\theta\nspace=\nspace f_\theta(\bz_t,t)$.

We then introduce an auxiliary categorical decode
$\widetilde{\bz}_t$ satisfying
\begin{align}
\widetilde{\bz}_t
\sim
q(\widetilde{\bz}_t\mid\bw_t)
=
\Cat(\widetilde{\bz}_t;\bw_t),
\end{align}
conditionally independently of the network input $\bz_t$ given $\bw_t$.
We optimize
\begin{equation}
\begin{aligned}
\bLoss{z_s\mid z_t,w_t}
(
\bw_t,
\hat{\bx}_\theta,
\bx;
s,t
)
\coloneq
\mathbb{E}_{
q(\widetilde{\bz}_t\mid\bw_t)
}
\Big[
\KL
{
q(\bz_s\mid\widetilde{\bz}_t,\bx)}
{
q(\bz_s\mid\widetilde{\bz}_t,\hat{\bx}_\theta)
}
\Big].
\end{aligned}
\label{eq:rdb_objective}
\end{equation}
This objective averages the standard discrete reverse KL divergence over an
auxiliary decoder sample from $\bw_t$, while retaining $\bz_t$ as the denoiser
input (\cref{fig:bridge_comparison}).

\noindent\textbf{Rao--Blackwellized form.}
The expectation over the auxiliary decoder sample $\widetilde{\bz}_t$ in
\cref{eq:rdb_objective} can be marginalized exactly:
\begin{equation}
\bLoss{z_s \mid z_t, w_t}(\bw_t,\hat{\bx}_\theta,\bx; s,t)
=
\inner{\bw_t}{\log \hat{\bp}_t - \log \bp_t}
+
\inner{\rhost{\bx}{\bw_t}}{\log \bp_s - \log \hat{\bp}_s}.
\label{eq:rdb_closed_form}
\end{equation}
This is the Rao--Blackwellized form of the Monte Carlo estimator that samples
$\widetilde{\bz}_t \sim q(\widetilde{\bz}_t\mid \bw_t)$ before evaluating the
discrete reverse KL. The independently sampled $\bz_t$ remains the denoiser
input and is not marginalized. The proof is given in \cref{app:rao_blackwellized_objective}.
The closed form eliminates sampling noise from $\widetilde{\bz}_t$ and depends
on $\bw_t$ only through
$\inner{\bw_t}{\log \hat{\bp}_t}$ and
$\rhost{\bx}{\bw_t}$.
Thus, a single prediction from the realized input $\bz_t$ suffices to evaluate
the bridge contributions of all auxiliary decodes.

\noindent\textbf{Continuous-time limit.}
The closed form in \cref{eq:rdb_closed_form} yields a non-degenerate
infinitesimal limit.
Let $s=t-\Delta$ with $\Delta \downarrow 0$.
Up to $\theta$-independent additive terms, the relaxed discrete bridge
objective satisfies
\begin{equation}
\bLoss{z_s \mid z_t, w_t}(\bw_t,\hat{\bx}_\theta,\bx; t-\Delta,t)
=
\Delta\,
\ell_{\mathrm{ct}}(\bw_t,\hat{\bx}_\theta,\bx,t)
+
o(\Delta),
\label{eq:ct_expansion}
\end{equation}
where
$\lambda(t)\coloneq-\frac{\mathrm{d}}{\mathrm{d}t}\log\alpha(t)$
and
{\small
\begin{equation}
\ell_{\mathrm{ct}}(\bw_t,\hat{\bx}_\theta,\bx,t)
=
\lambda(t)\Big[
\inner{\bw_t}{\bpi \oslash \hat{\bp}_t}
-
\inner{\bw_t}{\bpi \oslash \bp_t}\,
\inner{\bp_t}{\log \hat{\bp}_t}
+
\inner{\bpi \odot (\bw_t \oslash \bp_t)}{\log \hat{\bp}_t}
\Big].
\label{eq:ct_density_final}
\end{equation}
}
The corresponding continuous-time objective is
\begin{equation}
\mathcal{L}_{\mathrm{ct}}
=
\int_0^1
\EE{q(\bx)}{
\EE{q(\bw_t\mid \bx)}{
\EE{q(\bz_t\mid \bw_t)}{
\ell_{\mathrm{ct}}
\bigl(
\bw_t,
f_\theta(\bz_t,t),
\bx,
t
\bigr)
}
}
}
\,\mathrm{d}t.
\label{eq:ct_objective}
\end{equation}

The proof is deferred to \cref{app:proof_continuous_time_limit}.
\cref{eq:ct_objective} is therefore the continuous-time counterpart
of \cref{eq:rdb_objective}.

The structure of \cref{eq:rdb_objective,eq:ct_objective} is closely related to
UDLM~\citep{schiff2025simple}.
UDLM derives a continuous-time reverse-KL objective directly from the discrete
corrupted state \cref{eq:disc_kl}.
Our construction introduces the exact auxiliary variable $\bw_t$, averages
the same reverse-KL bridge over an auxiliary draw
$\widetilde{\bz}_t\sim q(\widetilde{\bz}_t\mid \bw_t)
\nspace=\nspace\Cat(\bw_t)$, and then takes the infinitesimal limit.
In this sense, \cref{eq:ct_objective} can be understood as a simplex-relaxed
continuous-time analogue of the UDLM objective.

\noindent\textbf{Alternative bridge objectives.}
One might instead directly match the marginal categorical bridge
$q(\bz_s\!\mid\!\bw_t,\bx)$ or the full augmented bridge
$q(\bz_s,\bw_s\!\mid\!\bw_t,\bx)$. More generally, different choices of the
conditioning variables and the matched bridge lead to distinct
objectives. For $s=t-\Delta$, direct marginal matching is $O(\Delta^2)$ and
therefore yields a degenerate continuous-time limit, whereas full augmented
matching contains an $O(1)$ Dirichlet-matching term and diverges under
continuous-time scaling. In contrast, our objective is $O(\Delta)$ and admits
a finite, nonzero continuous-time limit. We summarize these objectives, their
inequality relations, and their continuous-time limits in
\cref{app:alternative_objectives}.

\subsection{Sampling}
\label{sec:sampling}

At inference time, we run the reverse process on a grid
$1=t_N\nspace>\nspace t_{N-1}\nspace>\nspace\cdots\nspace>\nspace t_0\nspace=0$.
For a denoiser conditioned on $\bz_t$, the Dirichlet--categorical hierarchy
defines a reverse process over the augmented state $(\bz_t,\bw_t)$.
Since $q(\bw_1)\nspace=\nspace\Dir(\bw_1;\eta_1\bpi)$ and
$q(\bz_1\mid\bw_1)\nspace=\nspace\Cat(\bz_1;\bw_1)$, generation starts by
sampling $\bw_{t_N}\sim\Dir(\eta_{t_N}\bpi)$ followed by
$\bz_{t_N}\sim\Cat(\bw_{t_N})$.
Given adjacent times
$s\nspace=\nspace t_{n-1}\nspace<\nspace t\nspace=\nspace t_n$
and the current state $(\bz_t,\bw_t)$, the denoiser predicts
$\hat{\bx}_\theta\nspace=\nspace f_\theta(\bz_t,t)$, inducing the bridge
marginals $\hat{\bp}_s$ and $\hat{\bp}_t$ and the marginal reverse categorical
distribution $\rhost{\hat{\bx}_\theta}{\bw_t}$.

\noindent\textbf{Reverse sampler and variants.}
We sample from this marginal reverse distribution and then update the auxiliary
simplex state:
\begin{equation}
\bz_s
\sim
\Cat\nspace\left(
\rhost{\hat{\bx}_\theta}{\bw_t}
\right),
\qquad
\bw_s
\sim
\Dir\nspace\left(
\eta_s\hat{\bp}_s+\bz_s
\right).
\label{eq:stochastic_sampler}
\end{equation}
Here, $\bz_s$ serves as the next denoiser input, while $\bw_s$ carries the
simplex-valued bridge state.

The update contains two stochastic transitions, for $\bz_s$ and $\bw_s$,
yielding four basic variants depending on which transitions are
determinized.
The \textit{stochastic} sampler draws both as in
\cref{eq:stochastic_sampler}.
The $\bw$-\textit{deterministic} variant samples $\bz_s$ but replaces the Dirichlet draw
for $\bw_s$ by its conditional mean, whereas the $\bz$-\textit{deterministic} variant
replaces the categorical draw for $\bz_s$ by its mode while retaining the
Dirichlet draw for $\bw_s$.
The fully \textit{deterministic} variant applies both replacements.
Explicit update rules and derivations are given in
\cref{app:sampler_variants}, and the variants are compared in
\cref{tab:sampler_ablation}.

%% file: sec/4_experiments.tex
\section{Experiments}
\label{sec:experiments}

We evaluate Simplax on unconditional OpenWebText generation
\citep{gokaslan2019openwebtext} and constrained categorical Sudoku generation
\citep{lee2026flm,sflm}.

\begin{table*}[t]
  \centering
  \caption{
  OpenWebText unconditional generation at selected NFE values. Gen.\ ENT is
  generative unigram entropy and should be compared with the data entropy
  $5.44$. Gen.\ PPL is evaluated by the indicated external model.
  Time is median seconds per batch of 16 sequences of length 1024.
  The best and second-best values among diffusion models are shown in
  bold and underlined, respectively.
  }
  \label{tab:owt_generation_selected_nfe}
  \vspace{-0.1in}\scriptsize
  \setlength{\tabcolsep}{2.0pt}
  \resizebox{\textwidth}{!}{
  \begin{tabular}{@{}lcccccccccccc@{}}
    \toprule
    & \multicolumn{4}{c}{\textbf{NFE = 16}}
    & \multicolumn{4}{c}{\textbf{NFE = 128}}
    & \multicolumn{4}{c}{\textbf{NFE = 1024}} \\
    \cmidrule(lr){2-5}
    \cmidrule(lr){6-9}
    \cmidrule(lr){10-13}
    Method
      & ENT & GPT-2 L & Llama-2 & Time (s)
      & ENT & GPT-2 L & Llama-2 & Time (s)
      & ENT & GPT-2 L & Llama-2 & Time (s) \\
    \midrule
    AR
      & - & - & - & -
      & - & - & - & -
      & 5.43 & 22.8 & 16.0 & 8.46 \\
    \midrule
    CANDI~\citep{pynadath2025candi}
      & 5.43 & \underline{97.2} & 56.0 & \underline{0.648}
      & 5.45 & 67.2 & 36.8 & \underline{4.97}
      & 5.46 & 74.3 & 39.2 & \underline{39.7} \\
    UDLM~\citep{schiff2025simple}
      & 5.53 & 186.0 & 79.0 & 0.677
      & 5.49 & 62.9 & 34.6 & 5.53
      & 5.46 & 59.2 & 32.2 & 44.5 \\
    MDLM~\citep{sahoo2024simple}
      & 5.47 & 117.9 & 63.8 & \textbf{0.548}
      & 5.46 & 65.4 & 37.3 & \textbf{4.38}
      & 5.40 & 55.1 & 33.9 & \textbf{35.4} \\
    Duo~\citep{sahoo2025duality}
      & 5.45 & 164.0 & 86.0 & 0.677
      & 5.45 & 94.4 & 52.6 & 5.53
      & 5.43 & 88.2 & 49.5 & 44.5 \\
    FLM~\citep{lee2026flm}
      & 5.58 & 259.5 & 139.3 & 0.896
      & 5.42 & 112.0 & 66.5 & 7.12
      & 5.45 & 125.6 & 74.4 & 57.0 \\
    LangFlow~\citep{langflow}
      & 5.42 & 115.1 & 59.9 & 1.15
      & 5.42 & \underline{60.2} & \textbf{30.0} & 9.38
      & 5.41 & 68.3 & \underline{28.2} & 75.3 \\
    S-FLM~\citep{sflm}
      & 5.45 & 124.6 & 62.9 & 1.18
      & 5.43 & 103.4 & 52.4 & 9.98
      & 5.46 & 108.5 & 53.7 & 80.5 \\
    \midrule
    \textbf{Simplax}
      & 5.46 & 98.7 & \underline{54.2} & 0.814
      & 5.46 & 63.4 & 35.0 & 6.09
      & 5.46 & \underline{52.6} & 29.8 & 48.4 \\
    \textbf{Simplax (UDLM init)}
      & 5.42 & \textbf{81.5} & \textbf{45.2} & 0.814
      & 5.45 & \textbf{56.5} & \underline{31.1} & 6.09
      & 5.42 & \textbf{42.8} & \textbf{24.4} & 48.4 \\
    \bottomrule
  \end{tabular}
  }\vspace{-0.1in}
\end{table*}
\subsection{Experimental setup}
\label{sec:exp_settings}

For the main Simplax results, we use $\bz_t$ as the denoiser input, constant
concentration $\eta_t\equiv0.01$, and the $\bw$-deterministic reverse sampler.
These choices are examined in \cref{sec:owt_contribution_analysis}.

\noindent\textbf{OpenWebText.}
We tokenize OpenWebText with the GPT-2 BPE tokenizer
\citep{radford2019language}, giving $|\mathcal{V}|=50{,}257$, and use sequence
length $L=1024$.
All methods use the same $179\mathrm{M}$-parameter diffusion transformer
backbone of \citet{sahoo2024simple} and are trained for a total budget of
$1\mathrm{M}$ optimization steps.
Preprocessing, model architecture, optimization, checkpoint selection,
sampling, and computational-resource details are provided in
\cref{app:owt_experiment}.

\noindent\textbf{Sudoku.}
We build on the Sudoku benchmark of \citet{sflm} and evaluate
cross-clue generalization.
All models are trained only on puzzles with $30$ clues and are evaluated
using the same checkpoint on puzzles with $40$, $35$, $30$, $25$, $20$,
and $17$ clues.
Thus, the evaluation measures transfer to both more heavily and more
sparsely conditioned inputs than those seen during training.
The $17$-clue setting is particularly challenging, as $17$ is the minimum
number of clues for which a standard $9\times9$ Sudoku puzzle can admit a
unique solution \citep{mcguire2014no16,lin2013specific}.
We additionally evaluate unconditional Sudoku generation using an all-blank
puzzle prefix.
Dataset construction, sequence representation, model architectures,
optimization, and computational-resource details are provided in
\cref{app:sudoku_experiments}.

\noindent\textbf{Metrics.}
For OpenWebText, we draw $1024$ sequences and report generative unigram entropy
and generative perplexity under GPT-2 Large and Llama-2 7B
\citep{radford2019language,touvron2023llama}; GPT-2 XL results are reported
in \cref{sec:additional_owt_results}.
The OpenWebText unigram entropy is $5.44$ nats.
For all OpenWebText evaluations, we sweep $15$ temperatures from $0.84$ to
$1.12$ in increments of $0.02$ for each configuration and NFE budget, and
select the temperature whose generated entropy is closest to $5.44$ nats.
We also report sampling time at selected NFE budgets; runtime and peak GPU
memory are in \cref{sec:owt_runtime_memory}.
For conditional Sudoku, we report solving accuracy.
For unconditional Sudoku, we report validity, the fraction of generated boards
satisfying all Sudoku constraints.

\subsection{Unconditional generation on OpenWebText}
\label{sec:exp_owt}

We compare Simplax with AR, MDLM \citep{sahoo2024simple}, UDLM
\citep{schiff2025simple}, Duo \citep{sahoo2025duality}, CANDI
\citep{pynadath2025candi}, FLM \citep{lee2026flm}, LangFlow \citep{langflow},
and S-FLM \citep{sflm}.
We also evaluate a staged variant that trains UDLM for $800\mathrm{k}$ steps
and Simplax for $200\mathrm{k}$ steps under the same
$1\mathrm{M}$-step budget; other allocations are reported in
\cref{sec:owt_budget_allocation}.
We compare with ELF \citep{elf} separately in \cref{sec:owt_elf}, as its sweep does not reach the OpenWebText data entropy required for the
matched-entropy comparison.

\begin{wrapfigure}{R}{0.60\columnwidth}
  \vspace{-0.9\baselineskip}
  \centering
  \includegraphics[width=\linewidth]{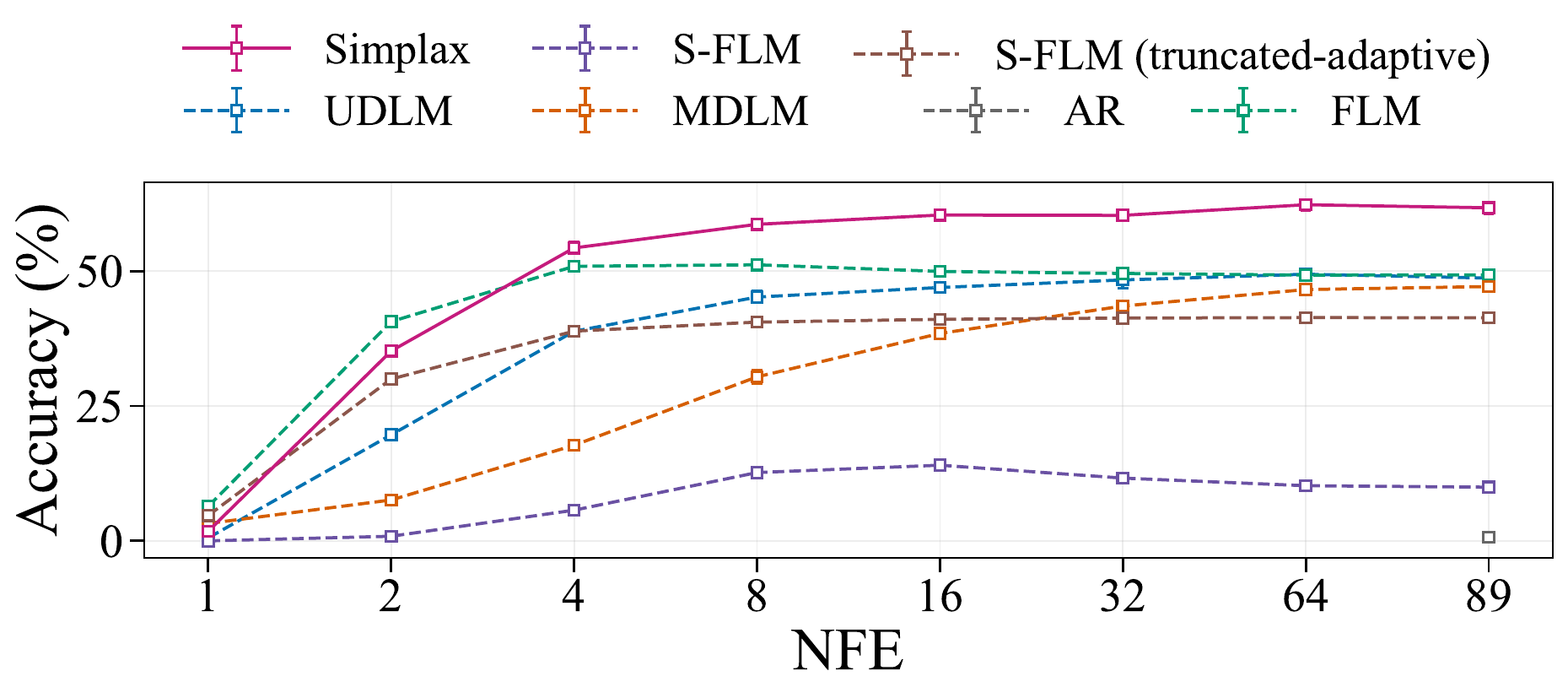}
  \caption{Few-step generation with 30 clues for both training and inference.}
  \label{fig:sudoku_few_steps}
  \vspace{-0.7\baselineskip}
\end{wrapfigure}
\

At the entropy-matched operating points, Simplax trained from scratch remains
competitive across inference budgets, while UDLM initialization further
reduces generative perplexity (\cref{tab:owt_generation_selected_nfe}).
The UDLM-initialized variant gives the lowest PPL among diffusion methods at
low and high NFE under all evaluators and remains close to LangFlow at the
intermediate budget.
The full temperature sweeps in \cref{fig:frontier_llama2_7b_grid} show the same
pattern beyond these operating points, with Simplax on or near the lower
PPL--entropy envelope across temperatures and NFE values.
The benefit of UDLM initialization further suggests that staged allocation of a fixed training budget can be beneficial: categorical diffusion
pretraining followed by Simplax training performs better than either objective
used throughout. This is similar in spirit to the staged training strategy of
Duo.

Simplax also has moderate sampling cost
(\cref{tab:owt_generation_selected_nfe}).
It is faster than FLM, LangFlow, and S-FLM, whose sampler updates require
dense matrix products involving the vocabulary embedding matrix.
It is slower than MDLM, UDLM, Duo, and CANDI, consistent with the cost of
maintaining dense $\boldsymbol{\rho}$ during Simplax sampling.
Thus, Simplax lies between categorical diffusion and continuous-state or
flow-based methods in sampling runtime while providing the improved
PPL--entropy tradeoff above.

Additional analyses are provided in the appendix. Bootstrap uncertainty estimates and matched-entropy GPT-2 XL results are reported in \cref{sec:additional_owt_results}, with GPT-2 Large and GPT-2 XL PPL--entropy frontiers in \cref{fig:frontier_gpt2_large_grid,fig:frontier_gpt2_xl_grid}. Detailed runtime and memory-usage procedures and results appear in \cref{sec:owt_runtime_memory,tab:owt_runtime_memory}. Representative generations at NFE $=16$, $128$, and $1024$ appear in \cref{sec:owt_qualitative}.

\begin{figure*}[t]
  \centering
  \includegraphics[width=\textwidth]{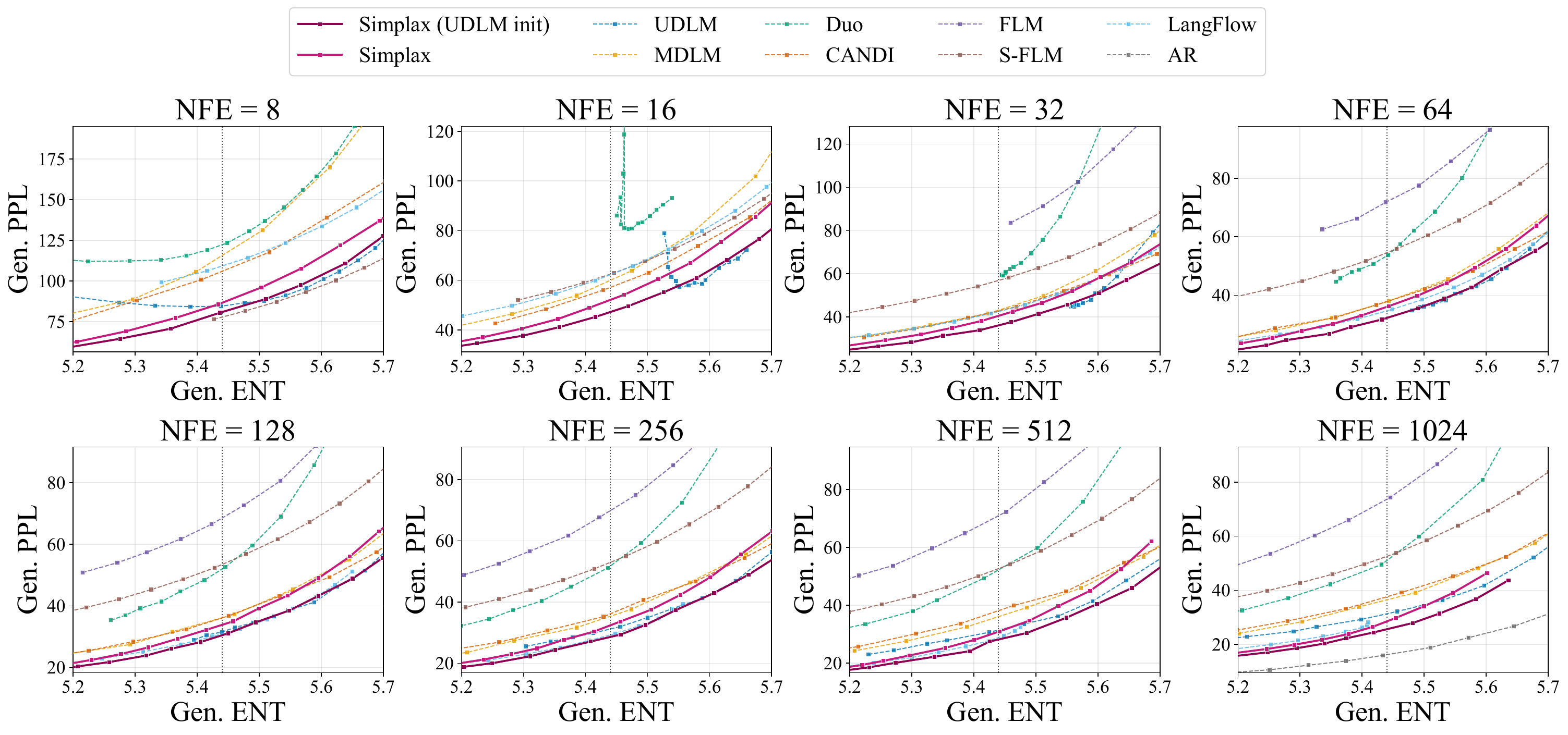}
  \vspace{-0.25in}
  \caption{
  Llama-2 7B generative frontiers on OpenWebText for
  $\mathrm{NFE}\in\{8,16,32,64,128,256,$ $512,1024\}$. Each panel shows the
  temperature-swept Gen.\ PPL--Gen.\ ENT tradeoff. Lower Gen.\ PPL is better,
  and the reference data entropy is $5.44$.
  }
  \label{fig:frontier_llama2_7b_grid}
\end{figure*}

\subsection{Constrained categorical generation on Sudoku}
\label{sec:exp_sudoku}

\begin{table*}[t]
    \centering
    \caption{
      Conditional Sudoku solving accuracy and unconditional Sudoku validity in
      percent.
      All models are trained with $30$ clues.
      We report the mean $\pm$ sample standard deviation over five sampling
      seeds using a fixed checkpoint and the same $2{,}000$ examples per seed.
    }
    \label{tab:sudoku_results}
    \vspace{-0.1in}
    \resizebox{\columnwidth}{!}{%
    \begin{tabular}{@{}lccccccc@{}}
      \toprule
      & \multicolumn{6}{c}{Conditional accuracy (\%)}
      & \multicolumn{1}{c}{Uncond. validity (\%)} \\
      \cmidrule(lr){2-7}
      \cmidrule(lr){8-8}
      Method
        & 40 clues
        & 35 clues
        & 30 clues
        & 25 clues
        & 20 clues
        & 17 clues
        & 0 clues \\
      \midrule
      AR
        & $16.44 \pm 0.38$
        & $4.08 \pm 0.39$
        & $0.67 \pm 0.04$
        & $0.08 \pm 0.03$
        & $0.00 \pm 0.00$
        & $0.00 \pm 0.00$
        & $8.48 \pm 0.67$ \\
      \midrule
      UDLM
        & $97.02 \pm 0.27$
        & $84.24 \pm 0.36$
        & $48.69 \pm 0.64$
        & \underline{$16.78 \pm 0.40$}
        & \underline{$5.44 \pm 0.53$}
        & \underline{$0.58 \pm 0.17$}
        & \underline{$80.71 \pm 1.07$} \\

      FLM
        & $96.86 \pm 0.27$
        & $83.96 \pm 0.12$
        & \underline{$49.27 \pm 0.74$}
        & $15.06 \pm 0.96$
        & $3.74 \pm 0.17$
        & $0.34 \pm 0.21$
        & $33.48 \pm 0.44$ \\

      MDLM
        & \underline{$98.10 \pm 0.27$}
        & \underline{$85.76 \pm 0.84$}
        & $47.12 \pm 0.68$
        & $11.87 \pm 0.53$
        & $2.49 \pm 0.72$
        & $0.23 \pm 0.07$
        & $22.02 \pm 0.75$ \\

      S-FLM
        & $77.28 \pm 0.62$
        & $42.87 \pm 0.37$
        & $9.96 \pm 0.77$
        & $0.96 \pm 0.16$
        & $0.04 \pm 0.02$
        & $0.00 \pm 0.00$
        & $1.52 \pm 0.20$ \\

      S-FLM (trunc.)
        & $94.99 \pm 0.32$
        & $79.00 \pm 0.32$
        & $41.34 \pm 0.38$
        & $11.53 \pm 0.26$
        & $3.08 \pm 0.45$
        & $0.29 \pm 0.11$
        & $4.58 \pm 0.84$ \\

      \midrule
      \textbf{Simplax}
        & $\mathbf{98.37 \pm 0.26}$
        & $\mathbf{90.53 \pm 0.45}$
        & $\mathbf{61.70 \pm 1.08}$
        & $\mathbf{24.69 \pm 1.04}$
        & $\mathbf{8.38 \pm 0.87}$
        & $\mathbf{0.82 \pm 0.13}$
        & $\mathbf{95.30 \pm 0.37}$ \\
      \bottomrule
    \end{tabular}
    }\vspace{-0.1in}
\end{table*}

We compare Simplax with AR, UDLM \citep{schiff2025simple},
FLM \citep{lee2026flm}, MDLM \citep{sahoo2024simple},
S-FLM \citep{sflm}, and truncated-adaptive S-FLM.

Simplax achieves the highest conditional accuracy across all clue densities
and the highest unconditional validity (\cref{tab:sudoku_results}).
The differences are small when many clues are observed, but become larger as
the conditioning information is reduced.
The same pattern extends to unconditional generation, where Simplax reaches
$95.30\%$ validity compared with $80.71\%$ for the strongest baseline.
Together, these results suggest that the improvement extends beyond exploiting
the observed clues to modeling the global constraints of valid Sudoku grids.

The few-step results in \cref{fig:sudoku_few_steps} show that this advantage
also appears under a limited sampling budget.
With four or more function evaluations, Simplax achieves the highest accuracy
among all evaluated methods and maintains this lead as the sampling budget
increases.
The separation therefore emerges after only a few refinement steps rather
than only in the large-NFE regime, suggesting that Simplax makes more effective
use of each function evaluation during generation.

Representative Sudoku instances in \cref{fig:sudoku-25clue-comparison} provide visual examples of conditional generation results. Additional few-step results under different inference-time clue counts are provided in \cref{sec:sudoku_additional_few_steps}.
\WFclear

\subsection{Contribution analysis of training and sampling}
\label{sec:owt_contribution_analysis}

Beyond the generation results, we examine which Simplax components drive the
improvement and which design choices matter. We study the denoiser input,
reverse sampler, bridge concentration, and training--sampling interaction.
We report GPT-2 Large Gen.\ PPL / Gen.\ ENT at the
temperature closest to the data entropy of $5.44$ on OpenWebText. Additional details are
provided in \cref{sec:contribution_analysis_frontiers}.

\begin{table*}[t]
  \centering


  \begin{minipage}[t]{0.40\textwidth}
    \vspace{0pt}
    \centering

    \begin{minipage}[t]{0.98\linewidth}
      \vspace{0pt}
      \centering

      \caption{\textbf{Input.} Simplax $50\mathrm{k}$, $\eta=0.01$, with
      stochastic sampling; each input is used in training and inference.}
      \label{tab:input_ablation}
      \scriptsize
      \setlength{\tabcolsep}{2.5pt}
      \begin{tabularx}{\linewidth}{
        @{}
        l
        >{\centering\arraybackslash}X
        >{\centering\arraybackslash}X
        >{\centering\arraybackslash}X
        @{}
      }
        \toprule
        & NFE $=16$ & NFE $=128$ & NFE $=1024$ \\
        \cmidrule(lr){2-2}
        \cmidrule(lr){3-3}
        \cmidrule(lr){4-4}
        Input & PPL / ENT & PPL / ENT & PPL / ENT \\
        \midrule
        $\bw_t$
          & $113.1 / 5.43$
          & $87.6 / 5.46$
          & $94.3 / 5.46$ \\
        \rowcolor{black!7}
        $\bz_t$
          & $105.1 / 5.43$
          & $73.6 / 5.46$
          & $56.2 / 5.43$ \\
        \bottomrule
      \end{tabularx}
    \end{minipage}
  \end{minipage}
  \hfill
  \begin{minipage}[t]{0.58\textwidth}
  \vspace{0pt}
  \centering

  \begin{minipage}[t]{0.98\linewidth}
    \vspace{0pt}
    \centering

    \caption{\textbf{Reverse-sampler quality and runtime.} Simplax
    $50\mathrm{k}$, $\eta=0.01$, and $\bz_t$ input.
    Runtime is measured at NFE $=1024$.}
    \label{tab:sampler_ablation}
    \scriptsize
    \setlength{\tabcolsep}{2.0pt}
    \begin{tabularx}{\linewidth}{
      @{}
      l
      >{\centering\arraybackslash}X
      >{\centering\arraybackslash}X
      >{\centering\arraybackslash}X
      >{\centering\arraybackslash}X
      @{}
    }
      \toprule
      Sampler
        & NFE $=16$
        & NFE $=128$
        & NFE $=1024$
        & Time (s) \\
      \cmidrule(lr){2-4}
        & PPL/ENT
        & PPL/ENT
        & PPL/ENT
        & \\
      \midrule
      \rowcolor{black!7}
      $\bw$-det.
        & $106.0/5.44$
        & $75.1/5.45$
        & $60.7/5.46$
        & $48.4$ \\
      $\bz$-det.
        & $20.8/1.07^{\dagger}$
        & $445.9/2.92^{\dagger}$
        & $694.8/4.79^{\dagger}$
        & $85.5$ \\
      Stochastic
        & $105.1/5.43$
        & $73.6/5.46$
        & $56.2/5.43$
        & $103.3$ \\
      Deterministic
        & $14.7/0.91^{\dagger}$
        & $14.8/0.91^{\dagger}$
        & $14.8/0.91^{\dagger}$
        & $38.8$ \\
      \bottomrule
    \end{tabularx}
  \end{minipage}
\end{minipage}

\vspace{0.15in}

  \begin{minipage}[b]{0.40\textwidth}
    \centering

    \begin{minipage}[b]{0.98\linewidth}
      \centering

      \caption{\textbf{Concentration.} Simplax $50\mathrm{k}$,
      stochastic sampling, $\bz_t$ input.}
      \label{tab:eta_ablation}
      \scriptsize
      \setlength{\tabcolsep}{2.5pt}
      \begin{tabularx}{\linewidth}{
        @{}
        l
        >{\centering\arraybackslash}X
        >{\centering\arraybackslash}X
        >{\centering\arraybackslash}X
        @{}
      }
        \toprule
        & NFE $=16$ & NFE $=128$ & NFE $=1024$ \\
        \cmidrule(lr){2-2}
        \cmidrule(lr){3-3}
        \cmidrule(lr){4-4}
        $\eta$ & PPL / ENT & PPL / ENT & PPL / ENT \\
        \midrule
        $0$
          & $103.9 / 5.44$
          & $81.8 / 5.47$
          & $78.7 / 5.45$ \\
        $10^{-3}$
          & $103.6 / 5.43$
          & $80.2 / 5.45$
          & $75.4 / 5.43$ \\
        \rowcolor{black!7}
        $10^{-2}$
          & $105.1 / 5.43$
          & $73.6 / 5.46$
          & $56.2 / 5.43$ \\
        $10^{-1}$
          & $132.4 / 5.43$
          & $69.1 / 5.45$
          & $32.3 / 5.15^{\dagger}$ \\
        \bottomrule
      \end{tabularx}
    \end{minipage}
  \end{minipage}
  \hfill
  \begin{minipage}[b]{0.58\textwidth}
    \centering

    \begin{minipage}[b]{0.94\linewidth}
      \centering

      \caption{\textbf{Training--sampling decomposition.}
      UDLM vs Simplax,
      $\eta=0.01$, $\bz_t$ input, $\bw$-det. sampler.}
      \label{tab:training_sampling_decomposition}
      \scriptsize
      \setlength{\tabcolsep}{2.5pt}
      \begin{tabularx}{\linewidth}{
        @{}
        l
        l
        >{\centering\arraybackslash}X
        >{\centering\arraybackslash}X
        >{\centering\arraybackslash}X
        @{}
      }
        \toprule
        Training & Sampler
          & NFE $=16$
          & NFE $=128$
          & NFE $=1024$ \\
        \cmidrule(lr){3-3}
        \cmidrule(lr){4-4}
        \cmidrule(lr){5-5}
        & & PPL / ENT & PPL / ENT & PPL / ENT \\
        \midrule
        UDLM
          & UDLM
          & $186.0 / 5.53^{\dagger}$
          & $62.9 / 5.49$
          & $59.2 / 5.46$ \\
        UDLM
          & Simplax
          & $82.8 / 5.42$
          & $80.4 / 5.46$
          & $217.8 / 5.44$ \\
        Simplax
          & UDLM
          & $184.3 / 5.48$
          & $54.8 / 5.43$
          & $55.4 / 5.43$ \\
        \rowcolor{black!7}
        Simplax
          & Simplax
          & $81.5 / 5.42$
          & $56.5 / 5.45$
          & $42.8 / 5.42$ \\
        \bottomrule
      \end{tabularx}
    \end{minipage}
  \end{minipage}

  \vspace{0.5em}

  \begin{minipage}{\textwidth}
    \footnotesize
    $\dagger$ marks sweeps that do not reach the data entropy
    $5.44 \pm 0.05$; their PPL is not directly comparable at matched entropy.
    \colorbox{black!7}{Shaded rows} indicate the design choices used in the
    main experiments.
  \end{minipage}
\end{table*}
\noindent\textbf{Training input.} We compare Simplax trained and evaluated with $\bz_t$ input against Simplax trained and evaluated  with $\bw_t$ input. 
The $\bz_t$ input yields a consistently better frontier, with the largest margin at NFE $=1024$ (\cref{tab:input_ablation}). We therefore adopt $\bz_t$, which also avoids the vocabulary-sized dense multiplication $\bw_t^{\mathsf T}E$ required by $\bw_t$ input.

\noindent\textbf{Reverse sampler.}
We compare the \textit{stochastic}, \textit{$\bw$-deterministic}, \textit{$\bz$-deterministic}, and \textit{deterministic} samplers introduced in \cref{sec:sampling,app:sampling_details} in terms of entropy-matched PPL and runtime. Runtime evaluation details are provided in \cref{sec:contribution_analysis_frontiers,sec:owt_runtime_memory}. The \textit{stochastic} and \textit{$\bw$-deterministic} frontiers nearly overlap across NFEs, while the latter reduces generation time (\cref{tab:sampler_ablation}). The \textit{deterministic} and \textit{$\bz$-deterministic} samplers fail to reach the data entropy. We therefore use \textit{$\bw$-deterministic} for its similar quality and lower sampling cost.

\noindent\textbf{Bridge concentration.}
We compare $\eta\in\{0,0.001,0.01,0.1\}$, with $\eta=0$ the categorical limit.
At $\eta=0.01$, low-NFE quality is preserved while the frontier improves at
higher NFEs, while $\eta=0.001$ remains similar (\cref{tab:eta_ablation}).
Larger concentration helps at intermediate NFE but is less reliable at high
NFE and may fail to reach the target entropy.
We therefore use $\eta=0.01$ across inference budgets.

\noindent\textbf{Training and sampling.}
Simplax preserves the UDLM marginal over $\bz_t$ while augmenting it with
auxiliary $\bw_t$, allowing either sampler to be used with either checkpoint.
Here, UDLM uses the $1\mathrm{M}$-step checkpoint, while Simplax uses
UDLM-initialized variant from \cref{sec:exp_owt}.
We therefore evaluate all four training--sampling combinations
(\cref{tab:training_sampling_decomposition}).
At low NFE, replacing the UDLM sampler with the Simplax sampler accounts for
most of the improvement, even with the UDLM checkpoint.
At high NFE, the Simplax sampler instead benefits strongly from Simplax
training, while the UDLM sampler largely removes this advantage.
Thus, the low-NFE gain is mainly associated with the Simplax reverse process,
whereas the high-NFE gain reflects training--sampling interaction.

%% file: sec/5_related_work.tex
\section{Related Work}
\label{sec:related_work}

\noindent\textbf{Discrete diffusion for categorical data.}
Discrete diffusion for categorical data was developed through early multinomial formulations and later unified and substantially generalized by D3PM, which introduced structured transition kernels such as uniform and absorbing corruptions and established the standard variational training recipe \citep{hoogeboom2021argmax,austin2021structured}. This framework was  extended to continuous-time formulations and alternative reverse objectives \citep{campbell2022a,lou2024discrete,zhang2025target}, and has since supported a broad line of language-modeling work covering masked, absorbing, and uniform-state diffusion \citep{sahoo2024simple,shi2024simplified,ou2025your,zheng2025masked,sahoo2025duality,deschenaux2026duality2,li2022diffusionlm}. Our method stays within this discrete-diffusion lineage: we keep the original categorical forward process and reverse posterior, and do not replace the primary generative state.

\noindent\textbf{Auxiliary-variable and hybrid formulations.}
A parallel line of work enriches discrete diffusion through auxiliary variables or structured reverse distributions. Di4C \citep{di4c} uses mixtures of product models to capture dimensional correlations, VADD \citep{vadd} introduces a Gaussian latent into masked denoising, and CoDD couples factorized outputs with probabilistic circuits \citep{codd}. Continuous and hybrid constructions include Gaussian-relaxed views in Duo and Duo++ \citep{sahoo2025duality,deschenaux2026duality2}, Euclidean denoising over one-hot states in FLM \citep{lee2026flm}, and discrete--continuous diffusion in CADD and CANDI \citep{zheng2025cadd,pynadath2025candi}. Simplax instead preserves categorical diffusion and keeps $\bz_t$ as the
denoiser input, using the simplex variable only to define the reverse bridge.

\noindent\textbf{Diffusion and flow on the simplex.}
Our method is related to a broader class of diffusion and flow
constructions on the simplex for generative modeling, including
softmax-transformed continuous diffusion \citep{floto2023diffusion},
categorical SDEs with Cox--Ingersoll--Ross dynamics
\citep{richemond2022categorical}, Dirichlet-based score models such as DDSM
\citep{avdeyev2023dirichlet}, Dirichlet Flow Matching
\citep{stark2024dirichlet}, and recent unifying views of discrete, Gaussian,
and simplicial diffusion \citep{chandra2026a}.
These works are geometrically close to ours through their use or analysis of
simplex-valued states, but differ in role: Simplax uses the simplex variable
as an exact auxiliary bridge attached to a standard discrete diffusion
process, rather than as the primary generative state.

%% file: sec/6_conclusion.tex
\section{Conclusion and Limitations}
\label{sec}

We introduced Simplax, an exact Dirichlet--categorical augmentation of uniform
discrete diffusion. Simplax preserves the categorical forward process while
introducing an auxiliary simplex state to derive a Rao--Blackwellized
reverse-bridge objective and stochastic ancestral sampler. The denoiser still operates on discrete categorical states, avoiding the dense
vocabulary-embedding product associated with continuous simplex inputs while
retaining standard embedding lookups. It improves the
perplexity--entropy tradeoff on OpenWebText and achieves the highest Sudoku
performance among compared methods across different clues from $40$ to $17$,
as well as in unconditional generation.

\noindent\textbf{Limitation.} The present formulation is specialized to uniform categorical corruption and
introduces an auxiliary simplex-valued state whose computational overhead
relative to standard discrete diffusion has not been fully characterized.
Moreover, the concentration schedule remains a design choice rather than being
determined by theory. Extending the construction to broader categorical
corruption kernels and developing more efficient reverse solvers remain
important future directions.

%% file: sec/999_acknowledgments.tex
\subsection*{Acknowledgments}

We thank Chanhyuk Lee, Jaehoon Yoo, and Jinwoo Kim for insightful discussions.

%% file: sec/999_iclr_statements.tex
\subsection*{AI use statement}

In this work, we used generative AI tools to assist with the development and
refinement of conceptual and mathematical formulations, experimental design,
implementation and debugging, and the interpretation of experimental results.
We also used generative AI tools for brainstorming, literature search and
summarization, and drafting and editing portions of the manuscript.
Generative AI tools were also used to assist with mathematical derivations and
proofs, including suggesting proof strategies and checking intermediate
derivation steps.

All AI-assisted mathematical arguments and claims were independently checked by
the authors. References suggested by AI tools were verified against the original
sources, and AI-assisted code was reviewed and tested by the authors. All
reported numerical results were obtained from actual experimental runs and
checked against the underlying outputs. We have reviewed all AI-assisted work
and take responsibility for the final content of this work, including text,
claims, code, and other artifacts produced with the aid of generative AI.

\subsection*{Ethics statement}

This work does not involve human subjects or the collection of personal data,
and we are not aware of any ethical concerns specific to this work.

\subsection*{Reproducibility statement}

We provide complete derivations and additional theoretical details in the
appendix. The experimental setup, including datasets and preprocessing, model
architectures, training hyperparameters, sampling configurations, evaluation
protocols, and computational resources, is documented in the experimental
sections and appendix. An anonymized implementation is provided as part of the
supplementary material.

%% file: sec_app/a_owt_experiment.tex
\section{OpenWebText Experimental Details}
\label{app:owt_experiment}

This section provides preprocessing, optimization, checkpoint, sampling, and
evaluation details that are omitted from the main text.
The dataset, tokenizer, sequence length, backbone, training budget, and
entropy-matched evaluation protocol are summarized in
\cref{sec:exp_settings,sec:exp_owt}.

\subsection{Training Setup}
\label{sec:owt_training_details}

\noindent\textbf{Data preprocessing.}
We use the \texttt{openwebtext-train} and
\texttt{openwebtext-valid} splits.
Documents are concatenated with an end-of-sequence token inserted between
adjacent documents and packed into fixed-length blocks of $1024$ GPT-2
tokens.

\noindent\textbf{Model architecture.}
All methods use the $179\mathrm{M}$-parameter diffusion transformer of
\citet{sahoo2024simple}, consisting of $12$ transformer blocks with rotary
position embeddings \citep{rope}, AdaLN time conditioning \citep{dit}, and a
softmax output head.

\noindent\textbf{Optimization details.}
Models trained in our common codebase use Adam \citep{adam} with learning rate
$3\times10^{-4}$, batch size $512$, $\beta_1=0.9$, $\beta_2=0.999$,
numerical constant $10^{-8}$, and gradient-norm clipping at $1.0$.
Training uses bfloat16 precision and a linear warmup over the first
$2{,}500$ optimizer steps, followed by a constant learning rate.

\noindent\textbf{Simplax training and checkpoint.}
We evaluate both Simplax trained from scratch for $1\mathrm{M}$ optimizer
steps and a UDLM-initialized variant.
The latter starts from the $800\mathrm{k}$-step UDLM checkpoint and is trained
with the Simplax objective for the remaining $200\mathrm{k}$ steps, keeping
the total optimization budget at $1\mathrm{M}$ steps.
In our measured training configuration, a $1\mathrm{M}$-step run requires
approximately ten days on eight NVIDIA H200 GPUs.

The network predicts the clean-token distribution and receives the categorical
state $\bz_t$ as input, while the relaxed state $\bw_t$ remains in the
training objective.
We use a uniform time schedule, constant Dirichlet concentration
$\eta=0.01$, and constant loss weighting.
We do not use an auxiliary self-conditioning input \citep{chen2023analog}.

Dirichlet sampling and concentration-dependent computations are performed in
float64.
Concentration values are restricted to
$[10^{-10},10^8]$ for numerical stability.
Following the FLM implementation \citep{lee2026flm}, before softmax we
softly bound the output logits as $\ell \leftarrow 30\tanh\left({\ell}/{30}\right)$.

\subsection{Baseline Checkpoints and Implementations}
\label{sec:owt_baseline_checkpoints}

\begin{description}

\item[\textbf{AR.}] We use the autoregressive OpenWebText checkpoint released with S-FLM at \url{https://huggingface.co/jdeschena/s-flm}. The checkpoint is evaluated without additional training.

\item[\textbf{CANDI.}]
We use an OpenWebText checkpoint provided directly by the CANDI authors after
contacting them.
We evaluate the checkpoint using the implementation released with the method,
without additional training.

\item[\textbf{UDLM.}]
We could not identify a publicly available OpenWebText checkpoint for UDLM.
We therefore train UDLM in our common codebase for 1M optimizer steps.
The reported UDLM results use the resulting 1M-step checkpoint.

\item[\textbf{MDLM.}]
We use the official OpenWebText checkpoint released by the authors at
\url{https://huggingface.co/kuleshov-group/mdlm-owt}.
The checkpoint is evaluated without additional training.

\item[\textbf{Duo.}]
We use the official checkpoint released by the authors at
\url{https://huggingface.co/s-sahoo/duo}.
The checkpoint is evaluated without additional training.

\item[\textbf{FLM.}]
We use the official OpenWebText checkpoint released by the authors at
\url{https://huggingface.co/david3684/FLM-B-OWT}.
The checkpoint is evaluated without additional training.

\item[\textbf{LangFlow.}]
We use the official OpenWebText checkpoint released by the authors at
\url{https://huggingface.co/Continuous-Rivals-Discrete/langflow-owt}.
The checkpoint is evaluated without additional training.

\item[\textbf{S-FLM.}]
We use the official OpenWebText checkpoint released by the authors at
\url{https://huggingface.co/jdeschena/s-flm}.
The checkpoint is evaluated without additional training.

\end{description}

\subsection{Sampling and Evaluation}
\label{sec:owt_sampling_details}

\cref{tab:sampling_details} summarizes the iterative sampler and terminal operation
used in our evaluation.

\begin{table}[t]
\centering
\caption{Sampling and terminal decoding used for OpenWebText and Sudoku experiments.}
\label{tab:sampling_details}
\small
\begin{tabularx}{\columnwidth}{l X l}
\toprule
Method & Sampler & Terminal decoding \\
\midrule
AR
& Left-to-right autoregressive sampling
& Token sampling \\

MDLM
& Absorbing diffusion ancestral sampler
& Ancestral sample \\

DUO / UDLM
& Uniform diffusion ancestral sampler
& Ancestral sample \\

FLM
& Continuous-state Euler ODE sampler
& Categorical Sample \\

Langflow
& Euler--EDM
& Temperature-scaled categorical sampling \\

SFM
& Spherical flow sampler
& Greedy argmax \\

\midrule

Simplax
& $\bw$-deterministic sampler
& Ancestal sample \\
\bottomrule
\end{tabularx}
\end{table}

\subsection{Runtime and Memory Usage}
\label{sec:owt_runtime_memory}

\noindent\textbf{Benchmark setup.}
All methods are benchmarked on the same GPU node equipped with one NVIDIA
H200 GPU with 141\,GB of HBM, 8 CPU cores, and 96\,GB of CPU memory. For all methods, we compile the backbone and the major RNG-independent
sampler components with \texttt{torch.compile}, using
\texttt{dynamic=False} and \texttt{max\_autotune=True}, with Triton CUDA
graphs disabled. Random-number generation, including categorical sampling,
remains eager. For AR, we enable KV caching and report eager execution, as compiling the
backbone and sampling loop was slower in our benchmark. Each method--NFE configuration is evaluated in a fresh Python
process with a fresh Inductor cache. We use the cached-integral implementation
for CANDI and the $\bw$-deterministic transition for Simplax.
We preserve the numerical precision of each reference implementation:
the backbone is evaluated in bfloat16 by default, while LangFlow uses its
FP32 sampler, S-FLM its exact-velocity FP64 sampler.

\noindent\textbf{Runtime.}
After loading the model and checkpoint, we perform two complete generation
runs for warm-up and measure the following five runs.
Each measurement is bracketed by CUDA synchronization and includes generation
only, excluding model loading, EMA copying, compilation and warm-up, token
decoding, perplexity evaluation, and I/O.
We report the median wall-clock time $\pm$ sample standard deviation over the
five runs, in seconds per batch; each batch contains 16 sequences of length
1024.

\noindent\textbf{Memory.}
We run full garbage collection before each measured repeat, outside the timed
region, and use the same generation-0 garbage-collection schedule within the
sampler for all methods. We retain the CUDA allocator cache after warm-up.
Immediately before each repeat, we reset the peak-memory statistics and report
the maximum \texttt{torch.cuda.max\_memory\_allocated()} value.
Thus, the reported memory is peak allocated CUDA tensor memory and includes the resident model and
checkpoint. Thus, the reported memory is peak allocated CUDA tensor memory and includes
the resident model and checkpoint.
\Cref{tab:owt_runtime_memory} summarizes the runtime and peak memory
measurements.

\begin{table*}[t]
\centering
\caption{
Sampling runtime and peak allocated GPU memory on OpenWebText.
Runtime is median seconds per batch of 16 sequences of length 1024
$\pm$ sample standard deviation over five runs.
Peak memory is the maximum observed across the three NFE settings.
The best and second-best values among discrete diffussion methods are shown in bold and underlined, respectively.
}
\label{tab:owt_runtime_memory}
\footnotesize
\setlength{\tabcolsep}{5pt}
\renewcommand{\arraystretch}{1.08}
\begin{tabular}{@{}lcccc@{}}
\toprule
Method
& NFE $=16$
& NFE $=128$
& NFE $=1024$
& Peak mem.\ (GiB) \\
\midrule
AR
& -
& -
& 8.463$\pm$ 0.027
& $2.659$ \\
\midrule
CANDI
& $\underline{0.6483{\pm0.0002}}$
& $\underline{4.968{\pm0.005}}$
& $\underline{39.692{\pm0.011}}$
& $26.52$ \\

UDLM
& $0.6765{\pm0.0011}$
& $5.529{\pm0.012}$
& $44.475{\pm0.015}$
& $26.47$ \\

MDLM
& $\mathbf{0.5477{\pm0.0009}}$
& $\mathbf{4.382{\pm0.013}}$
& $\mathbf{35.375{\pm0.030}}$
& $\underline{15.74}$ \\

Duo
& $0.6767{\pm0.0008}$
& $5.526{\pm0.009}$
& $44.490{\pm0.011}$
& $26.48$ \\

FLM
& $0.8957{\pm0.0016}$
& $7.122{\pm0.006}$
& $57.027{\pm0.003}$
& $18.81$ \\

LangFlow
& $1.1526{\pm0.0019}$
& $9.377{\pm0.016}$
& $75.259{\pm0.030}$
& $\mathbf{13.04}$ \\

S-FLM
& $1.1834{\pm0.0036}$
& $9.977{\pm0.015}$
& $80.489{\pm0.009}$
& $20.77$ \\

\midrule
Simplax
& $0.8141{\pm0.0003}$
& $6.093{\pm0.002}$
& $48.362{\pm0.002}$
& $20.34$ \\
\bottomrule
\end{tabular}
\end{table*}

\subsection{Comparison with ELF}
\label{sec:owt_elf}

ELF \citep{elf} is a continuous-state language model based on flow matching
in embedding space.
Its generation trajectory remains continuous until the final step, where the
predicted embedding is mapped to discrete tokens through a shared-weight
decoder.

Because the intermediate trajectory does not involve categorical sampling,
temperature can naturally be applied only to terminal token decoding and does
not control the preceding continuous dynamics.
Following the ELF evaluation protocol, we instead sweep
self-conditioning classifier-free guidance (SC-CFG) to vary the
quality--diversity tradeoff.
We use the SDE sampler with a logit-normal time schedule,
\texttt{sde\_gamma}=1.5, and CFG scale $1.0$, and sweep 15 SC-CFG scales:
$\{0,0.25,0.5,0.75,1,1.5,2,2.5,3,3.5,4,5,6,8,10\}$.

As shown in \cref{fig:elf_frontier_llama2_grid}, under Llama-2 7B evaluation,
the ELF curves do not reach the OpenWebText data entropy of $5.44$ at any
evaluated NFE.
ELF attains lower Gen.\ PPL than Simplax in the lower-entropy regime, but a
direct matched-entropy comparison is therefore not available.

\begin{figure*}[t]
  \centering
  \includegraphics[width=\textwidth]{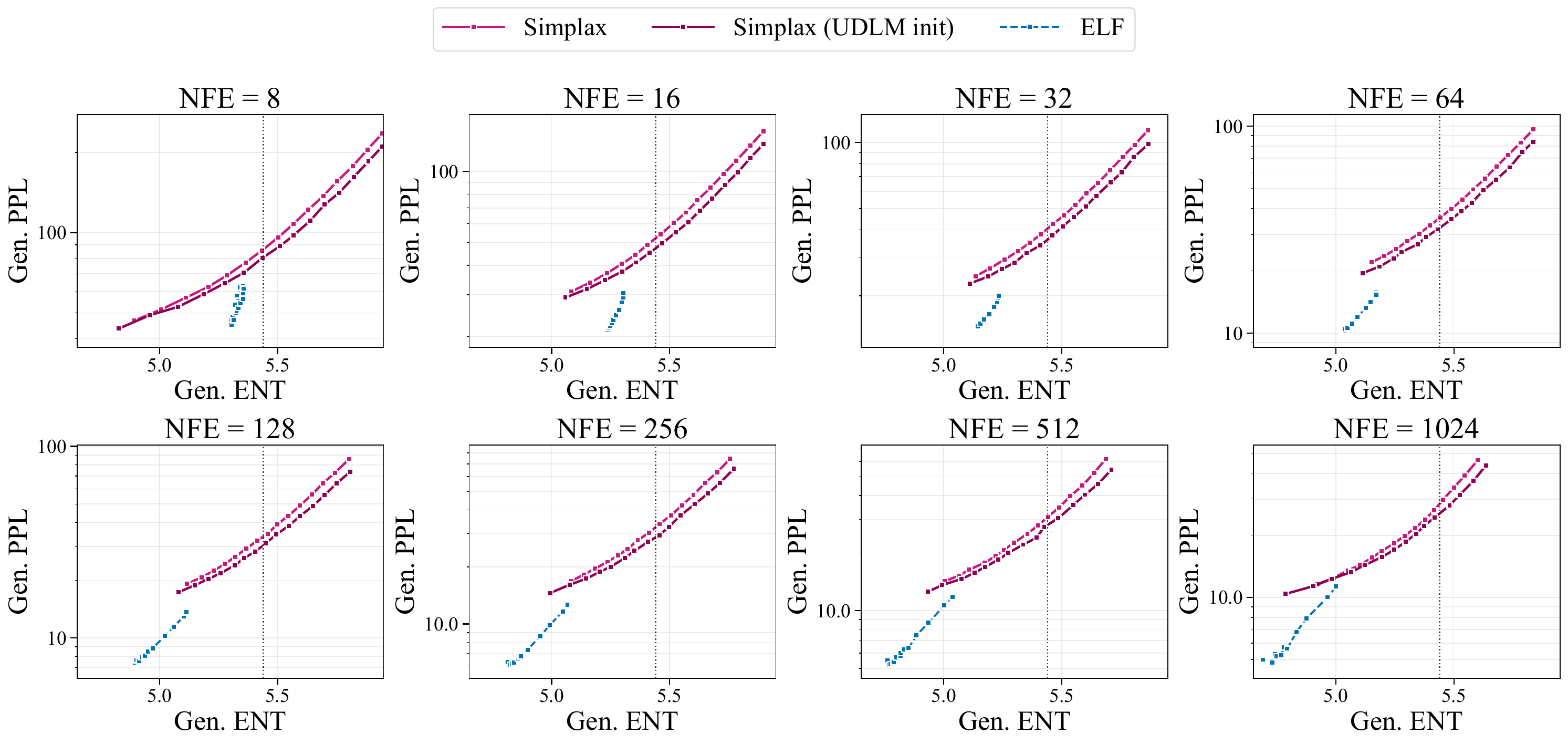}
  \caption{
    ELF and Simplax generative frontiers on OpenWebText under Llama-2 7B
    evaluation for
    $\mathrm{NFE}\in\{8,16,32,64,128,256,512,1024\}$.
    Each panel shows the Gen.\ PPL--Gen.\ ENT tradeoff obtained by sweeping
    temperature for Simplax and SC-CFG scale for ELF.
    Lower Gen.\ PPL is better, and the reference data entropy is $5.44$.
  }
  \label{fig:elf_frontier_llama2_grid}
\end{figure*}

\subsection{Additional OpenWebText Evaluation Results}
\label{sec:additional_owt_results}

The main text reports the temperature-swept generative frontier under
Llama-2 7B evaluation and generative perplexities at NFE $=16$, $128$,
and $1024$ using GPT-2 Large and Llama-2 7B.
Here, we additionally report the corresponding GPT-2 XL perplexities and
provide evaluator-specific results with bootstrap standard errors for
GPT-2 Large, GPT-2 XL, and Llama-2 7B
(\cref{tab:owt_ppl_gpt2_large,tab:owt_ppl_gpt2_xl,tab:owt_ppl_llama2}).
For each method and NFE, perplexity is evaluated at the temperature whose
generated unigram entropy is closest to the OpenWebText entropy of $5.44$.

\noindent\textbf{Bootstrap uncertainty.}
Evaluating multiple independently generated sets of $1{,}024$ sequences is
computationally expensive, particularly at NFE $=1024$.
We therefore estimate sampling uncertainty using a sequence-level bootstrap.
For each method and NFE, we fix the selected temperature before bootstrapping
and resample the $1{,}024$ generated sequences with replacement for
$10{,}000$ bootstrap replicates.
For each replicate, generative perplexity is recomputed from the pooled token
negative log-likelihood and token count of the resampled sequences.
We report the standard deviation across bootstrap replicates as the bootstrap
standard error.
This uncertainty reflects sampling variability of generative perplexity for
the fixed checkpoint and temperature, and does not include uncertainty from
temperature selection or independently repeated generation runs.
\Cref{tab:owt_ppl_gpt2_large,tab:owt_ppl_gpt2_xl,tab:owt_ppl_llama2}
report PPL $\pm$ bootstrap standard error at the temperature whose generated
unigram entropy is closest to $5.44$.

\begin{table*}[t]
\centering
\caption{OpenWebText generative perplexity evaluated by GPT-2 Large. The
  best and second-best values among discrete diffusion methods are shown in
  bold and underlined, respectively.}
\label{tab:owt_ppl_gpt2_large}
\scriptsize
\setlength{\tabcolsep}{3pt}
\renewcommand{\arraystretch}{1.08}
\resizebox{\textwidth}{!}{%
\begin{tabular}{lcccccccc}
\toprule
Method
& NFE $=8$
& NFE $=16$
& NFE $=32$
& NFE $=64$
& NFE $=128$
& NFE $=256$
& NFE $=512$
& NFE $=1024$ \\
\midrule
AR
& -
& -
& -
& -
& -
& -
& -
& $22.82\pm0.26$ \\
\midrule
CANDI
& $170.78\pm1.32$
& $\underline{97.23\pm0.85}$
& $80.05\pm0.69$
& $64.47\pm0.61$
& $67.17\pm0.64$
& $63.70\pm0.68$
& $75.83\pm1.01$
& $74.26\pm0.98$ \\

UDLM
& $160.31\pm1.13$
& $186.03\pm1.91$
& $84.28\pm0.66$
& $\underline{62.23\pm0.50}$
& $\mathbf{54.97\pm0.49}$
& $\underline{57.16\pm0.55}$
& $\underline{54.00\pm0.51}$
& $59.24\pm0.57$ \\

MDLM
& $195.83\pm1.37$
& $117.86\pm0.86$
& $\underline{72.73\pm0.55}$
& $66.22\pm0.56$
& $65.37\pm0.56$
& $66.90\pm0.63$
& $68.65\pm0.66$
& $55.13\pm0.54$ \\

Duo
& $221.00\pm1.67$
& $166.12\pm1.81$
& $96.99\pm0.90$
& $92.48\pm1.00$
& $93.91\pm1.12$
& $90.51\pm1.16$
& $88.63\pm1.14$
& $88.94\pm1.10$ \\

FLM
& $569.26\pm4.09$
& $259.45\pm1.78$
& $146.13\pm0.80$
& $122.33\pm0.69$
& $112.02\pm0.70$
& $111.57\pm0.84$
& $122.76\pm1.10$
& $125.60\pm1.29$ \\

LangFlow
& $211.22\pm1.48$
& $115.06\pm0.91$
& $80.70\pm0.70$
& $70.19\pm0.65$
& $60.20\pm0.62$
& $62.82\pm0.83$
& $67.35\pm1.25$
& $68.27\pm1.75$ \\

S-FLM
& $\underline{150.61\pm1.14}$
& $124.63\pm1.17$
& $116.29\pm1.26$
& $111.65\pm1.30$
& $103.39\pm1.22$
& $111.48\pm1.37$
& $109.53\pm1.38$
& $108.53\pm1.36$ \\

\midrule
Simplax
& $155.86\pm1.26$
& $98.72\pm0.79$
& $77.37\pm0.61$
& $64.74\pm0.51$
& $63.37\pm0.53$
& $60.43\pm0.49$
& $55.22\pm0.47$
& $\underline{52.55\pm0.44}$ \\

Simplax (UDLM init)
& $\mathbf{149.13\pm1.26}$
& $\mathbf{81.53\pm0.67}$
& $\mathbf{70.00\pm0.58}$
& $\mathbf{57.89\pm0.48}$
& $\underline{56.48\pm0.51}$
& $\mathbf{54.10\pm0.49}$
& $\mathbf{48.60\pm0.51}$
& $\mathbf{42.77\pm0.43}$ \\
\bottomrule
\end{tabular}%
}
\end{table*}

\begin{table*}[t]
\centering
\caption{OpenWebText generative perplexity evaluated by GPT-2 XL. The
  best and second-best values among discrete diffusion methods are shown in
  bold and underlined, respectively.}
\label{tab:owt_ppl_gpt2_xl}
\scriptsize
\setlength{\tabcolsep}{3pt}
\renewcommand{\arraystretch}{1.08}
\resizebox{\textwidth}{!}{%
\begin{tabular}{lcccccccc}
\toprule
Method
& NFE $=8$
& NFE $=16$
& NFE $=32$
& NFE $=64$
& NFE $=128$
& NFE $=256$
& NFE $=512$
& NFE $=1024$ \\
\midrule
AR
& -
& -
& -
& -
& -
& -
& -
& $24.12\pm0.28$ \\
\midrule
CANDI
& $174.46\pm1.34$
& $\underline{99.60\pm0.88}$
& $82.43\pm0.72$
& $66.51\pm0.63$
& $69.15\pm0.67$
& $65.76\pm0.71$
& $78.12\pm1.06$
& $76.60\pm1.02$ \\

UDLM
& $163.75\pm1.16$
& $190.07\pm1.95$
& $86.85\pm0.69$
& $\underline{64.14\pm0.52}$
& $\mathbf{56.71\pm0.51}$
& $\underline{58.92\pm0.58}$
& $\underline{55.79\pm0.53}$
& $61.15\pm0.60$ \\

MDLM
& $199.60\pm1.41$
& $120.61\pm0.88$
& $\underline{74.62\pm0.57}$
& $68.11\pm0.58$
& $67.18\pm0.59$
& $68.83\pm0.65$
& $70.52\pm0.68$
& $56.63\pm0.56$ \\

Duo
& $223.93\pm1.69$
& $168.33\pm1.84$
& $99.03\pm0.92$
& $94.30\pm1.02$
& $95.85\pm1.15$
& $92.40\pm1.19$
& $90.44\pm1.17$
& $90.64\pm1.13$ \\

FLM
& $572.55\pm4.09$
& $262.56\pm1.81$
& $148.21\pm0.82$
& $124.11\pm0.70$
& $113.72\pm0.71$
& $113.35\pm0.86$
& $124.47\pm1.12$
& $127.16\pm1.31$ \\

LangFlow
& $212.92\pm1.49$
& $116.64\pm0.92$
& $82.33\pm0.72$
& $71.81\pm0.67$
& $61.69\pm0.64$
& $64.43\pm0.86$
& $69.12\pm1.30$
& $70.01\pm1.83$ \\

S-FLM
& $\underline{152.51\pm1.16}$
& $126.37\pm1.20$
& $118.05\pm1.29$
& $113.55\pm1.34$
& $105.08\pm1.25$
& $113.20\pm1.39$
& $111.19\pm1.40$
& $110.21\pm1.39$ \\

\midrule
Simplax
& $159.23\pm1.29$
& $101.32\pm0.81$
& $79.79\pm0.64$
& $66.75\pm0.54$
& $65.44\pm0.55$
& $62.43\pm0.52$
& $57.09\pm0.48$
& $\underline{54.50\pm0.46}$ \\

Simplax (UDLM init)
& $\mathbf{152.50\pm1.30}$
& $\mathbf{83.92\pm0.69}$
& $\mathbf{72.11\pm0.60}$
& $\mathbf{59.77\pm0.50}$
& $\underline{58.40\pm0.53}$
& $\mathbf{55.98\pm0.52}$
& $\mathbf{50.32\pm0.53}$
& $\mathbf{44.47\pm0.44}$ \\
\bottomrule
\end{tabular}%
}
\end{table*}

\begin{table*}[t]
\centering
\caption{OpenWebText generative perplexity evaluated by Llama-2 7B. The
  best and second-best values among discrete diffusion methods are shown in
  bold and underlined, respectively.}
\label{tab:owt_ppl_llama2}
\scriptsize
\setlength{\tabcolsep}{3pt}
\renewcommand{\arraystretch}{1.08}
\resizebox{\textwidth}{!}{%
\begin{tabular}{lcccccccc}
\toprule
Method
& NFE $=8$
& NFE $=16$
& NFE $=32$
& NFE $=64$
& NFE $=128$
& NFE $=256$
& NFE $=512$
& NFE $=1024$ \\
\midrule
AR
& -
& -
& -
& -
& -
& -
& -
& $15.95\pm0.18$ \\
\midrule
CANDI
& $100.80\pm0.66$
& $56.03\pm0.43$
& $44.77\pm0.33$
& $36.71\pm0.31$
& $36.79\pm0.31$
& $35.22\pm0.34$
& $39.86\pm0.53$
& $39.17\pm0.50$ \\

UDLM
& $84.47\pm0.57$
& $78.98\pm0.67$
& $44.81\pm0.33$
& $\underline{34.86\pm0.26}$
& $31.51\pm0.24$
& $31.96\pm0.27$
& $30.62\pm0.28$
& $32.18\pm0.29$ \\

MDLM
& $105.75\pm0.74$
& $63.84\pm0.49$
& $\underline{41.64\pm0.29}$
& $38.02\pm0.29$
& $37.27\pm0.29$
& $37.64\pm0.31$
& $39.15\pm0.35$
& $33.88\pm0.31$ \\

Duo
& $123.74\pm0.85$
& $87.11\pm0.81$
& $58.68\pm0.52$
& $53.98\pm0.53$
& $53.27\pm0.60$
& $50.22\pm0.61$
& $49.25\pm0.62$
& $50.39\pm0.60$ \\

FLM
& $277.41\pm2.55$
& $139.27\pm1.18$
& $83.61\pm0.48$
& $71.79\pm0.39$
& $66.54\pm0.34$
& $67.62\pm0.37$
& $72.28\pm0.43$
& $74.35\pm0.50$ \\

LangFlow
& $106.23\pm0.74$
& $59.89\pm0.44$
& $41.66\pm0.32$
& $35.17\pm0.29$
& $\mathbf{30.03\pm0.28}$
& $\underline{29.68\pm0.35}$
& $\underline{29.45\pm0.47}$
& $\underline{28.16\pm0.59}$ \\

S-FLM
& $\mathbf{76.46\pm0.48}$
& $62.90\pm0.46$
& $58.24\pm0.49$
& $55.88\pm0.49$
& $52.41\pm0.47$
& $54.99\pm0.51$
& $54.10\pm0.51$
& $53.75\pm0.51$ \\

\midrule
Simplax
& $85.66\pm0.68$
& $\underline{54.15\pm0.44}$
& $42.50\pm0.32$
& $36.27\pm0.27$
& $34.98\pm0.28$
& $33.65\pm0.28$
& $30.93\pm0.26$
& $29.79\pm0.27$ \\

Simplax (UDLM init)
& $\underline{80.45\pm0.65}$
& $\mathbf{45.21\pm0.37}$
& $\mathbf{37.64\pm0.31}$
& $\mathbf{31.74\pm0.24}$
& $\underline{31.12\pm0.28}$
& $\mathbf{29.40\pm0.26}$
& $\mathbf{27.40\pm0.30}$
& $\mathbf{24.43\pm0.26}$ \\
\bottomrule
\end{tabular}%
}
\end{table*}

We additionally report the full temperature-swept generative frontiers under
GPT-2 Large and GPT-2 XL evaluation in
\cref{fig:frontier_gpt2_large_grid,fig:frontier_gpt2_xl_grid}.
These evaluations use the same generated samples, temperatures, and
generative entropy values as the main OpenWebText experiment; only the
external language model used to compute generative perplexity is changed.

\begin{figure*}[t]
  \centering
  \includegraphics[width=\textwidth]{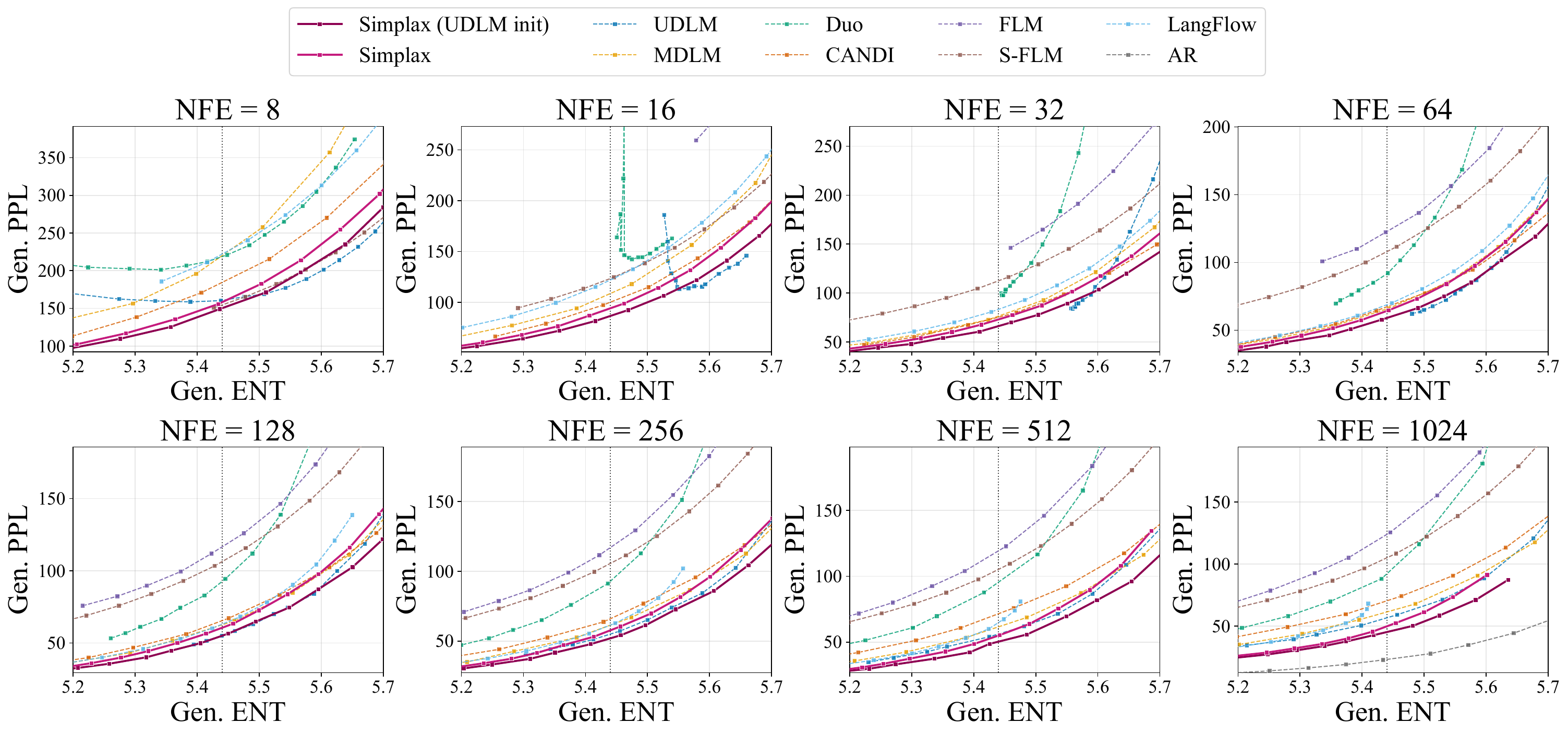}
  \caption{
    GPT-2 Large generative frontiers on OpenWebText for
    $\mathrm{NFE}\in\{8,16,32,64,128,256,512,1024\}$.
    Each panel shows the temperature-swept Gen.\ PPL--Gen.\ ENT tradeoff.
    Lower Gen.\ PPL is better, and the reference data entropy is $5.44$.
  }
  \label{fig:frontier_gpt2_large_grid}
\end{figure*}

\begin{figure*}[t]
  \centering
  \includegraphics[width=\textwidth]{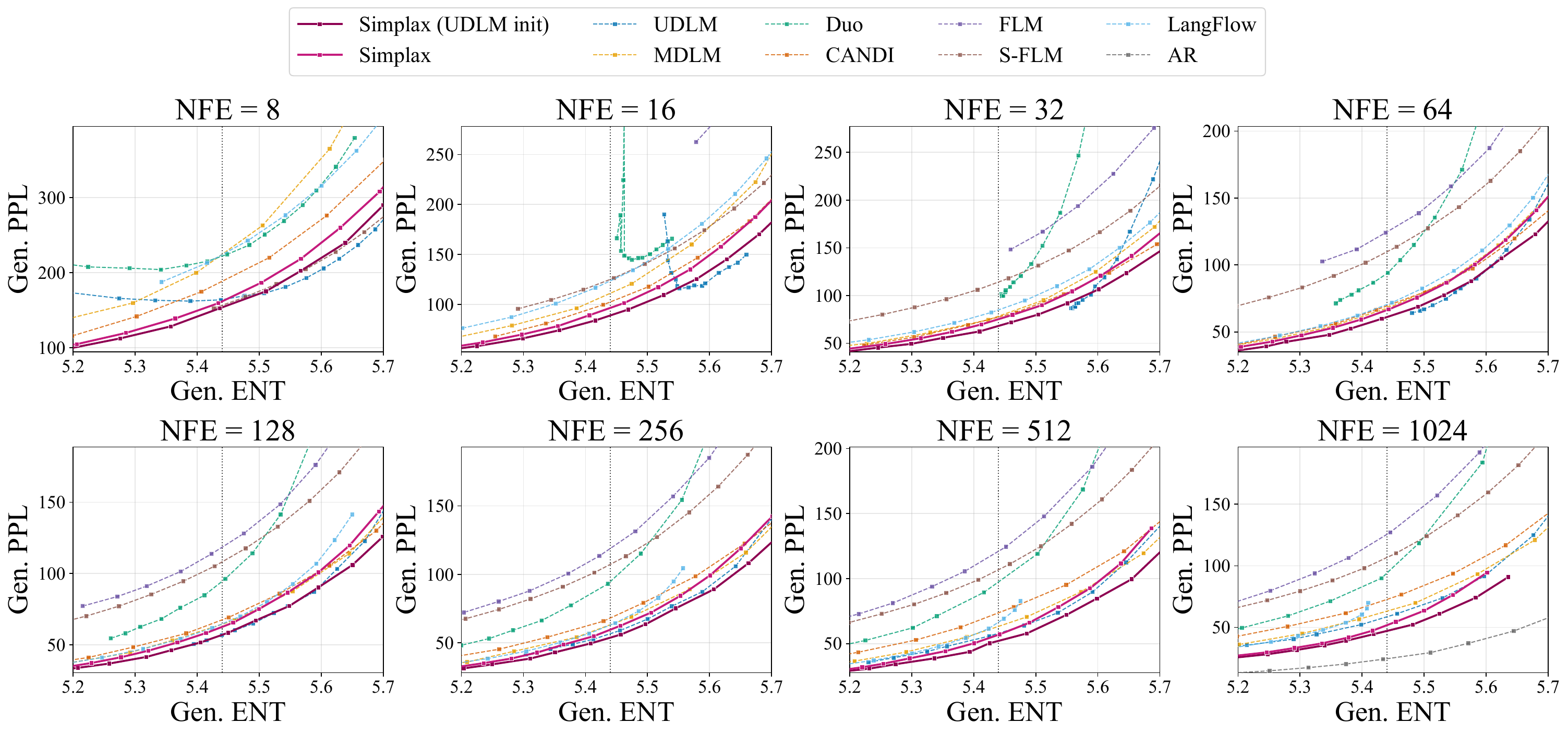}
  \caption{
    GPT-2 XL generative frontiers on OpenWebText for
    $\mathrm{NFE}\in\{8,16,32,64,128,256,512,1024\}$.
    Each panel shows the temperature-swept Gen.\ PPL--Gen.\ ENT tradeoff.
    Lower Gen.\ PPL is better, and the reference data entropy is $5.44$.
  }
  \label{fig:frontier_gpt2_xl_grid}
\end{figure*}

\subsection{Qualitative Generations}
\label{sec:owt_qualitative}

\cref{tab:owt-generations-nfe-16,tab:owt-generations-nfe-128,tab:owt-generations-nfe-1024}
show representative generations at NFE $=16$, $128$, and $1024$.
For each method and NFE, we select examples from the temperature whose
generated unigram entropy is closest to the OpenWebText entropy of $5.44$.
The generated text is not manually rewritten.
The excerpts are truncated at the positions marked by
\textcolor{owtgray}{[\ldots]}, and line wrapping is applied only for
presentation.

\begin{table*}[tp]
  \centering
  \caption{
    Representative unconditional OpenWebText generations at NFE $=16$.
    Entropy is the generative unigram entropy in nats per token.
  }
  \label{tab:owt-generations-nfe-16}
  \scriptsize
  \renewcommand{\arraystretch}{1.18}
  \setlength{\tabcolsep}{4pt}
  \begin{tabularx}{
      \textwidth
    }{
      @{}
      p{0.78in}
      >{\centering\arraybackslash}p{0.38in}
      X
      @{}
    }
    \toprule
    \textbf{Method}
      & \textbf{Ent.}
      & \textbf{Generated text} \\
    \midrule

    \rowcolor{owtpanel}
    \textbf{CANDI}
      & 5.44
      & July on our records and the song in July, and we're signing people
        just for the next album, but it didn't couple with laser or anything.
        I had just a video. I thought they were crazy, because there were a
        lot of people out there who did it right. They felt like it was
        completely under my radar.
        \enspace\textcolor{owtgray}{[\ldots]} \\[4pt]

    \textbf{UDLM}
      & 5.53
      & What with talk of a new bus un the altitude in the first place.
        In South Bend in November came the focus upon new cyclists -- with an
        being builder. The other duck, Alabama. In total at least 82 bikes
        were donated to the department.
        \enspace\textcolor{owtgray}{[\ldots]} \\[4pt]

    \rowcolor{owtpanel}
    \textbf{MDLM}
      & 5.43
      & Select Rally and Comm Rally was a great game with strong designed
        scenarios, but I did think about no-one should play it as middle of
        the road. You've already started to love my opinion on Two vs Two.
        You should check it, and stick and play this once again! We again
        stack you. This is the final year for the year.
        \enspace\textcolor{owtgray}{[\ldots]} \\[4pt]

    \textbf{Duo}
      & 5.45
      & I ever thought that it would be torrent be progression to try
        something new like my team base's football department. While I would
        have been optimistic about both the br of young players that hockey
        at the University at which went down in international football over
        the last of years and the way paths kept me focused so well I was
        pretty.
        \enspace\textcolor{owtgray}{[\ldots]} \\[4pt]

    \rowcolor{owtpanel}
    \textbf{FLM}
      & 5.58
      & Dec 16, 2015 Edit: I cant let my reader now clear that this painting
        looks like a collegecommunication.'' On Z. Thanks the box. that by the
        way I am out over to the river to buy your best chance to date for
        such a horrible year. But at which point after year they release the
        windows of 3,100 by 11 inches. St.
        \enspace\textcolor{owtgray}{[\ldots]} \\[4pt]

    \textbf{LangFlow}
      & 5.42
      & God, will thrive. No others, and not at all, will determine our faith.
        We live in transient despair. But we don't know how to do it tomorrow.
        We want to live. We are sinners, we do not have God. God. we have
        identified the freedom of the one Being, is made to the only kind
        that governs our nation.
        \enspace\textcolor{owtgray}{[\ldots]} \\[4pt]

    \rowcolor{owtpanel}
    \textbf{S-FLM}
      & 5.45
      & I feel like, how much does that have to work out there? Are you going
        into a different development company? I think speaking about the
        timing and the reputation of what I mean as something that have
        really enjoyed my career doing. I kind of think the company can change
        things.
        \enspace\textcolor{owtgray}{[\ldots]} \\[4pt]

    \textbf{Simplax}
      & 5.46
      & It was basically a percentage league, is there going to have to be one
        a year? They always were because they had them. If they knew that
        rule, they got to play one for their bodies. Then, when they voted for
        a ``bolt 1'' rule and what they got was, they want the fifth best
        defender to attack you, so they added
        \enspace\textcolor{owtgray}{[\ldots]} \\[4pt]

    \bottomrule
  \end{tabularx}
\end{table*}

\begin{table*}[tp]
  \centering
  \caption{
    Representative unconditional OpenWebText generations at NFE $=128$.
    Entropy is the generative unigram entropy in nats per token.
  }
  \label{tab:owt-generations-nfe-128}
  \scriptsize
  \renewcommand{\arraystretch}{1.18}
  \setlength{\tabcolsep}{4pt}
  \begin{tabularx}{
      \textwidth
    }{
      @{}
      p{0.78in}
      >{\centering\arraybackslash}p{0.38in}
      X
      @{}
    }
    \toprule
    \textbf{Method}
      & \textbf{Ent.}
      & \textbf{Generated text} \\
    \midrule

    \rowcolor{owtpanel}
    \textbf{CANDI}
      & 5.43
      & I think they should hear about their jobs more than they think. They
        know that they to hear about a huge percentage of young adult
        Americans that are helping create jobs. I realize that what our
        advisors and architects believe is true. They're in fact engaging in
        those American jobs, whether or not they've been doing it for several
        years.
        \enspace\textcolor{owtgray}{[\ldots]} \\[4pt]

    \textbf{UDLM}
      & 5.50
      & The show has been going on for years, and people who tell the
        community about the end, and what way it works are already very
        exciting. But it is difficult to know whether you will disagree, even
        whether you feel the show expects a lot. So tell other friends and
        friends who are just around you, the producers and the media what
        \enspace\textcolor{owtgray}{[\ldots]} \\[4pt]

    \rowcolor{owtpanel}
    \textbf{MDLM}
      & 5.45
      & It's been difficult times with, for me and Brett. It's important to
        tackle this issue on the long haul. I plan on putting it on hold again
        once I have my business, but committed to showing up again by the
        deadline. There's a lot of effort involved in an issue about a loan
        for \$20k, and it was in the middle, I
        \enspace\textcolor{owtgray}{[\ldots]} \\[4pt]

    \textbf{Duo}
      & 5.44
      & America's social justice system. The ability of America's people to
        deal with those that are alike and those different is unclear. But
        when tough and bad, the outcomes are not very different.'' Congressman
        Kathy Lee interviewed for this article told me what she called her
        priority business when serving as lieutenant senator of New York.
        \enspace\textcolor{owtgray}{[\ldots]} \\[4pt]

    \rowcolor{owtpanel}
    \textbf{FLM}
      & 5.43
      & I think that we would have to bring back aggression by the usual
        course, with war with the past, on a situation. If it was done, the
        fascists were inside them or there was nobody stirred up. People were
        part of the invasion by the Nazis, then they shot, they went on.
        \enspace\textcolor{owtgray}{[\ldots]} \\[4pt]

    \textbf{LangFlow}
      & 5.42
      & The body bloates with a layer of wood, black air that brings in your
        skin and exites aging. When reading about your life in the late 90s,
        what was the puzzle you were plotting in your life? I wanted to tell
        the secret of being an ideal person, and that everything that relies
        on that secret still exists.
        \enspace\textcolor{owtgray}{[\ldots]} \\[4pt]

    \rowcolor{owtpanel}
    \textbf{S-FLM}
      & 5.43
      & ``You,'' asked Reeves, who had a light conversation with her.
        ``You shouldn't sell a media conference. You have been around for
        weeks and you haven't come out to 83 game,'' Khanco said, ``WHAT?''
        ``Foreby, you have been making a fortune with free agency!'' she
        replied. ``I have to wonder what it was your secret. I plan to write
        everything with this team.
        \enspace\textcolor{owtgray}{[\ldots]} \\[4pt]

    \textbf{Simplax}
      & 5.45
      & ``I thought it was odd,'' Anna told me. ``I wanted to know,'' Elsa
        said nervously. It's a lovely thing when you use words the way they
        mean something. You seem to like it. A lot.'' That was nice.
        ``It needed to be answered,'' Anna said. ``I just didn't say it since
        I grew up.'' I said.
        \enspace\textcolor{owtgray}{[\ldots]} \\[4pt]

    \bottomrule
  \end{tabularx}
\end{table*}

\begin{table*}[tp]
  \centering
  \caption{
    Representative unconditional OpenWebText generations at NFE $=1024$.
    Entropy is the generative unigram entropy in nats per token.
  }
  \label{tab:owt-generations-nfe-1024}
  \scriptsize
  \renewcommand{\arraystretch}{1.18}
  \setlength{\tabcolsep}{4pt}
  \begin{tabularx}{
      \textwidth
    }{
      @{}
      p{0.78in}
      >{\centering\arraybackslash}p{0.38in}
      X
      @{}
    }
    \toprule
    \textbf{Method}
      & \textbf{Ent.}
      & \textbf{Generated text} \\
    \midrule

    \rowcolor{owtpanel}
    \textbf{CANDI}
      & 5.46
      & Tip: Every single mistake you face, you don't pay for it all the time.
        This is lost confidence. Of course, in this article, you want to be
        booking your own grooming option, just to pay at your local time
        schedule. You have to have a monthly with the brand Indigo Tattoo Shop
        for the monthly fee.
        \enspace\textcolor{owtgray}{[\ldots]} \\[4pt]

    \textbf{UDLM}
      & 5.45
      & Yeah: That's what I understand... and I'd presumably have come on
        board. I think it would be foolish to give something to him, but I
        guess he wouldn't want to write them, because he'd actually want it to
        remain in front of him for a few years.
        \enspace\textcolor{owtgray}{[\ldots]} \\[4pt]

    \rowcolor{owtpanel}
    \textbf{MDLM}
      & 5.44
      & No last of the songs have been put together -- we really want the
        whole development going. I really hope it turns out, or the clock
        falls on its own itself, and if we start putting them together, things
        will just blow up almost completely. That's pretty scary. This will be
        our first time back in the studio.
        \enspace\textcolor{owtgray}{[\ldots]} \\[4pt]

    \textbf{Duo}
      & 5.43
      & If for rats and pies. I've been digging for a hard break. With Ann --
        she'd have a long, bitter it by trying on. It never got there, I-
        She was just shy of puberty or a man. I wouldn't even want to tell my
        grandmother what was going on in her life. Those things went
        somewhere, baby.
        \enspace\textcolor{owtgray}{[\ldots]} \\[4pt]

    \rowcolor{owtpanel}
    \textbf{FLM}
      & 5.45
      & There can be something if you don't know you're someones not squad,
        that's got an entirely different breed but neither do instinctively
        believe what or everyone born with will do the thing. There are some
        of that kind of good will. You only do that stuff for nothing but fun.
        \enspace\textcolor{owtgray}{[\ldots]} \\[4pt]

    \textbf{LangFlow}
      & 5.41
      & 50. ``My messy experience is everything'' he said.'' Something that
        has been on for all my life, and I think I'm living with it, but that's
        tough when I start losing people's attention.'' He said he had
        accepted divorce for most of his family's care because he now relies
        on current degree's rent, healthcare, and social assistance children.
        \enspace\textcolor{owtgray}{[\ldots]} \\[4pt]

    \rowcolor{owtpanel}
    \textbf{S-FLM}
      & 5.45
      & I can't say no they. She's definitely a stockwoman. She's got a huge
        attitude. We don't make her look her the way she's someone that's
        popular. There's just a hint of that.'' If you would give us more time
        to the public on our \#1 gender department, you can tell us your
        thoughts in the comments below.
        \enspace\textcolor{owtgray}{[\ldots]} \\[4pt]

    \textbf{Simplax}
      & 5.44
      & How could I take it to him? Would I have a chat in my life? I didn't
        freak out. Even still, he did not have enough strength, and I was only
        millimeters away, and not even touching his chador. Only a rag with
        lain liquid it around his hair, and he could feel the innocence of a
        Playboy.
        \enspace\textcolor{owtgray}{[\ldots]} \\[4pt]

    \bottomrule
  \end{tabularx}
\end{table*}

%% file: sec_app/b_sudoku_experiment.tex
\section{Sudoku Experimental Details}
\label{app:sudoku_experiments}

This section provides additional details on the Sudoku experiments, including
dataset construction, sequence representation, model architectures, training,
sampling, and qualitative generations.

\subsection{Dataset and Sequence Representation}
\label{sec:sudoku_data_representation}

\noindent\textbf{Dataset.}
All Sudoku puzzles used in our experiments have unique solutions.
Following \citet{sflm}, we use the greedy Sudoku puzzle generator of
\citet{alp2024sudoku} to construct the training data and the evaluation
sets with $25$ or more clues.
The training set contains $48{,}000$ puzzles with $30$ observed cells, and
all models are trained exclusively on this $30$-clue training set.
For evaluation, we use $2{,}000$ puzzles for each clue count.

For the $20$- and $17$-clue settings, we instead use uniquely solvable
$17$-clue Sudoku puzzles studied by \citet{lin2013specific}.
The $17$-clue evaluation set uses these puzzles directly, while the
$20$-clue evaluation set is constructed by augmenting each $17$-clue puzzle
with three additional clues from its unique solution.
Seventeen is the minimum number of clues for which a standard
$9\times9$ Sudoku puzzle can admit a unique solution
\citep{mcguire2014no16,lin2013specific}.

At evaluation time, the same checkpoint trained exclusively on $30$-clue
puzzles is evaluated with $40$, $35$, $30$, $25$, $20$, and $17$ clues.
The $30$-clue setting matches the training clue density.
The $40$- and $35$-clue settings provide more conditioning information than
observed during training, whereas the $25$-, $20$-, and $17$-clue settings
progressively reduce the available conditioning information.
For the no-clue evaluation, all $81$ cells in the puzzle prefix are replaced
with the blank token, yielding unconditional Sudoku generation.

\noindent\textbf{Sequence representation.}
The vocabulary contains $12$ symbols: a blank-cell token, the digits
$1$--$9$, a row-separator token, and a \texttt{BOS} token.
Each example is represented by $180$ tokens:
\[
[\texttt{BOS}]
\;+\;
\text{89-token puzzle}
\;+\;
[\texttt{BOS}]
\;+\;
\text{89-token solution}.
\]
Each $89$-token board representation contains $81$ cell tokens and eight row
separators.
The resulting puzzle prefix has length $91$, and the generated solution has
length $89$.
Unobserved cells are represented explicitly by the blank token, so the prefix
length remains $91$ for every clue count.
The training objective is evaluated only on the solution portion of the
sequence.

\subsection{Model and Training Details}
\label{sec:sudoku_model_training}

\noindent\textbf{Shared architecture.}
All methods use Transformer models with hidden dimension $512$, eight
Transformer blocks, eight attention heads of dimension $64$, and dropout
probability $0.1$.
Their parameter counts range from $25.21$M to $28.59$M, with the principal
differences across methods arising from the training objective, time
conditioning, and method-specific input and output components.
The models use learned token embeddings without embedding--output weight
tying.
The autoregressive model uses causal attention and no time conditioning.
The remaining models use bidirectional attention and AdaLN-based time
conditioning with conditioning dimension $128$.
The methods additionally differ in their inference procedures, as described
in the corresponding method definitions.

\noindent\textbf{Optimization.}
All models are trained for 20k optimization steps with global batch
size $256$ using bfloat16 precision.
We use Adam with learning rate $3\times10^{-4}$,
$\beta_1=0.9$, $\beta_2=0.999$, and gradient-norm clipping at $1.0$.
The learning rate is linearly warmed up for $2{,}500$ steps and then held
constant.
We maintain an EMA of the parameters with decay $0.9999$ and use the averaged
parameters for evaluation.
Training for 20k steps requires approximately 30 minutes on one NVIDIA
H200 GPU.

\subsection{Sampling and Decoding Details}
\label{sec:sudoku_sampling}

Table~\ref{tab:sampling_details} summarizes the samplers and terminal
decoding rules used for the Sudoku experiments.

\subsection{Additional Few-Step Results}
\label{sec:sudoku_additional_few_steps}

Figure~\ref{fig:sudoku_few_steps} reports few-step generation with $30$ clues.
We additionally report results for the same models, trained exclusively on
$30$-clue puzzles, under $40$-, $35$-, and $25$-clue evaluation settings in
\cref{fig:sudoku_nfe_40,fig:sudoku_nfe_35,fig:sudoku_nfe_25}. Across these
clue counts, Simplax maintains strong performance even at low NFE, showing that
its few-step generation behavior extends beyond the $30$-clue training setting.

\begin{figure}[t]
  \centering
  \includegraphics[width=0.8\linewidth]{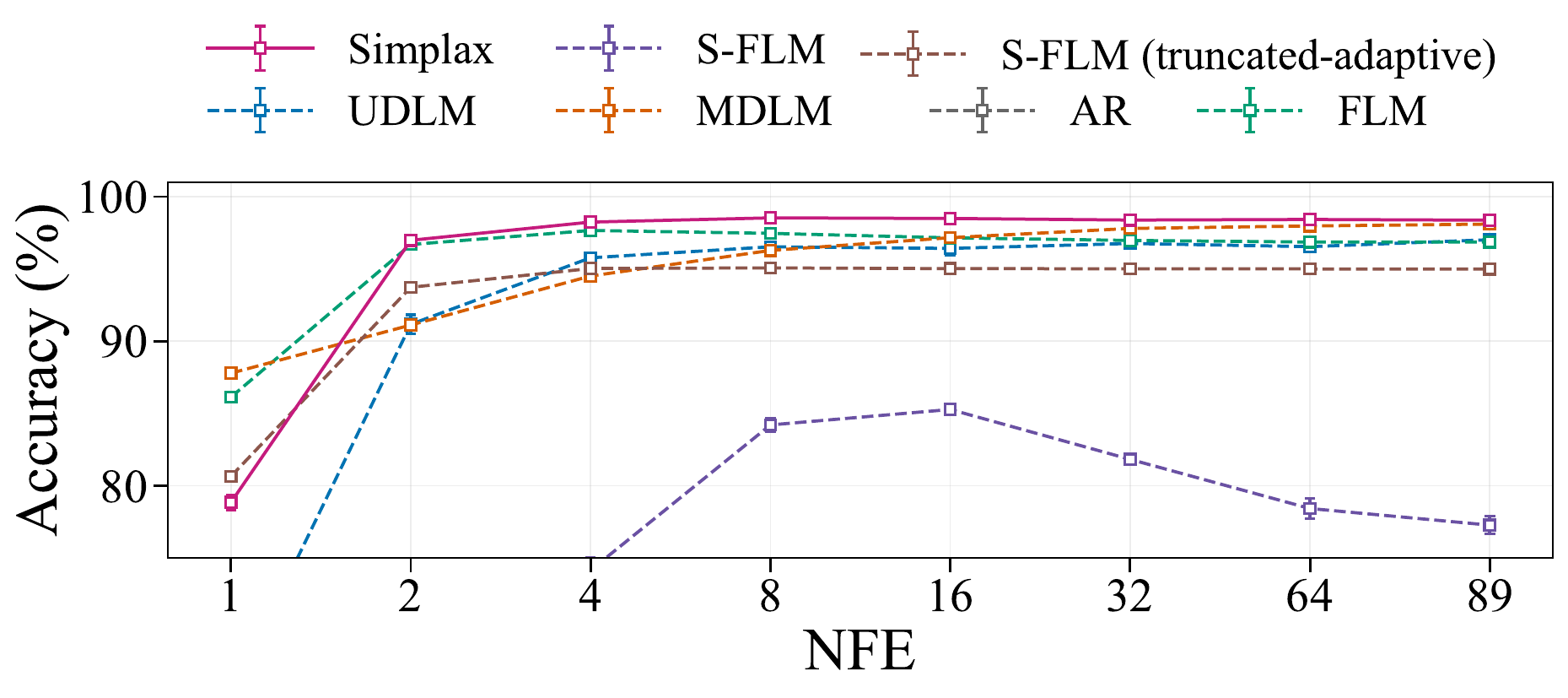}
  \caption{Accuracy versus NFE with $40$ clues.}
  \label{fig:sudoku_nfe_40}
\end{figure}

\begin{figure}[t]
  \centering
  \includegraphics[width=0.8\linewidth]{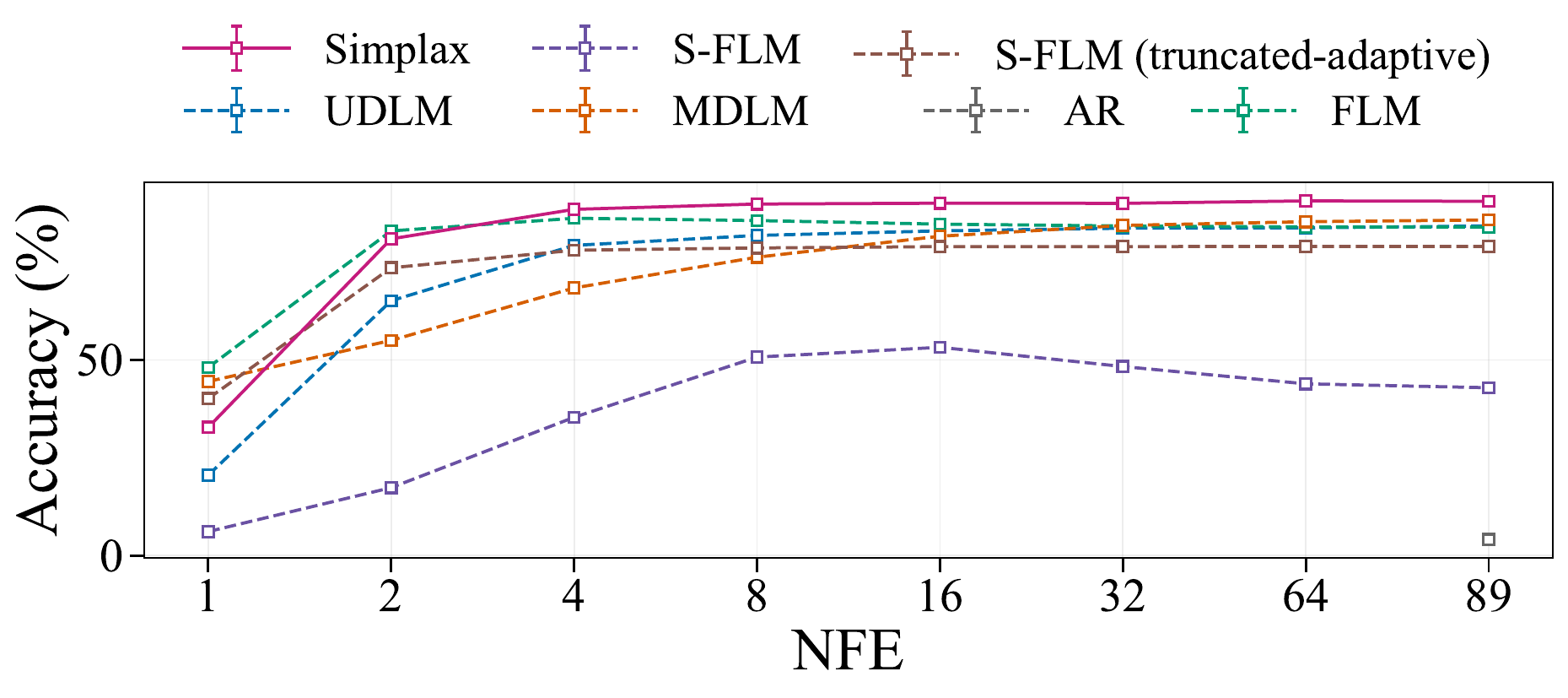}
  \caption{Accuracy versus NFE with $35$ clues.}
  \label{fig:sudoku_nfe_35}
\end{figure}

\begin{figure}[t]
  \centering
  \includegraphics[width=0.8\linewidth]{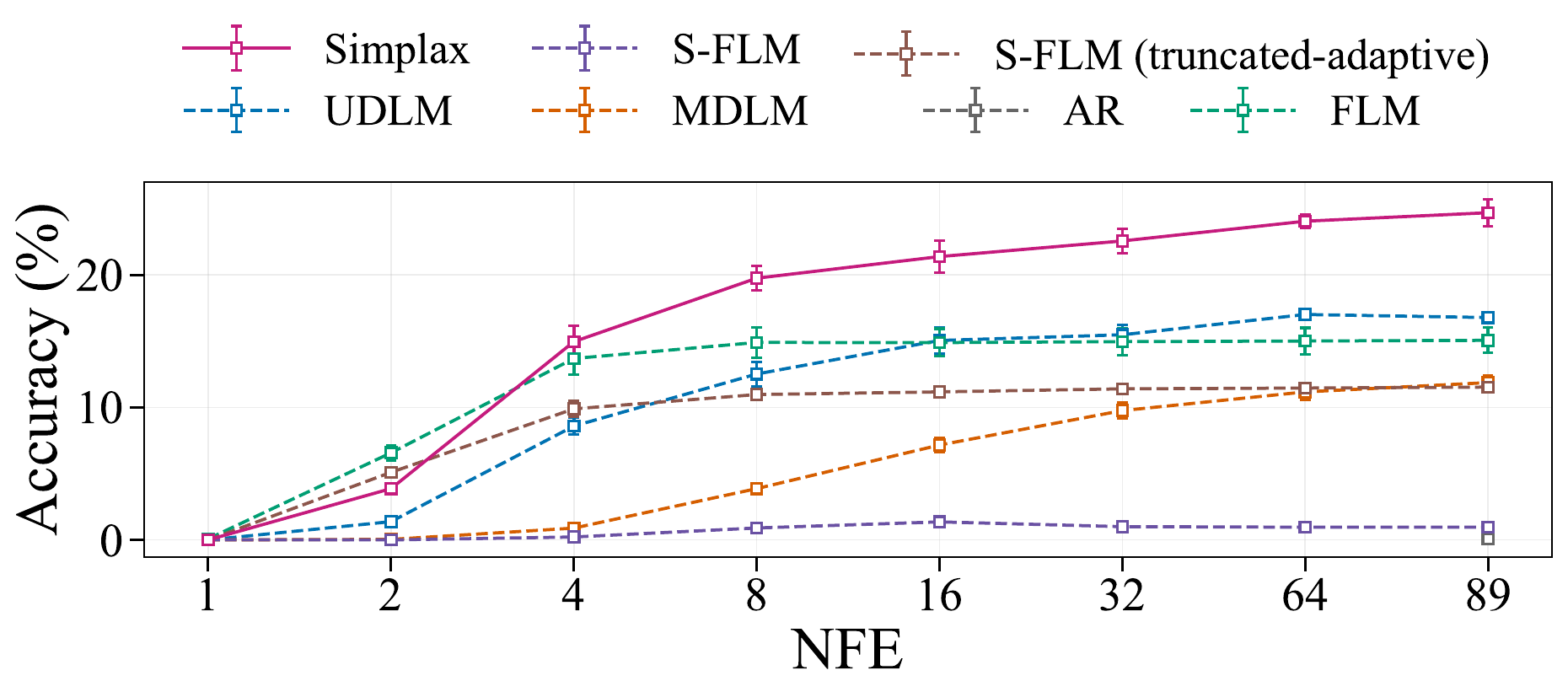}
  \caption{Accuracy versus NFE with $25$ clues.}
  \label{fig:sudoku_nfe_25}
\end{figure}

\subsection{Qualitative Generation Example}
\label{sec:sudoku_qualitative}
To complement the aggregate Sudoku results in the main text, we provide a
qualitative comparison of the generated solutions for a representative
$25$-clue puzzle (\cref{fig:sudoku-25clue-comparison}). This example illustrates the types of constraint violations
that are not visible from solving accuracy alone.

\begin{figure*}[t]
\centering
\includegraphics[width=\textwidth]{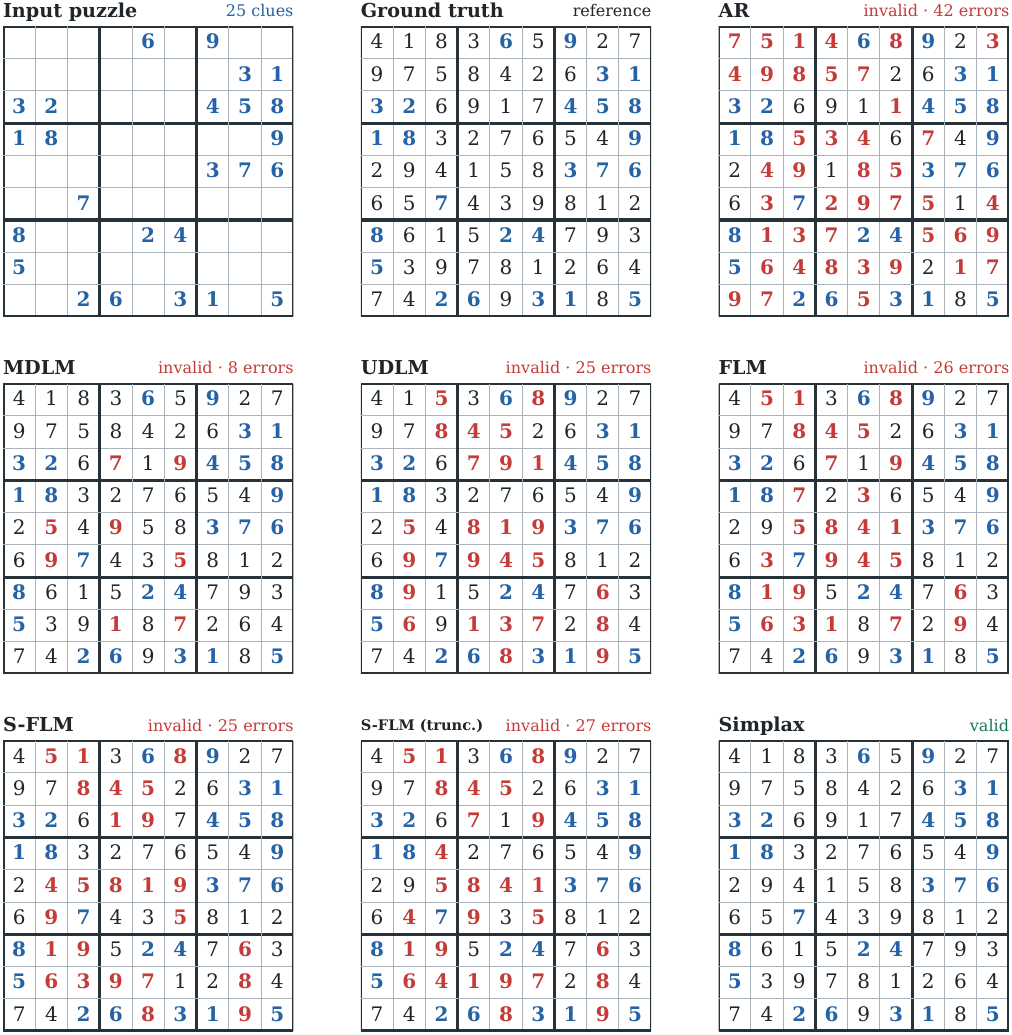}
\caption{
Generation results for a representative $25$-clue Sudoku puzzle at
NFE $=89$.
Blue entries are given clues, black entries agree with the unique solution,
and red entries differ from it.
All methods receive the same puzzle prefix.
Simplax produces the valid solution in this example, whereas the other methods
violate at least one Sudoku constraint.
}
\label{fig:sudoku-25clue-comparison}
\end{figure*}

%% file: sec_app/c_contribution_analysis.tex
\section{Contribution analysis details and frontiers}
\label{sec:contribution_analysis_frontiers}

This section provides additional details for the input, reverse-sampler,
bridge-concentration, training--sampling, and budget-allocation analyses on
OpenWebText. We also report the reverse-sampler runtime protocol and the
corresponding temperature-swept Gen.\ PPL--Gen.\ ENT curves.

\noindent\textbf{Input.}
We compare Simplax trained and evaluated with $\bz_t$ input against Simplax
trained and evaluated with $\bw_t$ input.
Both models are trained for $50\mathrm{k}$ steps with the Simplax objective
and $\eta=0.01$, and use stochastic sampling at inference
(\cref{fig:owt_input_ablation}).
We use $\bz_t$ in the main experiments because it yields the better frontier,
especially at high NFE, while also avoiding the dense vocabulary-sized
embedding multiplication required by $\bw_t$ input.

\begin{figure*}[t]
  \centering
  \includegraphics[width=\textwidth]{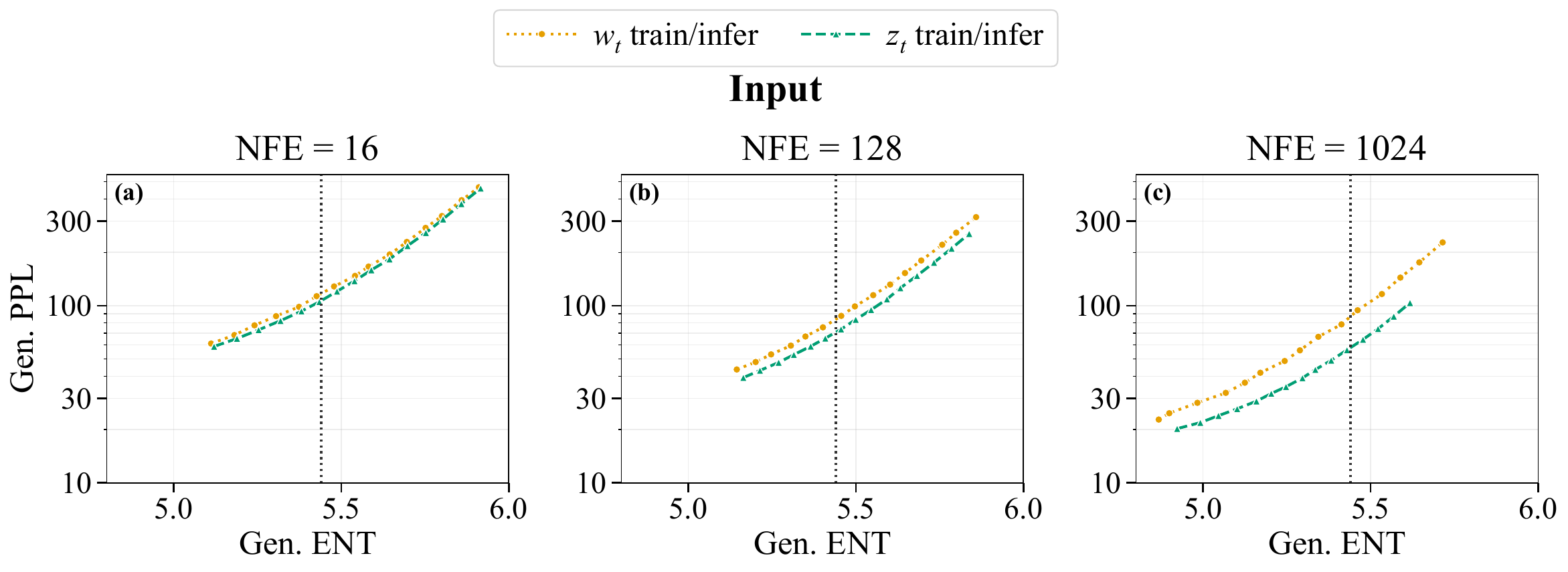}
  \caption{
    Temperature-swept input ablation on OpenWebText at
    NFE $=16$, $128$, and $1024$.
    The vertical dotted line marks the data entropy, $5.44$.
  }
  \label{fig:owt_input_ablation}
\end{figure*}

\noindent\textbf{Sampler quality and runtime.}
We compare the four Simplax samplers considered in the main text:
\textit{stochastic}, \textit{$\bw$-deterministic}, which makes only the $\bw$
update deterministic, \textit{$\bz$-deterministic}, which makes only the $\bz$
update deterministic, and \textit{deterministic}. For the frontier analysis, we
additionally evaluate an \textit{expectation} sampler. Their update rules are
detailed in \cref{app:sampling_details}. All variants use the same Simplax
checkpoint trained for $50\mathrm{k}$ steps with $\bz_t$ input and
$\eta=0.01$; only the sampler is changed.

Runtime and peak memory for the four main samplers are measured using the
protocol described in \cref{sec:owt_runtime_memory} and reported in
\cref{tab:simplax_sampler_runtime_memory}.
The \textit{stochastic} and \textit{$\bw$-deterministic} samplers yield similar
frontiers (\cref{fig:owt_reverse_sampler_ablation}), whereas
\textit{$\bz$-deterministic} fails to recover the target entropy.
The \textit{deterministic} and \textit{expectation} samplers collapse to low
entropy and fail to reach the data entropy.
The \textit{$\bw$-deterministic} sampler avoids the expensive Dirichlet sampling
step, which requires FP64 sampler computations, thereby reducing both runtime
and peak memory.
We therefore use \textit{$\bw$-deterministic} sampling in the main experiments,
which preserves the \textit{stochastic} sampler's quality while reducing
generation time.

\begin{figure*}[t]
  \centering
  \includegraphics[width=\textwidth]{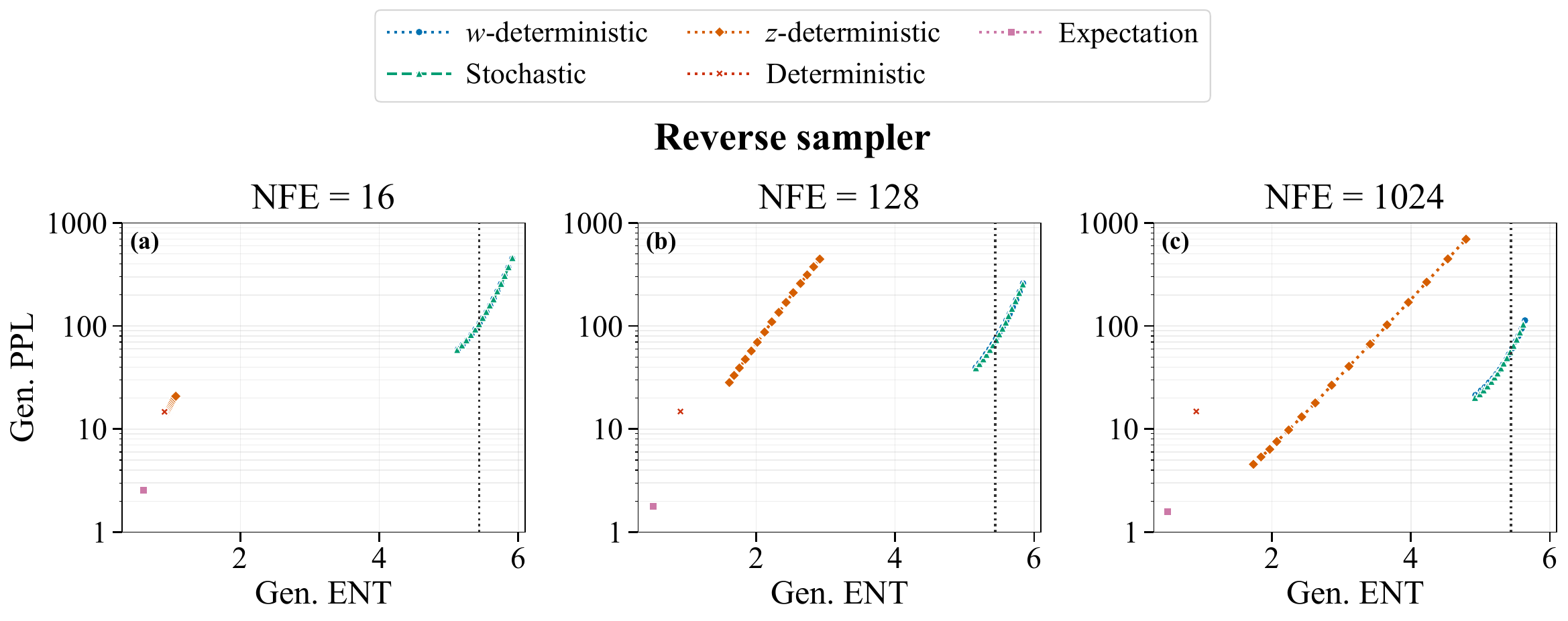}
  \caption{
    Temperature-swept reverse-sampler comparison on OpenWebText at
    NFE $=16$, $128$, and $1024$.
    The vertical dotted line marks the data entropy, $5.44$.
  }
  \label{fig:owt_reverse_sampler_ablation}
\end{figure*}

\begin{table*}[t]
\centering
\caption{
Runtime and peak allocated GPU memory of Simplax samplers on OpenWebText.
Runtime is median seconds per batch of 16 sequences of length 1024
$\pm$ sample standard deviation over five runs.
Peak memory is the maximum observed across the three NFE settings.
The best and second-best values are shown in bold and underlined, respectively.
}
\label{tab:simplax_sampler_runtime_memory}
\footnotesize
\setlength{\tabcolsep}{6pt}
\renewcommand{\arraystretch}{1.08}
\begin{tabular}{@{}lcccc@{}}
\toprule
Sampler
& NFE $=16$
& NFE $=128$
& NFE $=1024$
& Peak mem.\ (GiB) \\
\midrule
Stochastic
& $1.6722\pm0.0001$
& $12.975\pm0.001$
& $103.337\pm0.010$
& $26.48$ \\

\rowcolor{black!7}
$\bw$-deterministic
& $0.8136\pm0.0005$
& $6.101\pm0.002$
& $48.338\pm0.003$
& $\mathbf{20.34}$ \\

$\bz$-deterministic
& $1.3867\pm0.0004$
& $10.742\pm0.004$
& $85.478\pm0.011$
& $26.48$ \\

Deterministic
& $\underline{0.6566\pm0.0004}$
& $\underline{4.887\pm0.004}$
& $\underline{38.834\pm0.007}$
& $\mathbf{20.34}$ \\

Expectation
& $\mathbf{0.6148\pm0.0002}$
& $\mathbf{4.5390\pm0.0032}$
& $\mathbf{35.9733\pm0.0294}$
& $\mathbf{20.34}$ \\

\bottomrule
\end{tabular}
\end{table*}

\noindent\textbf{Bridge concentration.}
We compare $\eta\in\{0,0.001,0.01,0.1\}$ using Simplax models trained for
$50\mathrm{k}$ steps with $\bz_t$ input and evaluated with the stochastic
sampler (\cref{fig:owt_concentration_ablation}).
Each model is trained with the same concentration used at inference.
Small concentrations give similar low-NFE behavior, while $\eta=0.01$
provides a better high-NFE frontier; $\eta=0.1$ becomes less reliable and
fails to reach the target entropy at high NFE.
We therefore use $\eta=0.01$ in the main experiments.

\begin{figure*}[t]
  \centering
  \includegraphics[width=\textwidth]{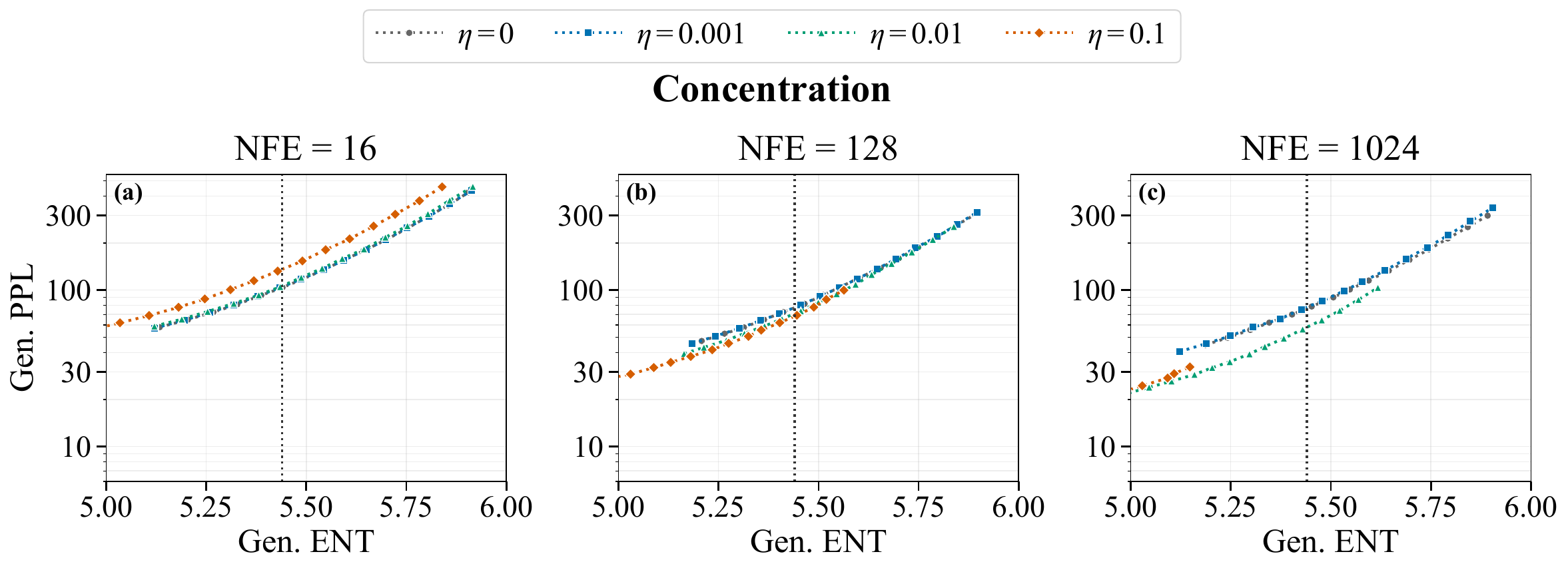}
  \caption{
    Temperature-swept bridge-concentration comparison on OpenWebText at
    NFE $=16$, $128$, and $1024$.
    The vertical dotted line marks the data entropy, $5.44$.
  }
  \label{fig:owt_concentration_ablation}
\end{figure*}

\noindent\textbf{Training--sampling combinations.}
We compare the UDLM $1\mathrm{M}$ checkpoint with the UDLM-initialized
Simplax checkpoint trained for $800\mathrm{k}$ UDLM steps followed by
$200\mathrm{k}$ Simplax steps.
Each checkpoint is evaluated with both the native UDLM ancestral sampler and
the Simplax $\bw$-deterministic sampler
(\cref{fig:owt_training_sampler_ablation}).
For Simplax, we use $\eta=0.01$ and $\bz_t$ input.

\begin{figure*}[t]
  \centering
  \includegraphics[width=\textwidth]{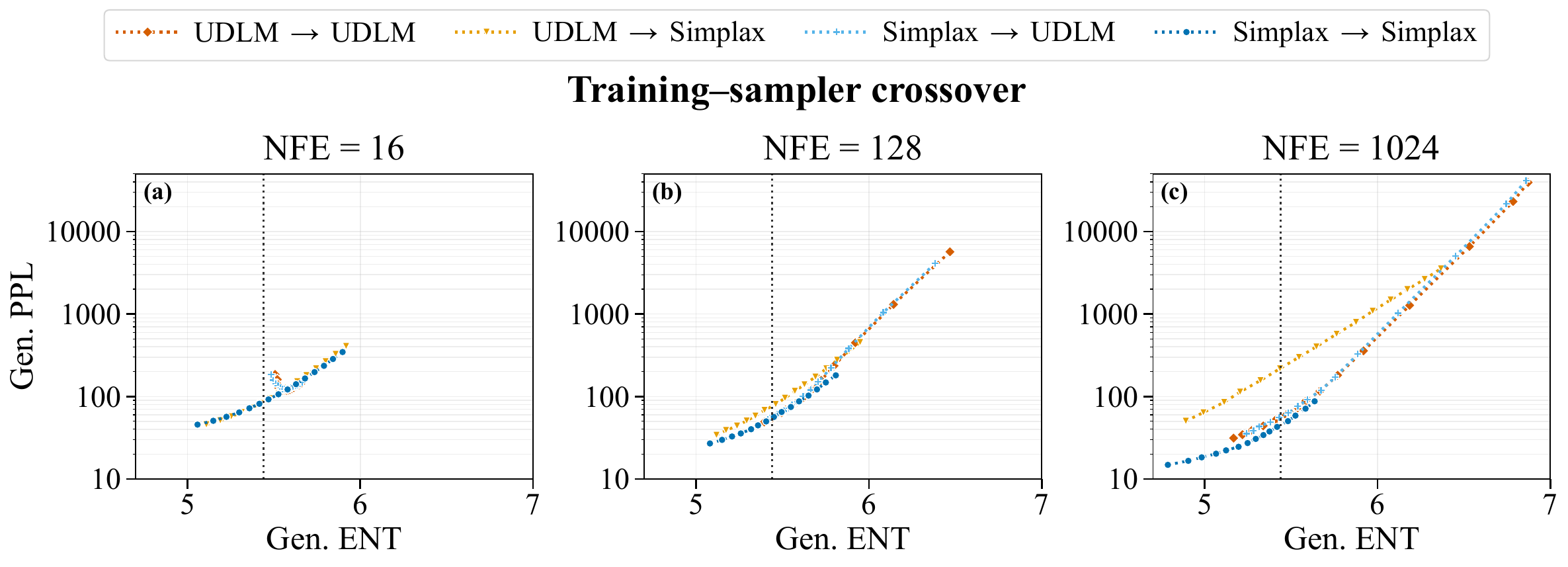}
  \caption{
    Temperature-swept training--sampling comparison on OpenWebText at
    NFE $=16$, $128$, and $1024$.
    The vertical dotted line marks the data entropy, $5.44$.
  }
  \label{fig:owt_training_sampler_ablation}
\end{figure*}

\paragraph{Budget allocation.}
\label{sec:owt_budget_allocation}
We compare UDLM $1\mathrm{M}$ and four staged allocations:
$600\mathrm{k}/400\mathrm{k}$, $700\mathrm{k}/300\mathrm{k}$,
$800\mathrm{k}/200\mathrm{k}$, and $900\mathrm{k}/100\mathrm{k}$ UDLM/Simplax
steps (\cref{tab:budget_allocation,fig:owt_budget_allocation}).
All staged variants perform similarly and consistently benefit from UDLM
initialization relative to training Simplax from scratch.
We use the $800\mathrm{k}/200\mathrm{k}$ allocation in the main experiments.

\begin{figure*}[t]
  \centering
  \includegraphics[width=\textwidth]{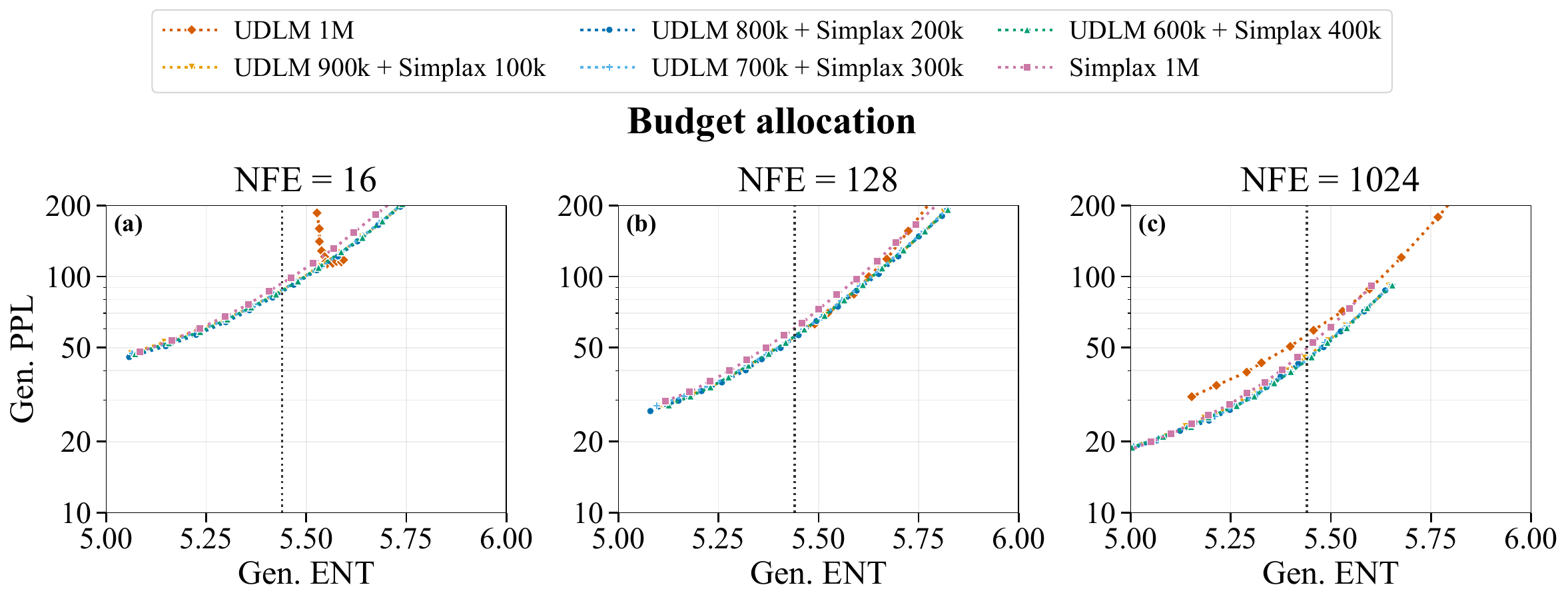}
  \caption{
    Temperature-swept budget-allocation comparison on OpenWebText at
    NFE $=16$, $128$, and $1024$.
    The vertical dotted line marks the data entropy, $5.44$.
  }
  \label{fig:owt_budget_allocation}
\end{figure*}

\begin{table*}[t]
  \centering
  \caption{
    Fixed-budget allocation between UDLM pretraining and Simplax training.
    We report GPT-2 Large Gen.\ PPL / Gen.\ ENT.
  }
  \label{tab:budget_allocation}
  \small
  \setlength{\tabcolsep}{7pt}
  \begin{tabular}{@{}lccc@{}}
    \toprule
    & NFE $=16$ & NFE $=128$ & NFE $=1024$ \\
    \cmidrule(lr){2-2}
    \cmidrule(lr){3-3}
    \cmidrule(lr){4-4}
    Allocation & PPL / ENT & PPL / ENT & PPL / ENT \\
    \midrule
    Simplax $1\mathrm{M}$
      & $98.7 / 5.46$
      & $63.4 / 5.46$
      & $52.6 / 5.45$ \\

    UDLM $600\mathrm{k}\rightarrow$ Simplax $400\mathrm{k}$
      & $84.2 / 5.42$
      & $52.6 / 5.42$
      & $45.6 / 5.45$ \\

    UDLM $700\mathrm{k}\rightarrow$ Simplax $300\mathrm{k}$
      & $84.0 / 5.42$
      & $52.4 / 5.42$
      & $45.4 / 5.45$ \\

    \rowcolor{black!7}
    UDLM $800\mathrm{k}\rightarrow$ Simplax $200\mathrm{k}$
      & $81.5 / 5.42$
      & $56.5 / 5.45$
      & $42.8 / 5.42$ \\

    UDLM $900\mathrm{k}\rightarrow$ Simplax $100\mathrm{k}$
      & $84.5 / 5.42$
      & $53.0 / 5.42$
      & $45.2 / 5.44$ \\

    UDLM $1\mathrm{M}$
      & $186.0 / 5.53^{\dagger}$
      & $62.9 / 5.49^{\dagger}$
      & $59.2 / 5.46$ \\
    \bottomrule
  \end{tabular}

  \vspace{0.3em}
  {\footnotesize
  $\dagger$ marks sweeps that do not reach the data entropy
  $5.44 \pm 0.05$; their PPL is not directly comparable at matched entropy.
  The shaded row is used in the main experiments.}
\end{table*}

%% file: sec_app/d_full_hierarchy.tex
\section{Auxiliary Identities and Full Dirichlet--Categorical Hierarchy}
\label{app:full_hierarchy}

This section develops the exact probabilistic structure behind the simplex relaxation.
The formulas below apply at positive diffusion times, where $\bp_t$ is strictly positive under the assumptions in \cref{sec:preliminaries}. We define the clean endpoint separately as the categorical state $\bz_0=\bx$; the hierarchy does not introduce a Dirichlet variable $\bw_0$.

The key point is that the auxiliary state $\bw_t$ is not introduced as a heuristic soft surrogate.
Rather, once we specify the Dirichlet bridge
\[
q(\bw_t \mid \bz_t,\bx)
=
\Dir\!\left(
\bw_t;
\eta_t \bp_t + \bz_t
\right),
\]
the resulting joint model admits a closed hierarchy in both directions:
the relaxed state has an exact Dirichlet marginal, the discrete state can be decoded exactly from $\bw_t$, and the reverse bridge remains tractable after marginalizing either $\bz_t$ or $\bz_s$.
We begin with two elementary identities that make these cancellations possible.

\subsection{Useful identities}
\label{app:useful_identities}

The first identity is the basic shift formula for Dirichlet densities.
It shows that adding a one-hot count to the concentration vector simply multiplies the base Dirichlet density by the corresponding simplex coordinate.

\begin{lemma}
\label{lem:dir_shift}
Let $\balpha\in\mathbb{R}_{>0}^K$ and let $\alpha_0=\sum_{i=1}^K \alpha_i$.
Then, for any $k\in\{1,\dots,K\}$,
\begin{equation}
\Dir(\bw; \balpha + \ek)
=
\frac{\alpha_0}{\alpha_k}
\, w_k \,
\Dir(\bw; \balpha).
\label{eq:dir_shift}
\end{equation}
\end{lemma}

\begin{proof}
By definition,
\[
\Dir(\bw; \balpha)
=
\frac{1}{B(\balpha)}
\prod_{i=1}^K w_i^{\alpha_i-1},
\qquad
B(\balpha)
=
\frac{\prod_{i=1}^K \Gamma(\alpha_i)}{\Gamma(\alpha_0)}.
\]
Hence
\[
\Dir(\bw; \balpha+\ek)
=
\frac{1}{B(\balpha+\ek)}
w_k
\prod_{i=1}^K w_i^{\alpha_i-1}.
\]
It remains to compare the normalizing constants:
\[
\frac{B(\balpha+\ek)}{B(\balpha)}
=
\frac{\Gamma(\alpha_k+1)}{\Gamma(\alpha_k)}
\frac{\Gamma(\alpha_0)}{\Gamma(\alpha_0+1)}
=
\frac{\alpha_k}{\alpha_0}.
\]
Therefore
\[
\frac{1}{B(\balpha+\ek)}
=
\frac{\alpha_0}{\alpha_k}
\frac{1}{B(\balpha)},
\]
which proves \cref{eq:dir_shift}.
\end{proof}

Specializing this identity to concentrations of the form $\eta \bp$ yields the cancellation that will be used throughout the appendix.

\begin{corollary}
\label{cor:dir_weighted_sum}
Let $\bp\in\simplex^{K-1}$ satisfy $p_k>0$ for all $k$, let $\eta>0$, and define $\balpha=\eta \bp$.
Then
\begin{equation}
p_k \, \Dir(\bw; \eta\bp+\ek)
=
w_k \, \Dir(\bw; \eta\bp).
\label{eq:dir_weighted_sum}
\end{equation}
\end{corollary}

\begin{proof}
Apply \cref{lem:dir_shift} with $\balpha=\eta\bp$.
Since $\alpha_0=\eta$ and $\alpha_k=\eta p_k$,
\[
\Dir(\bw; \eta\bp+\ek)
=
\frac{\eta}{\eta p_k}
w_k
\Dir(\bw; \eta\bp)
=
\frac{w_k}{p_k}
\Dir(\bw; \eta\bp).
\]
Multiplying both sides by $p_k$ gives \cref{eq:dir_weighted_sum}.
\end{proof}

The content of \cref{cor:dir_weighted_sum} is simple but important: a categorical mixture over one-hot shifts of a Dirichlet distribution collapses back to the unshifted Dirichlet density.
This is precisely the mechanism that makes the simplex relaxation exact rather than approximate.

\subsection{Full Dirichlet--categorical hierarchy}
\label{app:full_dir_cat_hierarchy}

We now return to the joint factorization
\begin{equation}
q(\bx,\bz_s,\bz_t,\bw_s,\bw_t)
=
q(\bx)\,
q(\bz_s\mid \bx)\,
q(\bw_s\mid \bz_s,\bx)\,
q(\bz_t\mid \bz_s)\,
q(\bw_t\mid \bz_t,\bx),
\label{eq:joint_graphical_model_app}
\end{equation}
together with
\begin{equation}
q(\bw_t \mid \bz_t,\bx)
=
\Dir\!\left(
\bw_t;
\eta_t \bp_t + \bz_t
\right).
\label{eq:dir_bridge_app}
\end{equation}

The next proposition summarizes the full hierarchy induced by this construction.
The first two statements identify the exact marginal and exact decoder at time $t$.
The third lifts the standard discrete reverse posterior from $\bz_t$ to $\bw_t$.
The last two show that, once this lift is performed, the reverse bridge over relaxed states becomes a mixture of shifted Dirichlet components.

\begin{proposition}
\label{prop:dir_cat_hierarchy_full}
Assume $t>0$. The Dirichlet--categorical hierarchy satisfies the following properties; statements involving $\bw_s$ additionally require $s>0$.
\begin{enumerate}
    \item The marginal distribution of the relaxed state is
    \begin{equation}
        q(\bw_t \mid \bx)
        =
        \Dir\!\left(
        \bw_t;
        \eta_t \bp_t
        \right).
        \label{eq:w_marginal_app}
    \end{equation}

    \item Given $\bw_t$, the variables $\bx$ and $\bz_t$ are conditionally independent, and the discrete state can be recovered from the relaxed state via
    \begin{equation}
        q(\bz_t \mid \bw_t,\bx)
        =
        q(\bz_t \mid \bw_t)
        =
        \Cat\!\left(
        \bz_t;
        \bw_t
        \right).
        \label{eq:z_given_w_app}
    \end{equation}

    \item For $s<t$, the reverse conditional posterior of $\bz_s$ given $\bw_t$ and $\bx$ is categorical:
    \begin{equation}
        q(\bz_s \mid \bw_t,\bx)
        =
        \Cat\!\left(
        \bz_s;
        \rhost{\bx}{\bw_t}
        \right),
        \label{eq:zs_given_wtx_app}
    \end{equation}
    where
    \begin{equation}
    \rhost{\bx}{\bw_t}
    \coloneq
    \bp_s
    \odot
    \left[
    \alpha_{t\mid s}
    \bigl(\bw_t \oslash \bp_t\bigr)
    +
    (1-\alpha_{t\mid s})
    \,
    \inner{\bw_t}{\bpi \oslash \bp_t}
    \one
    \right].
    \label{eq:rho_def_app}
    \end{equation}

    \item For $0<s<t$, the reverse conditional posterior of $\bw_s$ given $\bz_t$ and $\bx$ is a Dirichlet mixture:
    \begin{equation}
        q(\bw_s \mid \bz_t,\bx)
        =
        \sum_{k=1}^K
        r_{s\mid t,k}(\bx,\bz_t)
        \,
        \Dir\!\left(
        \bw_s;
        \eta_s \bp_s + \ek
        \right),
        \label{eq:ws_given_ztx_app}
    \end{equation}
    where $r_{s\mid t,k}(\bx,\bz_t)$ denotes the $k$-th component of
    $\rst{\bx}{\bz_t}$ defined in \cref{eq:reverse_posterior_vector}.

    \item For $0<s<t$, the reverse conditional posterior of $\bw_s$ given $\bw_t$ and $\bx$ is a Dirichlet mixture:
    \begin{equation}
        q(\bw_s \mid \bw_t,\bx)
        =
        \sum_{k=1}^K
        \rho_{s\mid t,k}(\bx,\bw_t)
        \,
        \Dir\!\left(
        \bw_s;
        \eta_s \bp_s + \ek
        \right).
        \label{eq:ws_given_wtx_app}
    \end{equation}
\end{enumerate}
\end{proposition}

\begin{proof}
\begin{enumerate}
    \item
    We begin with the marginal law of $\bw_t$.
    Marginalizing the discrete latent $\bz_t \sim \Cat(\bp_t)$ from the conditional bridge
    \[
    q(\bw_t \mid \bz_t,\bx)
    =
    \Dir\!\left(
    \bw_t;
    \eta_t \bp_t + \bz_t
    \right)
    \]
    gives
    \begin{align}
    q(\bw_t \mid \bx)
    &=
    \sum_{k=1}^K
    q(\bw_t \mid \bz_t=\ek,\bx)\,
    q(\bz_t=\ek \mid \bx)
    \notag\\
    &=
    \sum_{k=1}^K
    \Dir\!\left(
    \bw_t;
    \eta_t \bp_t + \ek
    \right)
    p_{t,k}(\bx).
    \label{eq:marginal_start_app}
    \end{align}
    Now \cref{cor:dir_weighted_sum} applies termwise:
    \[
    p_{t,k}(\bx)
    \Dir\!\left(
    \bw_t;
    \eta_t \bp_t + \ek
    \right)
    =
    w_{t,k}
    \Dir\!\left(
    \bw_t;
    \eta_t \bp_t
    \right).
    \]
    Substituting this into \cref{eq:marginal_start_app} collapses the mixture:
    \begin{align}
    q(\bw_t \mid \bx)
    &=
    \sum_{k=1}^K
    w_{t,k}
    \Dir\!\left(
    \bw_t;
    \eta_t \bp_t
    \right)
    \notag\\
    &=
    \left(\sum_{k=1}^K w_{t,k}\right)
    \Dir\!\left(
    \bw_t;
    \eta_t \bp_t
    \right)
    \notag\\
    &=
    \Dir\!\left(
    \bw_t;
    \eta_t \bp_t
    \right),
    \end{align}
    which proves \cref{eq:w_marginal_app}.

    \item
    We next derive the exact decoder from $\bw_t$ back to $\bz_t$.
    Fix $k\in\{1,\dots,K\}$.
    By Bayes' rule,
    \begin{align}
    q(\bz_t=\ek \mid \bw_t,\bx)
    &=
    \frac{
    q(\bw_t \mid \bz_t=\ek,\bx)\,
    q(\bz_t=\ek \mid \bx)
    }{
    q(\bw_t \mid \bx)
    }.
    \label{eq:bayes_prop2_app}
    \end{align}
    Using the previous computation together with \cref{cor:dir_weighted_sum}, the numerator becomes
    \[
    q(\bw_t \mid \bz_t=\ek,\bx)\,
    q(\bz_t=\ek \mid \bx)
    =
    w_{t,k}
    \Dir\!\left(
    \bw_t;
    \eta_t \bp_t
    \right),
    \]
    while the denominator is exactly $\Dir(\bw_t;\eta_t\bp_t)$.
    Hence
    \[
    q(\bz_t=\ek \mid \bw_t,\bx)=w_{t,k}.
    \]
    Since this holds for every $k$, we obtain
    \[
    q(\bz_t \mid \bw_t,\bx)=\Cat(\bz_t;\bw_t),
    \]
    and the right-hand side no longer depends on $\bx$.
    This proves \cref{eq:z_given_w_app}.

    \item
    We now lift the discrete reverse posterior from $\bz_t$ to $\bw_t$.
    Marginalizing over $\bz_t$ gives
    \begin{align}
    q(\bz_s \mid \bw_t,\bx)
    &=
    \sum_{j=1}^K
    q(\bz_s \mid \bz_t=\be_j,\bw_t,\bx)\,
    q(\bz_t=\be_j \mid \bw_t,\bx).
    \label{eq:zs_given_wt_expand_app}
    \end{align}
    Conditioned on $(\bz_t,\bx)$, the variable $\bz_s$ is independent of $\bw_t$, so the first factor is just the usual reverse posterior
    \[
    q(\bz_s \mid \bz_t=\be_j,\bw_t,\bx)
    =
    q(\bz_s \mid \bz_t=\be_j,\bx)
    =
    \Cat\!\left(
    \bz_s;
    \rst{\bx}{\be_j}
    \right).
    \]
    The second factor is the exact decoder derived above:
    \[
    q(\bz_t=\be_j \mid \bw_t,\bx)=w_{t,j}.
    \]
    Therefore
    \begin{equation}
    q(\bz_s \mid \bw_t,\bx)
    =
    \sum_{j=1}^K
    w_{t,j}\,
    \Cat\!\left(
    \bz_s;
    \rst{\bx}{\be_j}
    \right).
    \label{eq:zs_given_wt_mixture_app}
    \end{equation}
    A mixture of categorical distributions is again categorical, with parameter vector equal to the same convex combination of the component parameters.
    Thus
    \[
    q(\bz_s \mid \bw_t,\bx)
    =
    \Cat\!\left(
    \bz_s;
    \rhost{\bx}{\bw_t}
    \right),
    \qquad
    \rhost{\bx}{\bw_t}
    =
    \sum_{j=1}^K
    w_{t,j}\,
    \rst{\bx}{\be_j}.
    \]
    To obtain the closed form, substitute \cref{eq:reverse_posterior_vector} with $\bz_t=\be_j$:
    \begin{align}
    \rst{\bx}{\be_j}
    &=
    \frac{
    \left[
    \alpha_{t\mid s}\be_j
    +
    (1-\alpha_{t\mid s})\,\inner{\be_j}{\bpi}\,\one
    \right]
    \odot
    \bp_s
    }{
    \inner{\be_j}{\bp_t}
    }
    \notag\\
    &=
    \bp_s
    \odot
    \left[
    \alpha_{t\mid s}\frac{\be_j}{p_{t,j}(\bx)}
    +
    (1-\alpha_{t\mid s})
    \frac{\pi_j}{p_{t,j}(\bx)}
    \one
    \right].
    \label{eq:r_with_ej_app}
    \end{align}
    Averaging this expression under the weights $w_{t,j}$ yields
    \begin{align}
    \rhost{\bx}{\bw_t}
    &=
    \bp_s
    \odot
    \left[
    \alpha_{t\mid s}
    \sum_{j=1}^K
    w_{t,j}\frac{\be_j}{p_{t,j}(\bx)}
    +
    (1-\alpha_{t\mid s})
    \sum_{j=1}^K
    w_{t,j}\frac{\pi_j}{p_{t,j}(\bx)}
    \one
    \right]
    \notag\\
    &=
    \bp_s
    \odot
    \left[
    \alpha_{t\mid s}
    \bigl(\bw_t \oslash \bp_t\bigr)
    +
    (1-\alpha_{t\mid s})
    \inner{\bw_t}{\bpi \oslash \bp_t}
    \one
    \right],
    \end{align}
    which proves \cref{eq:zs_given_wtx_app} and \cref{eq:rho_def_app}.

    \item
    We next derive the reverse bridge over relaxed states conditioned on $\bz_t$.
    Since $\bw_s$ is conditionally independent of $\bz_t$ given $(\bz_s,\bx)$, marginalizing over $\bz_s$ gives
    \begin{align}
    q(\bw_s \mid \bz_t,\bx)
    &=
    \sum_{k=1}^K
    q(\bw_s \mid \bz_s=\ek,\bz_t,\bx)\,
    q(\bz_s=\ek \mid \bz_t,\bx)
    \notag\\
    &=
    \sum_{k=1}^K
    q(\bw_s \mid \bz_s=\ek,\bx)\,
    q(\bz_s=\ek \mid \bz_t,\bx).
    \label{eq:ws_zt_expand_app}
    \end{align}
    The first factor is exactly the shifted Dirichlet bridge at time $s$:
    \[
    q(\bw_s \mid \bz_s=\ek,\bx)
    =
    \Dir\!\left(
    \bw_s;
    \eta_s \bp_s+\ek
    \right).
    \]
    The second factor is the discrete reverse posterior
    \[
    q(\bz_s=\ek \mid \bz_t,\bx)=r_{s\mid t,k}(\bx,\bz_t).
    \]
    Substituting these identities into \cref{eq:ws_zt_expand_app} yields
    \[
    q(\bw_s \mid \bz_t,\bx)
    =
    \sum_{k=1}^K
    r_{s\mid t,k}(\bx,\bz_t)\,
    \Dir\!\left(
    \bw_s;
    \eta_s \bp_s+\ek
    \right),
    \]
    which proves \cref{eq:ws_given_ztx_app}.

    \item
    Finally, we lift this mixture from $\bz_t$ to $\bw_t$.
    Since $\bw_s$ is also conditionally independent of $\bw_t$ given $(\bz_s,\bx)$, we obtain
    \begin{align}
    q(\bw_s \mid \bw_t,\bx)
    &=
    \sum_{k=1}^K
    q(\bw_s \mid \bz_s=\ek,\bw_t,\bx)\,
    q(\bz_s=\ek \mid \bw_t,\bx)
    \notag\\
    &=
    \sum_{k=1}^K
    q(\bw_s \mid \bz_s=\ek,\bx)\,
    q(\bz_s=\ek \mid \bw_t,\bx).
    \label{eq:ws_wt_expand_app}
    \end{align}
    The first factor is again $\Dir(\bw_s;\eta_s\bp_s+\ek)$, while the second factor is now the lifted reverse posterior:
    \[
    q(\bz_s=\ek \mid \bw_t,\bx)=\rho_{s\mid t,k}(\bx,\bw_t).
    \]
    Therefore
    \[
    q(\bw_s \mid \bw_t,\bx)
    =
    \sum_{k=1}^K
    \rho_{s\mid t,k}(\bx,\bw_t)\,
    \Dir\!\left(
    \bw_s;
    \eta_s \bp_s+\ek
    \right),
    \]
    which proves \cref{eq:ws_given_wtx_app}.
\end{enumerate}
\end{proof}
Taken together, these identities show that the simplex variable sits inside the discrete diffusion process in a fully coherent way.
The relaxed state has the correct Dirichlet marginal, the discrete state can be decoded from it exactly, and the reverse bridge remains available in closed form after replacing either $\bz_t$ or $\bz_s$ by their simplex-conditioned posteriors.
This is the structural reason that the later training objectives and samplers can be written directly in terms of $\bw_t$ without abandoning the original categorical process.

%% file: sec_app/e_closed_form.tex
\section{Rao--Blackwellized Reverse-Bridge Objective}
\label{app:rao_blackwellized_objective}

This section proves the closed form of the relaxed discrete bridge objective
used in the main text. The auxiliary decoder sample $\widetilde{\bz}_t$ is
marginalized analytically; the independently sampled denoiser input $\bz_t$ is
not marginalized.

\begin{proposition}
\label{prop:rao_blackwellized_objective}
For $0\leq s<t\leq1$ with $t>0$, the relaxed discrete bridge objective satisfies
\begin{equation}
\bLoss{z_s\mid z_t,w_t}(\bw_t,\hat{\bx}_\theta,\bx; s,t)
=
\inner{\bw_t}{\log \hat{\bp}_t - \log \bp_t}
+
\inner{\rhost{\bx}{\bw_t}}{\log \bp_s - \log \hat{\bp}_s}.
\label{eq:ldisc_closed_form_app}
\end{equation}
\end{proposition}

\begin{proof}
    We first compute the Rao--Blackwellized categorical term
    $
    \bLoss{z_s\mid z_t,w_t}
    $.
    By definition,
    \begin{align}
    \bLoss{z_s\mid z_t,w_t}
    &=
    \EE{q(\widetilde{\bz}_t\mid \bw_t)}{
    \KL{
    q(\bz_s\mid \widetilde{\bz}_t,\bx)
    }{
    q(\bz_s\mid \widetilde{\bz}_t,\hat{\bx}_\theta)
    }
    }.
    \label{eq:rb_disc_start_app}
    \end{align}
    Using \cref{eq:z_given_w_app}, we have
    \[
    q(\widetilde{\bz}_t=\be_j\mid \bw_t)=w_{t,j}.
    \]
    For fixed $\widetilde{\bz}_t=\be_j$, both reverse conditionals are categorical:
    \[
    q(\bz_s\mid \widetilde{\bz}_t=\be_j,\bx)
    =
    \Cat\!\left(
    \bz_s;
    \rst{\bx}{\be_j}
    \right),
    \qquad
    q(\bz_s\mid \widetilde{\bz}_t=\be_j,\hat{\bx}_\theta)
    =
    \Cat\!\left(
    \bz_s;
    \rst{\hat{\bx}_\theta}{\be_j}
    \right).
    \]
    Expanding the expectation over $\widetilde{\bz}_t$ and the KL divergence between categorical distributions gives
    \begin{align}
    \bLoss{z_s\mid z_t,w_t}
    &=
    \sum_{j=1}^K
    w_{t,j}
    \sum_{k=1}^K
    r_{s\mid t,k}(\bx,\be_j)
    \log
    \frac{
    r_{s\mid t,k}(\bx,\be_j)
    }{
    r_{s\mid t,k}(\hat{\bx}_\theta,\be_j)
    }.
    \label{eq:rb_disc_expand_app}
    \end{align}
    Substituting the explicit form of the reverse posterior components, we obtain
    \begin{align}
    \frac{
    r_{s\mid t,k}(\bx,\be_j)
    }{
    r_{s\mid t,k}(\hat{\bx}_\theta,\be_j)
    }
    &=
    \frac{
    \bigl[\alpha_{t\mid s}\delta_{j,k}+(1-\alpha_{t\mid s})\pi_j\bigr]
    p_{s,k}(\bx)/p_{t,j}(\bx)
    }{
    \bigl[\alpha_{t\mid s}\delta_{j,k}+(1-\alpha_{t\mid s})\pi_j\bigr]
    p_{s,k}(\hat{\bx}_\theta)/p_{t,j}(\hat{\bx}_\theta)
    }
    \notag\\
    &=
    \frac{p_{s,k}(\bx)}{p_{s,k}(\hat{\bx}_\theta)}
    \frac{p_{t,j}(\hat{\bx}_\theta)}{p_{t,j}(\bx)}.
    \label{eq:rb_disc_ratio_app}
    \end{align}
    Hence
    \begin{align}
    \log
    \frac{
    r_{s\mid t,k}(\bx,\be_j)
    }{
    r_{s\mid t,k}(\hat{\bx}_\theta,\be_j)
    }
    &=
    \log \frac{p_{t,j}(\hat{\bx}_\theta)}{p_{t,j}(\bx)}
    +
    \log \frac{p_{s,k}(\bx)}{p_{s,k}(\hat{\bx}_\theta)}.
    \label{eq:rb_disc_logsplit_app}
    \end{align}
    Substituting \cref{eq:rb_disc_logsplit_app} into \cref{eq:rb_disc_expand_app}, the first contribution is
    \begin{align}
    \sum_{j=1}^K
    w_{t,j}
    \sum_{k=1}^K
    r_{s\mid t,k}(\bx,\be_j)
    \log \frac{p_{t,j}(\hat{\bx}_\theta)}{p_{t,j}(\bx)}
    &=
    \sum_{j=1}^K
    w_{t,j}
    \log \frac{p_{t,j}(\hat{\bx}_\theta)}{p_{t,j}(\bx)},
    \label{eq:rb_disc_firstterm_app}
    \end{align}
    because $r_{s\mid t}(\bx,\be_j)$ is a probability vector.
    For the second contribution, exchanging the order of summation gives
    \begin{align}
    \sum_{j=1}^K
    w_{t,j}
    \sum_{k=1}^K
    r_{s\mid t,k}(\bx,\be_j)
    \log \frac{p_{s,k}(\bx)}{p_{s,k}(\hat{\bx}_\theta)}
    &=
    \sum_{k=1}^K
    \left(
    \sum_{j=1}^K
    w_{t,j}r_{s\mid t,k}(\bx,\be_j)
    \right)
    \log \frac{p_{s,k}(\bx)}{p_{s,k}(\hat{\bx}_\theta)}.
    \label{eq:rb_disc_secondterm_pre_app}
    \end{align}
    The coefficient in parentheses is exactly the $k$-th component of
    $q(\bz_s\mid \bw_t,\bx)$, namely
    \[
    \sum_{j=1}^K
    w_{t,j}r_{s\mid t,k}(\bx,\be_j)
    =
    \rho_{s\mid t,k}(\bx,\bw_t).
    \]
    Therefore
    \begin{align}
    \bLoss{z_s\mid z_t,w_t}
    &=
    \sum_{j=1}^K
    w_{t,j}
    \log \frac{p_{t,j}(\hat{\bx}_\theta)}{p_{t,j}(\bx)}
    +
    \sum_{k=1}^K
    \rho_{s\mid t,k}(\bx,\bw_t)
    \log \frac{p_{s,k}(\bx)}{p_{s,k}(\hat{\bx}_\theta)}
    \notag\\
    &=
    \inner{\bw_t}{\log \hat{\bp}_t - \log \bp_t}
    +
    \inner{\rhost{\bx}{\bw_t}}{\log \bp_s - \log \hat{\bp}_s},
    \end{align}
    which proves \cref{eq:ldisc_closed_form_app}.
\end{proof}

Equation~\cref{eq:ldisc_closed_form_app} depends on the auxiliary state only
through the current-time average $\inner{\bw_t}{\log \hat{\bp}_t}$ and the
lifted reverse posterior $\rhost{\bx}{\bw_t}$. This is the expression used in
the main objective and in the continuous-time analysis.

%% file: sec_app/f_continuous_time.tex
\section{Continuous-Time Limit of the Main Objective}
\label{app:ct_other_objectives}

The main text shows that the relaxed discrete bridge objective
$
\bLoss{z_s\mid z_t,w_t}(\bw_t,\hat{\bx}_\theta,\bx; s,t)
$
admits a non-degenerate first-order continuous-time limit.
This section gives the full proof of the first-order local limit stated in the main text.

\subsection{Proof of the main continuous-time limit}
\label{app:proof_continuous_time_limit}

\begin{proposition}
\label{prop:continuous_time_limit_app}
Up to $\theta$-independent additive terms, the relaxed discrete bridge objective satisfies
\begin{equation}
\bLoss{z_s \mid z_t, w_t}(\bw_t,\hat{\bx}_\theta,\bx; t - \Delta,t)
=
\Delta\,
\ell_{\mathrm{ct}}(\bw_t,\hat{\bx}_\theta,\bx,t)
+
o(\Delta),
\label{eq:ct_expansion_app}
\end{equation}
where
\begin{align}
\ell_{\mathrm{ct}}(\bw_t,\hat{\bx}_\theta,\bx,t)
=
\lambda(t)\Bigg[
&\inner{\bw_t}{\bpi \oslash \hat{\bp}_t}
-
\inner{\bw_t}{\bpi \oslash \bp_t}\,
\inner{\bp_t}{\log \hat{\bp}_t}
+
\inner{\bpi \odot (\bw_t \oslash \bp_t)}{\log \hat{\bp}_t}
\Bigg].
\label{eq:ct_density_final_app}
\end{align}
and $\lambda(t)\coloneq -\frac{d}{dt}\log \alpha(t)$.
Consequently, the corresponding continuous-time objective is
\begin{equation}
\mathcal{L}_{\mathrm{ct}}
=
\int_0^1
\EE{q(\bx)}{
\EE{q(\bw_t\mid \bx)}{
\EE{q(\bz_t\mid \bw_t)}{
\ell_{\mathrm{ct}}
\bigl(
\bw_t,
f_\theta(\bz_t,t),
\bx,
t
\bigr)
}
}
}
\,\mathrm{d}t.
\label{eq:ct_objective_app}
\end{equation}
with $q(\bw_t\mid \bx)=\Dir(\bw_t;\eta_t \bp_t)$.
\end{proposition}

\begin{proof}[Proof of \Cref{prop:continuous_time_limit_app}]
Up to $\theta$-independent additive terms,
\Cref{eq:ldisc_closed_form_app} can be written as
\begin{equation}
\bLoss{z_s\mid z_t,w_t}(\bw_t,\hat{\bx}_\theta,\bx; s,t)
\equiv
\inner{\bw_t}{\log \hat{\bp}_t}
-
\inner{\rhost{\bx}{\bw_t}}{\log \hat{\bp}_s}.
\label{eq:app_ct_theta_only}
\end{equation}
We now set $s=t-\Delta$ and let $\Delta \downarrow 0$.

First, by definition,
\begin{equation}
\lambda(t) = -\frac{d}{dt}\log \alpha(t),
\end{equation}
so that
\begin{equation}
\alpha_{t\mid t-\Delta}
=
\frac{\alpha(t)}{\alpha(t-\Delta)}
=
1-\Delta\,\lambda(t)+o(\Delta).
\label{eq:app_alpha_ratio_expansion}
\end{equation}
Moreover,
\begin{equation}
\partial_t \bp_t
=
\partial_t\!\left(\alpha(t)\bx + (1-\alpha(t))\bpi\right)
=
-\lambda(t)\bigl(\bp_t-\bpi\bigr),
\end{equation}
hence
\begin{equation}
\bp_{t-\Delta}
=
\bp_t+\Delta\,\lambda(t)(\bp_t-\bpi)+o(\Delta).
\label{eq:app_pt_expansion}
\end{equation}

Next, we expand $\boldsymbol{\rho}_{t-\Delta\mid t}(\bx,\bw_t)$ using \Cref{eq:rho_def_app}.
Substituting \Cref{eq:app_alpha_ratio_expansion} and \Cref{eq:app_pt_expansion} gives
\begin{align}
&\boldsymbol{\rho}_{t-\Delta\mid t}(\bx,\bw_t) \notag\\
&=
\bp_{t-\Delta}
\odot
\left[
\alpha_{t\mid t-\Delta}
\bigl(\bw_t \oslash \bp_t\bigr)
+
\bigl(1-\alpha_{t\mid t-\Delta}\bigr)
\inner{\bw_t}{\bpi \oslash \bp_t}\one
\right]
\notag\\
&=
\Bigl(\bp_t+\Delta\,\lambda(t)(\bp_t-\bpi)\Bigr)
\odot
\Bigl[
\bw_t \oslash \bp_t
+
\Delta\,\lambda(t)
\Bigl(
\inner{\bw_t}{\bpi \oslash \bp_t}\one
-
\bw_t \oslash \bp_t
\Bigr)
\Bigr]
+o(\Delta)
\notag\\
&=
\bw_t
+
\Delta\,\lambda(t)
\left[
\inner{\bw_t}{\bpi \oslash \bp_t}\bp_t
-
\bpi \odot (\bw_t \oslash \bp_t)
\right]
+o(\Delta).
\label{eq:app_rho_expansion}
\end{align}

We now expand $\log \hat{\bp}_{t-\Delta}$.
Since $\hat{\bx}_\theta$ is held fixed in the local limit,
\begin{equation}
\hat{\bp}_t
=
\alpha(t)\hat{\bx}_\theta + (1-\alpha(t))\bpi,
\end{equation}
and therefore
\begin{equation}
\partial_t \hat{\bp}_t
=
-\lambda(t)\bigl(\hat{\bp}_t-\bpi\bigr).
\end{equation}
Dividing componentwise by $\hat{\bp}_t$ yields
\begin{equation}
\partial_t \log \hat{\bp}_t
=
-\lambda(t)\left(\one-\bpi \oslash \hat{\bp}_t\right).
\label{eq:app_logphat_derivative}
\end{equation}
Hence
\begin{equation}
\log \hat{\bp}_{t-\Delta}
=
\log \hat{\bp}_t
-
\Delta\,\partial_t \log \hat{\bp}_t
+
o(\Delta)
=
\log \hat{\bp}_t
+
\Delta\,\lambda(t)
\left(\one-\bpi \oslash \hat{\bp}_t\right)
+
o(\Delta).
\label{eq:app_logphat_expansion}
\end{equation}

For compactness, define
\begin{align}
A_t
&\coloneq
\inner{\bw_t}{\bpi \oslash \hat{\bp}_t},
\\
B_t
&\coloneq
\inner{\bw_t}{\bpi \oslash \bp_t}
\inner{\bp_t}{\log \hat{\bp}_t},
\\
C_t
&\coloneq
\inner{\bpi \odot (\bw_t \oslash \bp_t)}{\log \hat{\bp}_t}.
\end{align}
Substituting \Cref{eq:app_rho_expansion,eq:app_logphat_expansion} into
\Cref{eq:app_ct_theta_only} and collecting first-order terms gives
\begin{equation}
\bLoss{z_{t-\Delta}\mid z_t,w_t}
(\bw_t,\hat{\bx}_\theta,\bx;t-\Delta,t)
\equiv
\Delta\lambda(t)(A_t-B_t+C_t-1)+o(\Delta).
\label{eq:app_ct_before_drop_const}
\end{equation}
The term $-\Delta\,\lambda(t)$ is independent of $\theta$, so it can be discarded.
Therefore,
\begin{equation}
\bLoss{z_{t-\Delta}\mid z_t,w_t}(\bw_t,\hat{\bx}_\theta,\bx; t-\Delta,t)
=
\Delta\,\ell_{\mathrm{ct}}(\bw_t,\hat{\bx}_\theta,\bx,t)
+
o(\Delta),
\end{equation}
where, up to $\theta$-independent additive terms,
\begin{align}
\ell_{\mathrm{ct}}(\bw_t,\hat{\bx}_\theta,\bx,t)
&\equiv
\lambda(t)
\Bigg[
\inner{\bw_t}{\bpi \oslash \hat{\bp}_t}
-
\inner{\bw_t}{\bpi \oslash \bp_t}\inner{\bp_t}{\log \hat{\bp}_t}
+
\inner{\bpi \odot (\bw_t \oslash \bp_t)}{\log \hat{\bp}_t}
\Bigg].
\end{align}
This proves \Cref{eq:ct_expansion} and \Cref{eq:ct_density_final}.
Finally, integrating the density over time and taking expectation under $q(\bx)$ and $q(\bw_t\mid\bx)$ yields \Cref{eq:ct_objective_app}.
\end{proof}

%% file: sec_app/g_alternative_objectives.tex
\section{Alternative Surrogate Objectives}
\label{app:alternative_objectives}

In the main text, we optimize the relaxed discrete bridge objective
$
\bLoss{z_s\mid z_t,w_t}(\bw_t, \hat{\bx}_\theta, \bx; s, t)
$.
The denoiser prediction $\hat{\bx}_\theta=f_\theta(\bz_t,t)$ uses an independently sampled categorical network input. In this appendix, $\widetilde{\bz}_t$ denotes only the auxiliary categorical decode that is averaged inside the objective. This choice is best understood relative to a broader family of surrogate objectives induced by the same simplex relaxation.
The natural starting point is the exact relaxed bridge
\begin{equation}
\Loss{w_s\mid w_t}(\bw_t, \hat{\bx}_\theta, \bx; s, t)
\coloneq
\KL{q(\bw_s\mid \bw_t,\bx)}{q(\bw_s\mid \bw_t,\hat{\bx}_\theta)}.
\label{eq:direct_w_kl_app}
\end{equation}
This objective is the most direct one, but is generally intractable because
$q(\bw_s\mid \bw_t,\bx)$
is a Dirichlet mixture.
The tractable surrogates introduced below differ in two orthogonal ways:
first, whether they match only the discrete reverse state $\bz_s$ or the joint state $(\bz_s,\bw_s)$,
and second, whether they condition directly on $\bw_t$ or first decode $\widetilde{\bz}_t\sim q(\widetilde{\bz}_t\mid \bw_t)$ and average.

\subsection{A broader surrogate family}
\label{app:surrogate_family}

The simplest tractable surrogate matches only the lifted reverse posterior over $\bz_s$:
\begin{equation}
\Loss{z_s\mid w_t}(\bw_t, \hat{\bx}_\theta, \bx; s, t)
\coloneq
\KL{q(\bz_s\mid \bw_t,\bx)}{q(\bz_s\mid \bw_t,\hat{\bx}_\theta)}.
\label{eq:tight_cat_kl_app}
\end{equation}
This objective preserves the full relaxed conditioning information in $\bw_t$, but discards the relaxed target $\bw_s$.

The objective used in the main text instead decodes $\widetilde{\bz}_t$ from $\bw_t$ and averages the standard discrete reverse KL:
\begin{equation}
\bLoss{z_s\mid z_t,w_t}(\bw_t, \hat{\bx}_\theta, \bx; s, t)
\coloneq
\EE{q(\widetilde{\bz}_t\mid \bw_t)}{
\KL{q(\bz_s\mid \widetilde{\bz}_t,\bx)}{q(\bz_s\mid \widetilde{\bz}_t,\hat{\bx}_\theta)}
}.
\label{eq:loose_cat_kl_app}
\end{equation}
Compared with \cref{eq:tight_cat_kl_app}, this objective is looser because $\bw_t$ influences the reverse matching step only through the decoded categorical latent $\widetilde{\bz}_t$.
Its advantage is that it remains close to the standard discrete-diffusion objective and admits the non-degenerate continuous-time limit derived in the main text.

A richer alternative is to match the joint bridge over $(\bz_s,\bw_s)$ directly under the relaxed conditioning state:
\begin{equation}
\Loss{z_s,w_s\mid w_t}(\bw_t, \hat{\bx}_\theta, \bx; s, t)
\coloneq
\KL{
q(\bz_s,\bw_s\mid \bw_t,\bx)
}{
q(\bz_s,\bw_s\mid \bw_t,\hat{\bx}_\theta)
}.
\label{eq:tight_joint_kl_app}
\end{equation}
This objective is richer than the discrete surrogates because it also matches the relaxed bridge at time $s$.

Finally, one may combine the decoded conditioning of \cref{eq:loose_cat_kl_app} with the joint target of \cref{eq:tight_joint_kl_app}:
\begin{equation}
\bLoss{z_s,w_s\mid z_t,w_t}(\bw_t, \hat{\bx}_\theta, \bx; s, t)
\coloneq
\EE{q(\widetilde{\bz}_t\mid \bw_t)}{
\KL{
q(\bz_s,\bw_s\mid \widetilde{\bz}_t,\bw_t,\bx)
}{
q(\bz_s,\bw_s\mid \widetilde{\bz}_t,\bw_t,\hat{\bx}_\theta)
}
}.
\label{eq:loose_joint_kl_app}
\end{equation}
This is the richest tractable surrogate in the family: it keeps the relaxed target $\bw_s$ while also conditioning through the decoded latent $\widetilde{\bz}_t$.

The loose joint objective admits a useful chain-rule decomposition.
It shows that the main-text objective $\bLoss{z_s\mid z_t,w_t}$ is precisely the categorical part of the loose joint bridge.

\begin{proposition}
\label{prop:joint_kl_decomposition_app}
The tight and loose joint objectives admit the decompositions
\begin{align}
\Loss{z_s,w_s\mid w_t}(\bw_t, \hat{\bx}_\theta, \bx; s, t)
&=
\Loss{z_s\mid w_t}(\bw_t, \hat{\bx}_\theta, \bx; s, t)
+
\bLoss{w_s\mid z_s,w_t}(\bw_t, \hat{\bx}_\theta, \bx; s, t),
\label{eq:tight_joint_kl_decomposition_app}
\\
\bLoss{z_s,w_s\mid z_t,w_t}(\bw_t, \hat{\bx}_\theta, \bx; s, t)
&=
\bLoss{z_s\mid z_t,w_t}(\bw_t, \hat{\bx}_\theta, \bx; s, t)
+
\bLoss{w_s\mid z_s,w_t}(\bw_t, \hat{\bx}_\theta, \bx; s, t),
\label{eq:loose_joint_kl_decomposition_app}
\end{align}
where
\begin{equation}
\bLoss{w_s\mid z_s,w_t}(\bw_t, \hat{\bx}_\theta, \bx; s, t)
\coloneq
\EE{q(\bz_s\mid \bw_t,\bx)}{
\KL{
q(\bw_s\mid \bz_s,\bx)
}{
q(\bw_s\mid \bz_s,\hat{\bx}_\theta)
}
}.
\label{eq:loose_simplex_term_app}
\end{equation}
\end{proposition}

\begin{proof}
We first derive the tight decomposition \cref{eq:tight_joint_kl_decomposition_app}.

By the joint graphical model, we have the conditional independences
\[
\bw_s \perp\!\!\!\perp \bw_t \mid (\bz_s,\bx),
\qquad
q(\bz_s,\bw_s\mid \bw_t,\bx)
=
q(\bz_s\mid \bw_t,\bx)\,
q(\bw_s\mid \bz_s,\bx),
\]
and similarly on the model side,
\[
q(\bz_s,\bw_s\mid \bw_t,\hat{\bx}_\theta)
=
q(\bz_s\mid \bw_t,\hat{\bx}_\theta)\,
q(\bw_s\mid \bz_s,\hat{\bx}_\theta).
\]
Substituting these factorizations into
$
\Loss{z_s,w_s\mid w_t}
$
gives
\begin{align}
\Loss{z_s,w_s\mid w_t}
&=
\KL{
q(\bz_s\mid \bw_t,\bx)\,
q(\bw_s\mid \bz_s,\bx)
}{
q(\bz_s\mid \bw_t,\hat{\bx}_\theta)\,
q(\bw_s\mid \bz_s,\hat{\bx}_\theta)
}.
\label{eq:tight_joint_chain_start_app}
\end{align}
Applying the chain rule of KL yields
\begin{align}
\Loss{z_s,w_s\mid w_t}
&=
\KL{
q(\bz_s\mid \bw_t,\bx)
}{
q(\bz_s\mid \bw_t,\hat{\bx}_\theta)
}
\notag\\
&\quad+
\EE{q(\bz_s\mid \bw_t,\bx)}{
\KL{
q(\bw_s\mid \bz_s,\bx)
}{
q(\bw_s\mid \bz_s,\hat{\bx}_\theta)
}
}.
\label{eq:tight_joint_chain_split_app}
\end{align}
The first term is exactly
$
\Loss{z_s\mid w_t}
$,
while the second term is
$
\bLoss{w_s\mid z_s,w_t}
$
by definition.
This proves \cref{eq:tight_joint_kl_decomposition_app}.

We next derive the loose decomposition \cref{eq:loose_joint_kl_decomposition_app}.

The same graphical model implies the conditional independences
\[
\bz_s \perp\!\!\!\perp \bw_t \mid (\widetilde{\bz}_t,\bx),
\qquad
\bw_s \perp\!\!\!\perp (\widetilde{\bz}_t,\bw_t) \mid (\bz_s,\bx),
\]
and therefore
\begin{align}
q(\bz_s,\bw_s\mid \widetilde{\bz}_t,\bw_t,\bx)
&=
q(\bz_s\mid \widetilde{\bz}_t,\bx)\,
q(\bw_s\mid \bz_s,\bx),
\label{eq:loose_joint_factorization_true_app}
\\
q(\bz_s,\bw_s\mid \widetilde{\bz}_t,\bw_t,\hat{\bx}_\theta)
&=
q(\bz_s\mid \widetilde{\bz}_t,\hat{\bx}_\theta)\,
q(\bw_s\mid \bz_s,\hat{\bx}_\theta).
\label{eq:loose_joint_factorization_model_app}
\end{align}
Substituting \cref{eq:loose_joint_factorization_true_app} and
\cref{eq:loose_joint_factorization_model_app} into
$
\bLoss{z_s,w_s\mid z_t,w_t}
$
gives
\begin{align}
\bLoss{z_s,w_s\mid z_t,w_t}
&=
\EE{q(\widetilde{\bz}_t\mid \bw_t)}{
\KL{
q(\bz_s\mid \widetilde{\bz}_t,\bx)\,
q(\bw_s\mid \bz_s,\bx)
}{
q(\bz_s\mid \widetilde{\bz}_t,\hat{\bx}_\theta)\,
q(\bw_s\mid \bz_s,\hat{\bx}_\theta)
}
}.
\label{eq:loose_joint_chain_start_app}
\end{align}
Applying the chain rule of KL inside the expectation yields
\begin{align}
\bLoss{z_s,w_s\mid z_t,w_t}
&=
\EE{q(\widetilde{\bz}_t\mid \bw_t)}{
\KL{
q(\bz_s\mid \widetilde{\bz}_t,\bx)
}{
q(\bz_s\mid \widetilde{\bz}_t,\hat{\bx}_\theta)
}
}
\notag\\
&\quad+
\EE{q(\widetilde{\bz}_t\mid \bw_t)}{
\EE{q(\bz_s\mid \widetilde{\bz}_t,\bx)}{
\KL{
q(\bw_s\mid \bz_s,\bx)
}{
q(\bw_s\mid \bz_s,\hat{\bx}_\theta)
}
}
}.
\label{eq:loose_joint_chain_split_app}
\end{align}
The first term is exactly
$
\bLoss{z_s\mid z_t,w_t}
$.
For the second term, apply the law of total expectation:
\[
\EE{q(\widetilde{\bz}_t\mid \bw_t)}{
\EE{q(\bz_s\mid \widetilde{\bz}_t,\bx)}{[\cdot]}
}
=
\EE{q(\bz_s\mid \bw_t,\bx)}{[\cdot]}.
\]
Hence
\begin{align}
\bLoss{z_s,w_s\mid z_t,w_t}
&=
\bLoss{z_s\mid z_t,w_t}
\notag\\
&\quad+
\EE{q(\bz_s\mid \bw_t,\bx)}{
\KL{
q(\bw_s\mid \bz_s,\bx)
}{
q(\bw_s\mid \bz_s,\hat{\bx}_\theta)
}
},
\end{align}
which is exactly \cref{eq:loose_joint_kl_decomposition_app}.
\end{proof}

The decomposition in \cref{eq:loose_joint_kl_decomposition_app} clarifies the role of the selected main-text objective.
The relaxed discrete bridge keeps the categorical part of the joint bridge while discarding the additional simplex-matching term.
This is exactly the simplification that later makes the continuous-time limit non-degenerate.

\subsection{Auxiliary KL inequalities}
\label{app:aux_kl_inequalities}

We next record three standard KL inequalities that will be used to compare the surrogate objectives.

\begin{lemma}[Data processing inequality for KL divergence]
\label{lem:kl_dpi_app}
Let $P$ and $Q$ be two probability distributions on a space $\mathcal{X}$, and let
$\mathcal{K}(\by\mid \bx)$
be a stochastic kernel from $\mathcal{X}$ to $\mathcal{Y}$.
Define the pushforward distributions
\begin{equation}
(P\mathcal{K})(\by)
\coloneq
\EE{P(\bx)}{
\mathcal{K}(\by\mid \bx)
},
\qquad
(Q\mathcal{K})(\by)
\coloneq
\EE{Q(\bx)}{
\mathcal{K}(\by\mid \bx)
}.
\label{eq:pushforward_kernel_app}
\end{equation}
Then
\begin{equation}
\KL{P\mathcal{K}}{Q\mathcal{K}}
\le
\KL{P}{Q}.
\label{eq:kl_dpi_app}
\end{equation}
\end{lemma}

\begin{proof}
This follows directly from Theorem 4.1 of \citet{kullback1951information} by taking $T$ to be the stochastic kernel $\mathcal{K}$.
\end{proof}

A particularly important special case is marginalization.

\begin{corollary}[Marginalization cannot increase KL]
\label{cor:kl_marginal_app}
Let $P_{U,V}$ and $Q_{U,V}$ be two joint distributions, with marginals $P_U$ and $Q_U$.
Then
\begin{equation}
\KL{P_U}{Q_U}
\le
\KL{P_{U,V}}{Q_{U,V}}.
\label{eq:kl_marginal_app}
\end{equation}
\end{corollary}

\begin{proof}
Take $\mathcal{K}$ to be the projection kernel $(u,v)\mapsto u$ in \Cref{lem:kl_dpi_app}.
\end{proof}

\begin{lemma}[Joint convexity of KL divergence]
\label{lem:kl_joint_convexity_app}
Let $\lambda_i \ge 0$ with $\sum_{i=1}^m \lambda_i = 1$, and let $P_i,Q_i$ be probability distributions on a common space.
Then
\begin{equation}
\KL{\sum_{i=1}^m \lambda_i P_i}{\sum_{i=1}^m \lambda_i Q_i}
\le
\sum_{i=1}^m \lambda_i \KL{P_i}{Q_i}.
\label{eq:kl_joint_convexity_app}
\end{equation}
\end{lemma}

\begin{proof}
This follows from the joint convexity of relative entropy; see Theorem 2.7.2 in \citet{cover2006elements}.
\end{proof}

\subsection{Relations among the surrogate objectives}
\label{app:objective_relations}

The valid relations among the surrogate objectives follow from marginalization
and joint convexity of KL divergence.
Marginalizing either $\bz_s$ or $\bw_s$ from a joint bridge gives the
corresponding categorical or simplex bound.
In addition, averaging over the auxiliary decoded state
$\widetilde{\bz}_t\sim q(\widetilde{\bz}_t\mid\bw_t)$
gives the bounds from the direct categorical and direct joint objectives to
their decoded counterparts.

To state the decoded simplex bound, define
\begin{equation}
\bLoss{w_s\mid z_t,w_t}
(\bw_t,\hat{\bx}_\theta,\bx; s,t)
\coloneq
\EE{q(\widetilde{\bz}_t\mid\bw_t)}{
\KL{
q(\bw_s\mid\widetilde{\bz}_t,\bw_t,\bx)
}{
q(\bw_s\mid\widetilde{\bz}_t,\bw_t,\hat{\bx}_\theta)
}
}.
\label{eq:decoded_simplex_objective_app}
\end{equation}

\begin{proposition}
\label{prop:bound_app}
For all
$\bw_t \in \simplex^{K-1}$,
$\hat{\bx}_\theta \in \simplex^{K-1}$,
and
$\bx \in \Vocab$,
\begin{equation}
\begin{aligned}
\Loss{z_s\mid w_t}
&\le
\Loss{z_s,w_s\mid w_t},
\\
\Loss{w_s\mid w_t}
&\le
\Loss{z_s,w_s\mid w_t},
\\
\Loss{z_s,w_s\mid w_t}
&\le
\bLoss{z_s,w_s\mid z_t,w_t}.
\end{aligned}
\label{eq:bound_chain_w_app}
\end{equation}
Moreover,
\begin{equation}
\begin{aligned}
\Loss{z_s\mid w_t}
&\le
\bLoss{z_s\mid z_t,w_t}
\le
\bLoss{z_s,w_s\mid z_t,w_t},
\\
\bLoss{w_s\mid z_t,w_t}
&\le
\bLoss{z_s,w_s\mid z_t,w_t}.
\end{aligned}
\label{eq:bound_chain_z_app}
\end{equation}
\end{proposition}



\begin{proof}
We first prove \cref{eq:bound_chain_w_app}.

\begin{enumerate}
\item
The distributions
$q(\bz_s\mid\bw_t,\bx)$
and
$q(\bw_s\mid\bw_t,\bx)$
are the corresponding marginals of
$q(\bz_s,\bw_s\mid\bw_t,\bx)$.
The same statement holds on the model side.
Hence, by \Cref{cor:kl_marginal_app},
\begin{align}
\Loss{z_s\mid w_t}(\bw_t,\hat{\bx}_\theta,\bx)
&\le
\KL{
q(\bz_s,\bw_s\mid\bw_t,\bx)
}{
q(\bz_s,\bw_s\mid\bw_t,\hat{\bx}_\theta)
}
\notag\\
&=
\Loss{z_s,w_s\mid w_t}
(\bw_t,\hat{\bx}_\theta,\bx),
\label{eq:proof_direct_z_to_joint_app}
\\
\Loss{w_s\mid w_t}(\bw_t,\hat{\bx}_\theta,\bx)
&\le
\KL{
q(\bz_s,\bw_s\mid\bw_t,\bx)
}{
q(\bz_s,\bw_s\mid\bw_t,\hat{\bx}_\theta)
}
\notag\\
&=
\Loss{z_s,w_s\mid w_t}
(\bw_t,\hat{\bx}_\theta,\bx).
\label{eq:proof_upper_bound_joint_marginal_app}
\end{align}

\item
By marginalizing over $\widetilde{\bz}_t$, we have
\begin{align}
q(\bz_s,\bw_s\mid\bw_t,\bx)
&=
\EE{q(\widetilde{\bz}_t\mid\bw_t)}{
q(\bz_s,\bw_s
\mid
\widetilde{\bz}_t,\bw_t,\bx)
},
\label{eq:proof_upper_bound_joint_mix_true_app}
\\
q(\bz_s,\bw_s\mid\bw_t,\hat{\bx}_\theta)
&=
\EE{q(\widetilde{\bz}_t\mid\bw_t)}{
q(\bz_s,\bw_s
\mid
\widetilde{\bz}_t,\bw_t,\hat{\bx}_\theta)
}.
\label{eq:proof_upper_bound_joint_mix_model_app}
\end{align}
The mixing law is shared by both sides because it is always
$q(\widetilde{\bz}_t\mid\bw_t)$.
Applying \Cref{lem:kl_joint_convexity_app} yields
\begin{align}
\Loss{z_s,w_s\mid w_t}
(\bw_t,\hat{\bx}_\theta,\bx)
&=
\KL{
q(\bz_s,\bw_s\mid\bw_t,\bx)
}{
q(\bz_s,\bw_s\mid\bw_t,\hat{\bx}_\theta)
}
\notag\\
&\le
\EE{q(\widetilde{\bz}_t\mid\bw_t)}{
\KL{
q(\bz_s,\bw_s
\mid
\widetilde{\bz}_t,\bw_t,\bx)
}{
q(\bz_s,\bw_s
\mid
\widetilde{\bz}_t,\bw_t,\hat{\bx}_\theta)
}
}
\notag\\
&=
\bLoss{z_s,w_s\mid z_t,w_t}
(\bw_t,\hat{\bx}_\theta,\bx).
\label{eq:proof_upper_bound_joint_mix_bound_app}
\end{align}
\end{enumerate}

Combining
\cref{eq:proof_direct_z_to_joint_app},
\cref{eq:proof_upper_bound_joint_marginal_app},
and
\cref{eq:proof_upper_bound_joint_mix_bound_app}
proves \cref{eq:bound_chain_w_app}.

We now prove \cref{eq:bound_chain_z_app}.

\begin{enumerate}
\item
Using the conditional independence
\[
\bz_s
\perp\!\!\!\perp
\bw_t
\mid
(\widetilde{\bz}_t,\bx),
\]
and marginalizing over $\widetilde{\bz}_t$, we have
\begin{align}
q(\bz_s\mid\bw_t,\bx)
&=
\EE{q(\widetilde{\bz}_t\mid\bw_t)}{
q(\bz_s\mid\widetilde{\bz}_t,\bx)
},
\label{eq:proof_bound_z_mix_true_app}
\\
q(\bz_s\mid\bw_t,\hat{\bx}_\theta)
&=
\EE{q(\widetilde{\bz}_t\mid\bw_t)}{
q(\bz_s
\mid
\widetilde{\bz}_t,\hat{\bx}_\theta)
}.
\label{eq:proof_bound_z_mix_model_app}
\end{align}
Applying \Cref{lem:kl_joint_convexity_app} to these mixtures yields
\begin{equation}
\Loss{z_s\mid w_t}
(\bw_t,\hat{\bx}_\theta,\bx)
\le
\bLoss{z_s\mid z_t,w_t}
(\bw_t,\hat{\bx}_\theta,\bx).
\label{eq:proof_bound_z_first_app}
\end{equation}

\item
For each fixed $\widetilde{\bz}_t$,
$q(\bz_s\mid\widetilde{\bz}_t,\bx)$
and
$q(\bw_s\mid\widetilde{\bz}_t,\bw_t,\bx)$
are the corresponding marginals of
\[
q(\bz_s,\bw_s
\mid
\widetilde{\bz}_t,\bw_t,\bx),
\]
where
\[
q(\bz_s
\mid
\widetilde{\bz}_t,\bw_t,\bx)
=
q(\bz_s
\mid
\widetilde{\bz}_t,\bx).
\]
The same statements hold on the model side.
Therefore, by \Cref{cor:kl_marginal_app},
\begin{align}
&\KL{
q(\bz_s\mid\widetilde{\bz}_t,\bx)
}{
q(\bz_s
\mid
\widetilde{\bz}_t,\hat{\bx}_\theta)
}
\notag\\
&\qquad\le
\KL{
q(\bz_s,\bw_s
\mid
\widetilde{\bz}_t,\bw_t,\bx)
}{
q(\bz_s,\bw_s
\mid
\widetilde{\bz}_t,\bw_t,\hat{\bx}_\theta)
},
\label{eq:proof_bound_z_to_joint_pointwise_app}
\\
&\KL{
q(\bw_s
\mid
\widetilde{\bz}_t,\bw_t,\bx)
}{
q(\bw_s
\mid
\widetilde{\bz}_t,\bw_t,\hat{\bx}_\theta)
}
\notag\\
&\qquad\le
\KL{
q(\bz_s,\bw_s
\mid
\widetilde{\bz}_t,\bw_t,\bx)
}{
q(\bz_s,\bw_s
\mid
\widetilde{\bz}_t,\bw_t,\hat{\bx}_\theta)
}.
\label{eq:proof_bound_w_to_joint_pointwise_app}
\end{align}
Averaging both inequalities over
$q(\widetilde{\bz}_t\mid\bw_t)$
gives
\begin{align}
\bLoss{z_s\mid z_t,w_t}
(\bw_t,\hat{\bx}_\theta,\bx)
&\le
\bLoss{z_s,w_s\mid z_t,w_t}
(\bw_t,\hat{\bx}_\theta,\bx),
\label{eq:proof_bound_z_to_joint_app}
\\
\bLoss{w_s\mid z_t,w_t}
(\bw_t,\hat{\bx}_\theta,\bx)
&\le
\bLoss{z_s,w_s\mid z_t,w_t}
(\bw_t,\hat{\bx}_\theta,\bx).
\label{eq:proof_bound_w_to_joint_app}
\end{align}
\end{enumerate}

Combining
\cref{eq:proof_bound_z_first_app},
\cref{eq:proof_bound_z_to_joint_app},
and
\cref{eq:proof_bound_w_to_joint_app}
proves \cref{eq:bound_chain_z_app}.
\end{proof}


The joint surrogate contains an additional simplex-matching term. The following
identity and proposition give its closed form.

\begin{lemma}
\label{lem:shifted_dirichlet_log_moment}
For any $j,k\in\{1,\dots,K\}$,
\begin{equation}
\EE{\Dir\!\left(\cdot;\eta_s \bp_s+\be_k\right)}{\log w_j}
=
\psi\!\left(\eta_s p_{s,j}(\bx)+\delta_{j,k}\right)
-
\psi(\eta_s+1).
\label{eq:shifted_dirichlet_log_moment}
\end{equation}
\end{lemma}

\begin{proof}
For a Dirichlet random vector with parameter
$
\balpha=(\alpha_1,\dots,\alpha_K)
$,
the standard identity is
\[
\EE{\Dir(\cdot;\balpha)}{\log w_j}
=
\psi(\alpha_j)-\psi\!\left(\sum_{m=1}^K \alpha_m\right).
\]
Applying this with
\[
\balpha=\eta_s \bp_s+\be_k
\]
gives
\[
\alpha_j=\eta_s p_{s,j}(\bx)+\delta_{j,k},
\qquad
\sum_{m=1}^K \alpha_m
=
\eta_s\sum_{m=1}^K p_{s,m}(\bx)+1
=
\eta_s+1,
\]
which proves \cref{eq:shifted_dirichlet_log_moment}.
\end{proof}

\begin{proposition}
\label{prop:simplex_matching_closed_form}
For $0<s<t\leq1$, the simplex-matching term satisfies
\begin{align}
\bLoss{w_s\mid z_s,w_t}(\bw_t,\hat{\bx}_\theta,\bx; s,t)
&=
\KL{\Dir\!\left(\cdot;\eta_s \bp_s\right)}{\Dir\!\left(\cdot;\eta_s \hat{\bp}_s\right)}
\notag\\
&\quad+
\inner{\rhost{\bx}{\bw_t}}{
\log \hat{\bp}_s
-
\log \bp_s
+
\one
-
\hat{\bp}_s\oslash \bp_s
}.
\label{eq:lsimp_closed_form_app}
\end{align}
\end{proposition}

\begin{proof}
We compute the simplex term
    $
    \bLoss{w_s\mid z_s,w_t}
    $.
    By definition,
    \begin{equation}
    \bLoss{w_s\mid z_s,w_t}
    =
    \sum_{k=1}^K
    \rho_{s\mid t,k}(\bx,\bw_t)\,
    \KL{
    \Dir\!\left(\cdot;\eta_s \bp_s+\be_k\right)
    }{
    \Dir\!\left(\cdot;\eta_s \hat{\bp}_s+\be_k\right)
    }.
    \label{eq:proof_lsimp_expand_app}
    \end{equation}
    Fix $k\in\{1,\dots,K\}$.
    By \Cref{lem:dir_shift},
    \begin{align}
    \Dir\!\left(\bw;\eta_s \bp_s+\be_k\right)
    &=
    \frac{w_k}{p_{s,k}(\bx)}
    \Dir\!\left(\bw;\eta_s \bp_s\right),
    \label{eq:proof_shift_identity_x_app}
    \\
    \Dir\!\left(\bw;\eta_s \hat{\bp}_s+\be_k\right)
    &=
    \frac{w_k}{p_{s,k}(\hat{\bx}_\theta)}
    \Dir\!\left(\bw;\eta_s \hat{\bp}_s\right).
    \label{eq:proof_shift_identity_xhat_app}
    \end{align}
    Taking the logarithm of the ratio gives
    \begin{align}
    \log
    \frac{
    \Dir\!\left(\bw;\eta_s \bp_s+\be_k\right)
    }{
    \Dir\!\left(\bw;\eta_s \hat{\bp}_s+\be_k\right)
    }
    &=
    \log
    \frac{
    \Dir\!\left(\bw;\eta_s \bp_s\right)
    }{
    \Dir\!\left(\bw;\eta_s \hat{\bp}_s\right)
    }
    +
    \log p_{s,k}(\hat{\bx}_\theta)
    -
    \log p_{s,k}(\bx).
    \label{eq:proof_shifted_log_ratio_app}
    \end{align}

    We now compare the expectation of the unshifted log-ratio under the shifted Dirichlet law with the KL between the unshifted Dirichlet distributions.
    Expanding the Dirichlet density gives
    \begin{align}
    \log
    \frac{
    \Dir\!\left(\bw;\eta_s \bp_s\right)
    }{
    \Dir\!\left(\bw;\eta_s \hat{\bp}_s\right)
    }
    &=
    \log
    \frac{
    B(\eta_s\hat{\bp}_s)
    }{
    B(\eta_s\bp_s)
    }
    +
    \eta_s
    \sum_{j=1}^K
    \bigl(
    p_{s,j}(\bx)-p_{s,j}(\hat{\bx}_\theta)
    \bigr)\log w_j.
    \label{eq:proof_unshifted_log_ratio_app}
    \end{align}
    Taking expectation under
    $
    \Dir(\cdot;\eta_s \bp_s+\be_k)
    $
    and applying \Cref{lem:shifted_dirichlet_log_moment} gives
    \begin{align}
    &\EE{\Dir\!\left(\cdot;\eta_s \bp_s+\be_k\right)}{
    \log
    \frac{
    \Dir\!\left(\bw;\eta_s \bp_s\right)
    }{
    \Dir\!\left(\bw;\eta_s \hat{\bp}_s\right)
    }
    }
    \notag\\
    &\qquad=
    \log
    \frac{
    B(\eta_s\hat{\bp}_s)
    }{
    B(\eta_s\bp_s)
    }
    +
    \eta_s
    \sum_{j=1}^K
    \bigl(
    p_{s,j}(\bx)-p_{s,j}(\hat{\bx}_\theta)
    \bigr)
    \bigl(
    \psi(\eta_s p_{s,j}(\bx)+\delta_{j,k})
    -
    \psi(\eta_s+1)
    \bigr).
    \label{eq:proof_shifted_expectation_step1_app}
    \end{align}
    On the other hand, the KL divergence between the unshifted Dirichlet distributions is
    \begin{align}
    &\KL{
    \Dir\!\left(\cdot;\eta_s \bp_s\right)
    }{
    \Dir\!\left(\cdot;\eta_s \hat{\bp}_s\right)
    }\notag\\
    &=
    \log
    \frac{
    B(\eta_s\hat{\bp}_s)
    }{
    B(\eta_s\bp_s)
    }
    +
    \eta_s
    \sum_{j=1}^K
    \bigl(
    p_{s,j}(\bx)-p_{s,j}(\hat{\bx}_\theta)
    \bigr)
    \bigl(
    \psi(\eta_s p_{s,j}(\bx))
    -
    \psi(\eta_s)
    \bigr).
    \label{eq:proof_unshifted_kl_formula_app}
    \end{align}
    Subtracting \cref{eq:proof_unshifted_kl_formula_app} from
    \cref{eq:proof_shifted_expectation_step1_app}, and using
    \[
    \psi(a+1)=\psi(a)+\frac{1}{a},
    \qquad
    \psi(\eta_s+1)=\psi(\eta_s)+\frac{1}{\eta_s},
    \]
    yields
    \begin{align}
    &\EE{\Dir\!\left(\cdot;\eta_s \bp_s+\be_k\right)}{
    \log
    \frac{
    \Dir\!\left(\bw;\eta_s \bp_s\right)
    }{
    \Dir\!\left(\bw;\eta_s \hat{\bp}_s\right)
    }
    }
    \notag\\
    &=
    \KL{
    \Dir\!\left(\cdot;\eta_s \bp_s\right)
    }{
    \Dir\!\left(\cdot;\eta_s \hat{\bp}_s\right)
    }
    +
    1-\frac{p_{s,k}(\hat{\bx}_\theta)}{p_{s,k}(\bx)}.
    \label{eq:proof_shifted_expectation_revised_app}
    \end{align}
    Combining \cref{eq:proof_shifted_log_ratio_app} with
    \cref{eq:proof_shifted_expectation_revised_app}, we obtain
    \begin{align}
    &\KL{
    \Dir\!\left(\cdot;\eta_s \bp_s+\be_k\right)
    }{
    \Dir\!\left(\cdot;\eta_s \hat{\bp}_s+\be_k\right)
    } \notag\\
    &=
    \KL{
    \Dir\!\left(\cdot;\eta_s \bp_s\right)
    }{
    \Dir\!\left(\cdot;\eta_s \hat{\bp}_s\right)
    }
    +
    \log p_{s,k}(\hat{\bx}_\theta)
    -
    \log p_{s,k}(\bx)
    \notag\\
    &\quad+
    1-\frac{p_{s,k}(\hat{\bx}_\theta)}{p_{s,k}(\bx)}.
    \label{eq:proof_shifted_kl_final_app}
    \end{align}
    Substituting \cref{eq:proof_shifted_kl_final_app} into
    \cref{eq:proof_lsimp_expand_app}, and using
    $
    \sum_{k=1}^K \rho_{s\mid t,k}(\bx,\bw_t)=1
    $,
    yields
    \begin{align}
    \bLoss{w_s\mid z_s,w_t}
    &=
    \KL{
    \Dir\!\left(\cdot;\eta_s \bp_s\right)
    }{
    \Dir\!\left(\cdot;\eta_s \hat{\bp}_s\right)
    }
    \notag\\
    &\quad+
    \sum_{k=1}^K
    \rho_{s\mid t,k}(\bx,\bw_t)
    \left(
    \log p_{s,k}(\hat{\bx}_\theta)
    -
    \log p_{s,k}(\bx)
    +
    1-\frac{p_{s,k}(\hat{\bx}_\theta)}{p_{s,k}(\bx)}
    \right)
    \notag\\
    &=
    \KL{\Dir\!\left(\cdot;\eta_s \bp_s\right)}{\Dir\!\left(\cdot;\eta_s \hat{\bp}_s\right)}
    +
    \inner{\rhost{\bx}{\bw_t}}{
    \log \hat{\bp}_s
    -
    \log \bp_s
    +
    \one
    -
    \hat{\bp}_s\oslash \bp_s
    },
    \end{align}
    which proves \cref{eq:lsimp_closed_form_app}.
\end{proof}

\subsection{Why the other surrogate objectives do not yield suitable continuous-time objectives}
\label{app:ct_other_surrogates}

The relaxed discrete bridge is distinguished by its first-order scaling.
We now show that the remaining surrogates behave differently in the local limit:
the tight discrete objective vanishes at second order, whereas the joint objectives retain an $O(1)$ simplex-matching term.

\begin{proposition}
\label{prop:ct_other_objectives}
Let $s=t-\Delta$ with $\Delta\downarrow 0$, and assume that $\eta_t$ is continuous in $t$.

\begin{enumerate}
    \item The tight discrete objective satisfies
    \begin{equation}
    \Loss{z_{t-\Delta}\mid w_t}(\bw_t,\hat{\bx}_\theta,\bx; t-\Delta,t)
    =
    O(\Delta^2).
    \label{eq:tight_discrete_second_order_app}
    \end{equation}

    \item Define
    \begin{align}
    \mathcal{S}_t(\bw_t,\hat{\bx}_\theta,\bx)
    &\coloneq
    \KL{\Dir\!\left(\cdot;\eta_t \bp_t\right)}{\Dir\!\left(\cdot;\eta_t \hat{\bp}_t\right)}
    +
    \inner{\bw_t}{
    \log \hat{\bp}_t
    -
    \log \bp_t
    +
    \one
    -
    \hat{\bp}_t\oslash \bp_t
    }.
    \label{eq:ct_simplex_limit_term_app}
    \end{align}
    Then the simplex term satisfies
    \begin{equation}
    \bLoss{w_{t-\Delta}\mid z_{t-\Delta},w_t}(\bw_t,\hat{\bx}_\theta,\bx; t-\Delta,t)
    =
    \mathcal{S}_t(\bw_t,\hat{\bx}_\theta,\bx)
    +
    o(1).
    \label{eq:simplex_term_o1_app}
    \end{equation}

    \item Consequently,
    \begin{align}
    \Loss{z_{t-\Delta},w_{t-\Delta}\mid w_t}(\bw_t,\hat{\bx}_\theta,\bx; t-\Delta,t)
    &=
    \mathcal{S}_t(\bw_t,\hat{\bx}_\theta,\bx)
    +
    o(1),
    \label{eq:tight_joint_o1_app}
    \\
    \bLoss{z_{t-\Delta},w_{t-\Delta}\mid z_t,w_t}(\bw_t,\hat{\bx}_\theta,\bx; t-\Delta,t)
    &=
    \mathcal{S}_t(\bw_t,\hat{\bx}_\theta,\bx)
    +
    o(1).
    \label{eq:loose_joint_o1_app}
    \end{align}
\end{enumerate}
\end{proposition}

\begin{proof}
\begin{enumerate}
    \item
    We first analyze the tight discrete objective.
    Using \Cref{eq:rho_def_app} together with the same expansions
    \cref{eq:app_alpha_ratio_expansion} and \cref{eq:app_pt_expansion} as above, we obtain
    \begin{align}
    \boldsymbol{\rho}_{t-\Delta\mid t}(\bx,\bw_t)
    &=
    \bw_t
    +
    \Delta\,\ba_t(\bx,\bw_t)
    +
    o(\Delta),
    \label{eq:rho_true_local_app}
    \\
    \boldsymbol{\rho}_{t-\Delta\mid t}(\hat{\bx}_\theta,\bw_t)
    &=
    \bw_t
    +
    \Delta\,\hat{\ba}_t(\hat{\bx}_\theta,\bw_t)
    +
    o(\Delta),
    \label{eq:rho_model_local_app}
    \end{align}
    where
    \begin{align}
    \ba_t(\bx,\bw_t)
    &=
    \lambda(t)
    \left[
    \inner{\bw_t}{\bpi \oslash \bp_t}\bp_t
    -
    \bpi \odot (\bw_t \oslash \bp_t)
    \right],
    \label{eq:a_true_local_app}
    \\
    \hat{\ba}_t(\hat{\bx}_\theta,\bw_t)
    &=
    \lambda(t)
    \left[
    \inner{\bw_t}{\bpi \oslash \hat{\bp}_t}\hat{\bp}_t
    -
    \bpi \odot (\bw_t \oslash \hat{\bp}_t)
    \right].
    \label{eq:a_model_local_app}
    \end{align}
    Since both vectors in \cref{eq:rho_true_local_app} and \cref{eq:rho_model_local_app} are probability vectors, their first-order perturbations satisfy
    \[
    \inner{\one}{\ba_t(\bx,\bw_t)}=0,
    \qquad
    \inner{\one}{\hat{\ba}_t(\hat{\bx}_\theta,\bw_t)}=0.
    \]
    Therefore, expanding the categorical KL around the common base point $\bw_t$ shows that the first-order term cancels:
    \begin{align}
    \Loss{z_{t-\Delta}\mid w_t}
    &=
    \KL{
    \Cat\!\left(\cdot;
    \bw_t+\Delta\,\ba_t+o(\Delta)\right)
    }{
    \Cat\!\left(\cdot;
    \bw_t+\Delta\,\hat{\ba}_t+o(\Delta)\right)
    }
    \notag\\
    &=
    O(\Delta^2).
    \end{align}
    This proves \cref{eq:tight_discrete_second_order_app}.

    \item
    We next analyze the simplex term using its closed form \cref{eq:lsimp_closed_form_app}.
    Since $\eta_t$ is continuous,
    \[
    \eta_{t-\Delta}=\eta_t+o(1).
    \]
    Moreover, by \cref{eq:app_pt_expansion} and its analogue for $\hat{\bp}_{t-\Delta}$,
    \[
    \bp_{t-\Delta}=\bp_t+O(\Delta),
    \qquad
    \hat{\bp}_{t-\Delta}=\hat{\bp}_t+O(\Delta).
    \]
    Finally, \Cref{eq:app_rho_expansion} gives
    \[
    \boldsymbol{\rho}_{t-\Delta\mid t}(\bx,\bw_t)
    =
    \bw_t+O(\Delta).
    \]
    Substituting these expansions into \cref{eq:lsimp_closed_form_app} yields
    \begin{align}
    \bLoss{w_{t-\Delta}\mid z_{t-\Delta},w_t}
    &=
    \KL{\Dir\!\left(\cdot;\eta_t \bp_t\right)}{\Dir\!\left(\cdot;\eta_t \hat{\bp}_t\right)}
    \notag\\
    &\quad+
    \inner{\bw_t}{
    \log \hat{\bp}_t
    -
    \log \bp_t
    +
    \one
    -
    \hat{\bp}_t\oslash \bp_t
    }
    +
    o(1),
    \end{align}
    which is exactly \cref{eq:simplex_term_o1_app}.

    \item
    The asymptotics of the two joint objectives now follow from the decompositions in
    \Cref{eq:tight_joint_kl_decomposition_app,eq:loose_joint_kl_decomposition_app}.
    For the tight joint objective,
    \[
    \Loss{z_{t-\Delta},w_{t-\Delta}\mid w_t}
    =
    \Loss{z_{t-\Delta}\mid w_t}
    +
    \bLoss{w_{t-\Delta}\mid z_{t-\Delta},w_t},
    \]
    and combining \cref{eq:tight_discrete_second_order_app} with
    \cref{eq:simplex_term_o1_app} gives
    \cref{eq:tight_joint_o1_app}.

    For the loose joint objective,
    \[
    \bLoss{z_{t-\Delta},w_{t-\Delta}\mid z_t,w_t}
    =
    \bLoss{z_{t-\Delta}\mid z_t,w_t}
    +
    \bLoss{w_{t-\Delta}\mid z_{t-\Delta},w_t}.
    \]
    The first term is $o(1)$ by \Cref{prop:continuous_time_limit_app}, while the second is given by \cref{eq:simplex_term_o1_app}.
    This yields \cref{eq:loose_joint_o1_app}.
\end{enumerate}
\end{proof}

The proposition makes the selection of the main-text objective precise.
The tight discrete objective is too small in the local limit: after dividing by $\Delta$, it vanishes.
The joint objectives behave in the opposite way: they contain a generally nonzero $O(1)$ simplex-matching term, so they do not reduce to a finite first-order training density.
The relaxed discrete bridge sits exactly between these two extremes, which is why it is the natural objective for the continuous-time formulation.

%% file: sec_app/h_sampling.tex
\section{Sampling Details}
\label{app:sampling_details}
\label{app:sampler_variants}

The sampler in \cref{sec:sampling} is obtained by replacing the unknown clean
state in the reverse bridge with the denoiser prediction
\begin{equation}
\hat{\bx}_\theta
=
f_\theta(\bz_t,t)
\in
\simplex^{K-1}.
\label{eq:plugin_prediction_app}
\end{equation}
For adjacent times $s<t$, this gives the plug-in reverse hierarchy
\begin{align}
q_\theta(\bz_s\mid\bz_t,\bw_t)
&=
\Cat\!\left(
\bz_s;
\rhost{\hat{\bx}_\theta}{\bw_t}
\right),
\label{eq:plugin_reverse_z_app}
\\
q_\theta(\bw_s\mid\bz_s,\bz_t,\bw_t)
&=
\Dir\!\left(
\bw_s;
\eta_s\hat{\bp}_s+\bz_s
\right).
\label{eq:plugin_reverse_w_app}
\end{align}
The stochastic sampler follows these two conditionals ancestrally.

\noindent\textbf{Deterministic categorical update.}
The categorical transition can be determinized by replacing the draw from
\cref{eq:plugin_reverse_z_app} with the mode of the reverse categorical
posterior:
\begin{equation}
k^\star
=
\arg\max_k
\rho_{s\mid t,k}
\!\left(
\hat{\bx}_\theta,\bw_t
\right),
\qquad
\bz_s
\leftarrow
\mathbf e_{k^\star}.
\label{eq:z_deterministic_update_app}
\end{equation}

\noindent\textbf{Deterministic simplex update.}
Conditioned on a categorical state $\bz_s$, the simplex transition is
\[
q_\theta(\bw_s\mid\bz_s,\bz_t,\bw_t)
=
\Dir\!\left(
\bw_s;
\eta_s\hat{\bp}_s+\bz_s
\right).
\]
Since $\hat{\bp}_s$ and $\bz_s$ both sum to one, its conditional mean is
\begin{equation}
\mathbb{E}
\left[
\bw_s
\mid
\bz_s,\bz_t,\bw_t
\right]
=
\frac{
\eta_s\hat{\bp}_s+\bz_s
}{
\eta_s+1
}.
\label{eq:w_deterministic_mean_app}
\end{equation}
The $\bw$-deterministic update therefore replaces the Dirichlet draw with
\begin{equation}
\bw_s
\leftarrow
\frac{
\eta_s\hat{\bp}_s+\bz_s
}{
\eta_s+1
}.
\label{eq:w_deterministic_update_app}
\end{equation}

\noindent\textbf{Sampler variants.}
The categorical and simplex transitions admit stochastic and deterministic
updates, giving four basic sampler variants. Let
\begin{equation}
k^\star
=
\arg\max_k
\rho_{s\mid t,k}
\!\left(
\hat{\bx}_\theta,\bw_t
\right).
\label{eq:sampler_mode_app}
\end{equation}
The four variants are
\begin{align}
\text{Stochastic:}\qquad
\bz_s
&\sim
\Cat\!\left(
\rhost{\hat{\bx}_\theta}{\bw_t}
\right),
&
\bw_s
&\sim
\Dir\!\left(
\eta_s\hat{\bp}_s+\bz_s
\right),
\label{eq:stochastic_sampler_app}
\\
\bw\text{-deterministic:}\qquad
\bz_s
&\sim
\Cat\!\left(
\rhost{\hat{\bx}_\theta}{\bw_t}
\right),
&
\bw_s
&\leftarrow
\frac{
\eta_s\hat{\bp}_s+\bz_s
}{
\eta_s+1
},
\label{eq:w_det_sampler_app}
\\
\bz\text{-deterministic:}\qquad
\bz_s
&\leftarrow
\mathbf e_{k^\star},
&
\bw_s
&\sim
\Dir\!\left(
\eta_s\hat{\bp}_s+\bz_s
\right),
\label{eq:z_det_sampler_app}
\\
\text{Deterministic:}\qquad
\bz_s
&\leftarrow
\mathbf e_{k^\star},
&
\bw_s
&\leftarrow
\frac{
\eta_s\hat{\bp}_s+\bz_s
}{
\eta_s+1
}.
\label{eq:det_sampler_app}
\end{align}

We additionally consider an \textit{expectation} sampler, which marginalizes
over the categorical transition when updating the simplex state. Since
\begin{equation}
\mathbb{E}
\left[
\bw_s
\mid
\bz_t,\bw_t
\right]
=
\frac{
\eta_s\hat{\bp}_s
+
\rhost{\hat{\bx}_\theta}{\bw_t}
}{
\eta_s+1
},
\label{eq:expectation_sampler_mean_app}
\end{equation}
the expectation sampler uses
\begin{equation}
\bw_s
\leftarrow
\frac{
\eta_s\hat{\bp}_s
+
\rhost{\hat{\bx}_\theta}{\bw_t}
}{
\eta_s+1
}.
\label{eq:expectation_sampler_app}
\end{equation}
Unlike the deterministic sampler, which first selects the categorical mode
and then takes the corresponding conditional Dirichlet mean, the expectation
sampler averages over the categorical transition before updating $\bw_s$.

%% file: iclr2027_conference.bib
@inproceedings{
hoogeboom2021argmax,
title={Argmax Flows and Multinomial Diffusion: Learning Categorical Distributions},
author={Emiel Hoogeboom and Didrik Nielsen and Priyank Jaini and Patrick Forr{\'e} and Max Welling},
booktitle={Advances in Neural Information Processing Systems},
year={2021},
}

@inproceedings{
austin2021structured,
title={Structured Denoising Diffusion Models in Discrete State-Spaces},
author={Jacob Austin and Daniel D. Johnson and Jonathan Ho and Daniel Tarlow and Rianne van den Berg},
booktitle={Advances in Neural Information Processing Systems},
year={2021},
}

@inproceedings{
campbell2022a,
title={A Continuous Time Framework for Discrete Denoising Models},
author={Andrew Campbell and Joe Benton and Valentin De Bortoli and Tom Rainforth and George Deligiannidis and Arnaud Doucet},
booktitle={Advances in Neural Information Processing Systems},
year={2022},
}

@inproceedings{lou2024discrete,
title={Discrete Diffusion Modeling by Estimating the Ratios of the Data Distribution},
author={Lou, Aaron and Meng, Chenlin and Ermon, Stefano},
booktitle={International Conference on Machine Learning},
year={2024},
}

@inproceedings{
sahoo2024simple,
title={Simple and Effective Masked Diffusion Language Models},
author={Subham Sekhar Sahoo and Marianne Arriola and Aaron Gokaslan and Edgar Mariano Marroquin and Alexander M Rush and Yair Schiff and Justin T Chiu and Volodymyr Kuleshov},
booktitle={The Thirty-eighth Annual Conference on Neural Information Processing Systems},
year={2024},
}

@inproceedings{
shi2024simplified,
title={Simplified and Generalized Masked Diffusion for Discrete Data},
author={Jiaxin Shi and Kehang Han and Zhe Wang and Arnaud Doucet and Michalis Titsias},
booktitle={The Thirty-eighth Annual Conference on Neural Information Processing Systems},
year={2024},
}

@inproceedings{
schiff2025simple,
title={Simple Guidance Mechanisms for Discrete Diffusion Models},
author={Yair Schiff and Subham Sekhar Sahoo and Hao Phung and Guanghan Wang and Sam Boshar and Hugo Dalla-torre and Bernardo P de Almeida and Alexander M Rush and Thomas PIERROT and Volodymyr Kuleshov},
booktitle={The Thirteenth International Conference on Learning Representations},
year={2025},
}

@inproceedings{
ou2025your,
title={Your Absorbing Discrete Diffusion Secretly Models the Conditional Distributions of Clean Data},
author={Jingyang Ou and Shen Nie and Kaiwen Xue and Fengqi Zhu and Jiacheng Sun and Zhenguo Li and Chongxuan Li},
booktitle={The Thirteenth International Conference on Learning Representations},
year={2025},
}

@inproceedings{
zheng2025masked,
title={Masked Diffusion Models are Secretly Time-Agnostic Masked Models and Exploit Inaccurate Categorical Sampling},
author={Kaiwen Zheng and Yongxin Chen and Hanzi Mao and Ming-Yu Liu and Jun Zhu and Qinsheng Zhang},
booktitle={The Thirteenth International Conference on Learning Representations},
year={2025},
}

@inproceedings{
sahoo2025duality,
title={The Diffusion Duality},
author={Subham Sekhar Sahoo and Justin Deschenaux and Aaron Gokaslan and Guanghan Wang and Justin T Chiu and Volodymyr Kuleshov},
booktitle={Forty-second International Conference on Machine Learning},
year={2025},
}

@inproceedings{
zhang2025target,
title={Target Concrete Score Matching: A Holistic Framework for Discrete Diffusion},
author={Ruixiang Zhang and Shuangfei Zhai and Yizhe Zhang and James Thornton and Zijing Ou and Joshua M. Susskind and Navdeep Jaitly},
booktitle={Forty-second International Conference on Machine Learning},
year={2025},
}

@inproceedings{avdeyev2023dirichlet,
title={Dirichlet Diffusion Score Model for Biological Sequence Generation},
author={Avdeyev, Pavel and Shi, Chenlai and Tan, Yuhao and Dudnyk, Kseniia and Zhou, Jian},
booktitle={International Conference on Machine Learning},
year={2023},
}

@inproceedings{
stark2024dirichlet,
title={Dirichlet Flow Matching with Applications to {DNA} Sequence Design},
author={Hannes Stark and Bowen Jing and Chenyu Wang and Gabriele Corso and Bonnie Berger and Regina Barzilay and Tommi Jaakkola},
booktitle={Forty-first International Conference on Machine Learning},
year={2024},
}

@inproceedings{
chandra2026a,
title={A Unification of Discrete, Gaussian, and Simplicial Diffusion},
author={Nuria Alina Chandra and Yucen Lily Li and Alan Nawzad Amin and Alex Ali and Joshua Rollins and Sebastian W. Ober and Aniruddh Raghu and Andrew Gordon Wilson},
booktitle={The Fourteenth International Conference on Learning Representations},
year={2026},
}

@inproceedings{
deschenaux2026duality2,
title={The Diffusion Duality, Chapter {II}: {$\Psi$}-Samplers and Efficient Curriculum},
author={Justin Deschenaux and Caglar Gulcehre and Subham Sekhar Sahoo},
booktitle={The Fourteenth International Conference on Learning Representations},
year={2026},
}

@article{kullback1951information,
  author  = {Kullback, Solomon and Leibler, Richard A.},
  title   = {On Information and Sufficiency},
  journal = {The Annals of Mathematical Statistics},
  year    = {1951},
}

@book{
cover2006elements,
title={Elements of Information Theory},
author={Thomas M. Cover and Joy A. Thomas},
publisher={John Wiley \& Sons},
year={2006},
}

@inproceedings{
zheng2025cadd,
title={Continuously Augmented Discrete Diffusion model for Categorical Generative Modeling},
author={Huangjie Zheng and Shansan Gong and Ruixiang Zhang and Tianrong Chen and Jiatao Gu and Mingyuan Zhou and Navdeep Jaitly and Yizhe Zhang},
booktitle={The Fourteenth International Conference on Learning Representations},
year={2026},
}

@inproceedings{
pynadath2025candi,
title={CANDI: Hybrid Discrete-Continuous Diffusion Models},
author={Patrick Pynadath and Jiaxin Shi and Ruqi Zhang},
booktitle={International Conference on Machine Learning},
year={2026},
}

@misc{lee2026flm,
      title={Flow Map Language Models: One-step Language Modeling via Continuous Denoising}, 
      author={Chanhyuk Lee and Jaehoon Yoo and Manan Agarwal and Sheel Shah and Jerry Huang and Aditi Raghunathan and Seunghoon Hong and Nicholas M. Boffi and Jinwoo Kim},
      year={2026},
      eprint={2602.16813},
      archivePrefix={arXiv},
      primaryClass={cs.CL},
      url={https://arxiv.org/abs/2602.16813}, 
}

@inproceedings{
richemond2022categorical,
title={Categorical {SDE}s with Simplex Diffusion},
author={Pierre Harvey Richemond and Sander Dieleman and Arnaud Doucet},
booktitle={ICML 2023 Workshop: Sampling and Optimization in Discrete Space},
year={2023},
}

@inproceedings{
floto2023diffusion,
title={Diffusion on the Probability Simplex},
author={Griffin Floto and Thorsteinn Jonsson and Mihai Nica and Scott Sanner and Eric Zhengyu Zhu},
booktitle={ICML 2023 Workshop: Sampling and Optimization in Discrete Space},
year={2023},
}

@article{radford2019language,
title={Language Models are Unsupervised Multitask Learners},
author={Radford, Alec and Wu, Jeffrey and Child, Rewon and Luan, David and Amodei, Dario and Sutskever, Ilya},
journal={OpenAI Technical Report},
year={2019},
}

@article{touvron2023llama,
  title   = {Llama 2: Open Foundation and Fine-Tuned Chat Models},
  author  = {Touvron, Hugo and Martin, Louis and Stone, Kevin and Albert, Peter
             and Almahairi, Amjad and Babaei, Yasmine and Bashlykov, Nikolay
             and Batra, Soumya and Bhargava, Prajjwal and Bhosale, Shruti
             and Bikel, Dan and Blecher, Lukas and Canton Ferrer, Cristian
             and Chen, Moya and Cucurull, Guillem and Esiobu, David
             and Fernandes, Jude and Fu, Jeremy and Fu, Wenyin and Fuller, Brian
             and Gao, Cynthia and Goswami, Vedanuj and Goyal, Naman
             and Hartshorn, Anthony and Hosseini, Saghar and Hou, Rui
             and Inan, Hakan and Kardas, Marcin and Kerkez, Viktor
             and Khabsa, Madian and Kloumann, Isabel and Korenev, Artem
             and Koura, Punit Singh and Lachaux, Marie-Anne and Lavril, Thibaut
             and Lee, Jenya and Liskovich, Diana and Lu, Yinghai and Mao, Yuning
             and Martinet, Xavier and Mihaylov, Todor and Mishra, Pushkar
             and Molybog, Igor and Nie, Yixin and Poulton, Andrew
             and Reizenstein, Jeremy and Rungta, Rashi and Saladi, Kalyan
             and Schelten, Alan and Silva, Ruan and Smith, Eric Michael
             and Subramanian, Ranjan and Tan, Xiaoqing Ellen and Tang, Binh
             and Taylor, Ross and Williams, Adina and Kuan, Jian Xiang
             and Xu, Puxin and Yan, Zheng and Zarov, Iliyan and Zhang, Yuchen
             and Fan, Angela and Kambadur, Melanie and Narang, Sharan
             and Rodriguez, Aurelien and Stojnic, Robert and Edunov, Sergey
             and Scialom, Thomas},
  journal = {arXiv preprint arXiv:2307.09288},
  year    = {2023}
}

@misc{gokaslan2019openwebtext,
title={{OpenWebText} Corpus},
author={Gokaslan, Aaron and Cohen, Vanya},
year={2019},
}

@misc{elf,
      title={ELF: Embedded Language Flows}, 
      author={Keya Hu and Linlu Qiu and Yiyang Lu and Hanhong Zhao and Tianhong Li and Yoon Kim and Jacob Andreas and Kaiming He},
      year={2026},
      eprint={2605.10938},
      archivePrefix={arXiv},
      primaryClass={cs.CL},
      url={https://arxiv.org/abs/2605.10938}, 
}

@misc{sflm,
      title={Language Modeling with Hyperspherical Flows}, 
      author={Justin Deschenaux and Caglar Gulcehre},
      year={2026},
      eprint={2605.11125},
      archivePrefix={arXiv},
      primaryClass={cs.LG},
      url={https://arxiv.org/abs/2605.11125}, 
}

@inproceedings{dit,
  title={Scalable diffusion models with transformers},
  author={Peebles, William and Xie, Saining},
  booktitle={Proceedings of the IEEE/CVF international conference on computer vision},
  year={2023}
}

@article{rope,
  title={Roformer: Enhanced transformer with rotary position embedding},
  author={Su, Jianlin and Ahmed, Murtadha and Lu, Yu and Pan, Shengfeng and Bo, Wen and Liu, Yunfeng},
  journal={Neurocomputing},
  volume={568},
  year={2024},
}

@inproceedings{
adam,
title={Adam: A Method for Stochastic Optimization},
author={Diederik P. Kingma and Jimmy Ba},
booktitle={International Conference on Learning Representations},
year={2015},
}

@misc{langflow,
      title={LangFlow: Continuous Diffusion Rivals Discrete in Language Modeling}, 
      author={Yuxin Chen and Chumeng Liang and Hangke Sui and Ruihan Guo and Chaoran Cheng and Jiaxuan You and Ge Liu},
      year={2026},
      eprint={2604.11748},
      archivePrefix={arXiv},
      primaryClass={cs.CL},
      url={https://arxiv.org/abs/2604.11748}, 
}

@inproceedings{
rutte2026scaling,
title={Scaling Behavior of Discrete Diffusion Language Models},
author={Dimitri von R{\"u}tte and Janis Fluri and Omead Pooladzandi and Bernhard Sch{\"o}lkopf and Thomas Hofmann and Antonio Orvieto},
booktitle={The Fourteenth International Conference on Learning Representations},
year={2026},
}

@book{kotz2000continuous,
  title={Continuous Multivariate Distributions, Volume 1: Models and Applications},
  author={Kotz, Samuel and Balakrishnan, Narayanaswamy and Johnson, Norman L.},
  edition={2},
  year={2000},
  publisher={Wiley}
}

@article{zhao2025unified,
  title={Unified Discrete Diffusion for Categorical Data},
  author={Zhao, Lingxiao and Ding, Xueying and Yu, Lijun and Akoglu, Leman},
  journal={Journal of Machine Learning Research},
  year={2025}
}

@inproceedings{sahoo2026scaling,
      title={Scaling Beyond Masked Diffusion Language Models}, 
      author={Subham Sekhar Sahoo and Jean-Marie Lemercier and Zhihan Yang and Justin Deschenaux and Jingyu Liu and John Thickstun and Ante Jukic},
      year={2026},
      booktitle={The Forty-Third International Conference on Machine Learning},
}

@inproceedings{
rutte2025generalized,
title={Generalized Interpolating Discrete Diffusion},
author={Dimitri von R{\"u}tte and Janis Fluri and Yuhui Ding and Antonio Orvieto and Bernhard Sch{\"o}lkopf and Thomas Hofmann},
booktitle={Forty-second International Conference on Machine Learning},
year={2025},
}

@misc{alp2024sudoku,
  author       = {Ali Alp},
  title        = {Sudoku Puzzle Generator},
  year         = {2024},
  howpublished = {\url{https://github.com/alicommit-malp/sudoku}}
}

@article{lin2013specific,
  author  = {Hung-Hsuan Lin and I-Chen Wu and Tinghan Wei},
  title   = {On Specific 17-clue Sudoku Puzzles},
  journal = {ICGA Journal},
  year    = {2013},
}

@article{mcguire2014no16,
  author  = {Gary McGuire and Bastian Tugemann and Gilles Civario},
  title   = {There Is No 16-Clue Sudoku: Solving the Sudoku Minimum Number of
             Clues Problem via Hitting Set Enumeration},
  journal = {Experimental Mathematics},
  year    = {2014},
}

@inproceedings{di4c,
  title = 	 {Distillation of Discrete Diffusion through Dimensional Correlations},
  author =       {Hayakawa, Satoshi and Takida, Yuhta and Imaizumi, Masaaki and Wakaki, Hiromi and Mitsufuji, Yuki},
  booktitle = 	 {Proceedings of the 42nd International Conference on Machine Learning},
  year = 	 {2025},
}

@inproceedings{
vadd,
title={Variational Autoencoding Discrete Diffusion with Enhanced Dimensional Correlations Modeling},
author={Tianyu Xie and Shuchen Xue and Zijin Feng and Tianyang Hu and Jiacheng Sun and Zhenguo Li and Cheng Zhang},
booktitle={The Fourteenth International Conference on Learning Representations},
year={2026},
}

@inproceedings{
codd,
title={Breaking the Factorization Barrier in Diffusion Language Models},
author={Ian Li and Zilei Shao and Benjie Wang and Rose Yu and Guy Van den Broeck and Anji Liu},
booktitle={Forty-third International Conference on Machine Learning},
year={2026},
}

@inproceedings{
li2022diffusionlm,
title={Diffusion-{LM} Improves Controllable Text Generation},
author={Xiang Lisa Li and John Thickstun and Ishaan Gulrajani and Percy Liang and Tatsunori Hashimoto},
booktitle={Advances in Neural Information Processing Systems},
year={2022},
}

@misc{diffusiongemma2026,
  author       = {{Google DeepMind}},
  title        = {DiffusionGemma},
  year         = {2026},
  howpublished = {\url{https://ai.google.dev/gemma/docs/diffusiongemma/model_card}}
}

@misc{gourevitch2026uniformdiffusionmodelsrevisited,
      title={Uniform Diffusion Models Revisited: Leave-One-Out Denoiser and Absorbing State Reformulation}, 
      author={Samson Gourevitch and Yazid Janati and Dario Shariatian and Umut Simsekli and Eric Moulines and Eric P. Xing and Alain Durmus},
      year={2026},
      eprint={2605.22765},
      archivePrefix={arXiv},
      primaryClass={cs.LG},
      url={https://arxiv.org/abs/2605.22765}, 
}

@misc{roos2026categoricalflowmaps,
      title={Categorical Flow Maps}, 
      author={Daan Roos and Oscar Davis and Floor Eijkelboom and Michael Bronstein and Max Welling and İsmail İlkan Ceylan and Luca Ambrogioni and Jan-Willem van de Meent},
      year={2026},
      eprint={2602.12233},
      archivePrefix={arXiv},
      primaryClass={cs.LG},
      url={https://arxiv.org/abs/2602.12233}, 
}

@inproceedings{
jo2025continuous,
title={Continuous Diffusion Model for Language Modeling},
author={Jaehyeong Jo and Sung Ju Hwang},
booktitle={The Thirty-ninth Annual Conference on Neural Information Processing Systems},
year={2025},
}

@inproceedings{
chen2023analog,
title={Analog Bits: Generating Discrete Data using Diffusion Models with Self-Conditioning},
author={Ting Chen and Ruixiang Zhang and Geoffrey Hinton},
booktitle={The Eleventh International Conference on Learning Representations },
year={2023},
}
